\documentclass{article} 

\usepackage{iclr2027_conference,times}

\usepackage[utf8]{inputenc} 
\usepackage[T1]{fontenc}    
\usepackage{hyperref}       
\usepackage{url}            
\usepackage{booktabs}       
\usepackage{amsfonts}       
\usepackage{nicefrac}       
\usepackage{microtype}      
\usepackage{xcolor}         

\usepackage{graphicx}
\usepackage{subcaption}

\newcommand{\stitle}[1]{\vspace*{0.4em}\noindent{\bf #1.\/}}

\usepackage{amsmath}
\usepackage{amssymb}
\usepackage{mathtools}
\usepackage{amsthm}
\usepackage{multirow}
\usepackage{tikz}
\usepackage{caption}
\usepackage{algorithm}
\usepackage{algpseudocodex}
\usepackage{siunitx}
\usepackage{enumitem}
\usepackage{makecell}
\usepackage{wrapfig}
\usepackage{float}

\theoremstyle{plain}
\newtheorem{theorem}{Theorem}[section]

\newtheorem{lemma}[theorem]{Lemma}
\newtheorem{corollary}[theorem]{Corollary}
\theoremstyle{definition}
\newtheorem{definition}[theorem]{Definition}

\theoremstyle{remark}

\newcommand{\mV}[0]{\mathbf{V}}
\newcommand{\mS}[0]{\mathbf{S}}

\newcommand{\mI}[0]{\mathbf{I}}
\newcommand{\vg}[0]{\mathbf{g}}
\newcommand{\vx}[0]{\mathbf{x}}
\newcommand{\1}[0]{\mathbb{I}}

\title{Anon: Extrapolating Adaptivity Beyond SGD and Adam
}

\author{Yiheng Zhang$^{1}$ \quad Kaiyan Zhao$^{1}$ \quad Shaowu Wu$^{1}$ \quad Yiming Wang$^{2}$ \\
\textbf{Jiajun Wu}$^{3}$ \quad \textbf{Leong Hou U}$^{2}$ \quad \textbf{Steve Drew}$^{3}$ \quad \textbf{Xiaoguang Niu}$^{1}$\thanks{Corresponding author. Email: \texttt{xgniu@whu.edu.cn}} \\[4pt]
$^{1}$Wuhan University \quad $^{2}$University of Macau \quad $^{3}$University of Calgary \\
}

\iclrfinalcopy

\begin{document}

\maketitle
\lhead{}

\begin{abstract}
Adaptive optimizers such as Adam and non-adaptive methods like SGD exhibit distinct generalization capabilities across different architectures. Prior tunable optimizers attempt to bridge this gap by strictly interpolating between SGD and Adam, effectively confining adaptivity within the 0-to-1 bound. However, this restricted interpolation is fundamentally insufficient: we reveal that optimal adaptivity often requires extrapolation, such as negative adaptivity for classical CNNs and adaptivity of at least one ($\gamma \geq 1$) for Transformers. Extrapolating adaptivity theoretically violates the strict non-decreasing pre-conditioner assumption, often leading to divergence in existing methods. To break this barrier, we propose Anon, an optimizer that achieves fully continuous adaptivity extrapolation across the entire real-number spectrum. To guarantee provable stability in these out-of-bound regimes, we introduce Incremental Delay Update (IDU), a novel mechanism that bypasses hard max-tracking strategies. We theoretically establish Anon's convergence in both convex and non-convex settings. Empirically, by exploring previously unreachable adaptivity landscapes, Anon demonstrates highly competitive and scalable performance among state-of-the-art element-wise optimizers on representative image classification, diffusion, and large language modeling tasks. Our code is available at \href{https://anonymous.4open.science/r/Anon-6511/}{https://anonymous.4open.science/r/Anon-6511/}.
\end{abstract}

\section{Introduction}

Modern deep learning models rely heavily on optimization algorithms for effective training. Despite the wide success of adaptive optimizers such as Adam \citep{kingma2014adam} in large-scale models like diffusion networks \citep{nichol2021improved} and large language models (LLMs) \citep{brown2020language,touvron2023llama}, they are often outperformed by non-adaptive methods such as SGD \citep{robbins1951stochastic} in classical architectures like CNNs \citep{wilson2017marginal}. Prior efforts attempt to bridge this generalization gap by interpolating between the behaviors of SGD and Adam \citep{chen2018closing, ozkara2024mada, Savarese_2021_CVPR}. However, large-scale independent benchmarks reveal that achieving the ``best of both worlds'' remains highly elusive in practice \citep{pmlr-v139-schmidt21a}. While interpolation-based methods offer theoretical appeal, they can still struggle to reliably match the superior generalization of well-tuned SGD on classical vision models or the efficacy of Adam on massive architectures. This empirical bottleneck raises a critical question: \emph{How can we break this interpolation barrier to achieve a higher performance upper bound across diverse architectures?}

To achieve this breakthrough, we must identify the fundamental mechanical difference between SGD and Adam: how they adjust step sizes in response to historical gradient magnitudes. We formalize this ``scale compensation degree'' as \emph{adaptivity} (\S\ref{sec:adaptivity}). Existing optimizers merely represent fixed points or safe interpolations within a restricted spectrum: SGD exhibits zero adaptivity ($A=0$), while Adam exhibits full adaptivity ($A\approx1$). Consequently, prior tunable methods strictly confine themselves to this $[0, 1]$ bound. However, we reveal that reaching a higher performance upper bound often demands \emph{extrapolated} adaptivity, such as negative adaptivity ($A < 0$) for classical CNNs to aggressively escape sharp minima, or adaptivity of at least one ($A \geq 1$) for Transformers to handle heavy-tailed gradients. 

However, naively extrapolating adaptivity introduces a severe theoretical roadblock. Without proper safeguards, moving beyond the $[0, 1]$ bound makes the pre-conditioner highly susceptible to violating the strict ``non-decreasing'' assumption \citep{reddi2019convergence}, frequently triggering divergence. Traditional workarounds, like the hard max-tracking in AMSGrad, theoretically prevent divergence but compromise empirical performance. To overcome this limitation, we propose \textbf{Anon} (\textbf{A}daptivity \textbf{N}on-restricted \textbf{O}ptimizer with a \textbf{N}ovel convergence technique). Anon is the first optimizer that enables real-valued, continuous adaptivity extrapolation ($\gamma \in \mathbb{R}$) via a novel \textbf{Incremental Delay Update (IDU)} mechanism. By replacing rigid max-tracking with a provably stable multi-scale accumulator, IDU bypasses the non-decreasing assumption, ensuring robust convergence across the entire real-number spectrum.

In summary, our main contributions are threefold:
\begin{itemize}[topsep=2pt,parsep=2pt,itemsep=1pt, leftmargin=*]
    \item \textbf{A Unified View of Extrapolated Adaptivity:} We formalize adaptivity as a mathematically rigorous property. We demonstrate that merely interpolating between SGD and Adam is insufficient, and optimal generalization fundamentally requires continuous extrapolation ($\gamma \in \mathbb{R}$).
    \item \textbf{The Anon Optimizer and IDU Mechanism:} We propose Anon powered by the novel IDU mechanism. IDU safely bypasses the strict non-decreasing pre-conditioner assumption, guaranteeing theoretical convergence matching AMSGrad and ideal Adam (under standard assumptions) even in extreme extrapolated regimes.
    \item \textbf{A Highly Competitive Element-wise Baseline:} Through evaluations on CNNs, diffusion models, and a proof-of-concept study on 1B-parameter LLMs, Anon explores previously unreachable adaptivity landscapes, establishing a robust and scalable baseline among state-of-the-art element-wise optimizers.
\end{itemize}

\section{Preliminaries}
\label{sec:preliminaries}
\subsection{Review of the Frame of Optimizers}

We focus on {\it first-order optimizers}, which are widely used to train deep learning models. To facilitate a unified understanding of their differences and commonalities, we introduce a generic framework, summarized in Algorithm~\ref{algo:optimizers}. 


\begin{wrapfigure}{r}{0.55\textwidth}
\centering
\vspace{-20pt}
\begin{minipage}{0.54\textwidth}
{
\begin{algorithm}[H]
\caption{Generic Optimizer Method Frame}
\label{algo:optimizers}
\small
\begin{algorithmic}[1]
\Require $\boldsymbol{\theta}$, $\eta$, $\{\phi_t,{\psi}_t\}^\infty_{t=1}$

\While{$\boldsymbol{\theta}_t$ not converged}
    \State $\boldsymbol{g}_t \gets \nabla f_t(\boldsymbol{\theta}_{t})$
    \State $\boldsymbol{m}_t \gets 
    \left(
    {\phi}_t(\boldsymbol{g}_{1:t,1}),
    \ldots,
    {\phi}_t(\boldsymbol{g}_{1:t,d})
    \right)^\top$
    \State ${S}_t \gets 
    \operatorname{diag}
    \left(
    {\psi}_t(\boldsymbol{g}_{1:t,1}),
    \ldots,
    {\psi}_t(\boldsymbol{g}_{1:t,d})
    \right)$
    \State $\boldsymbol{\theta}_t \gets 
    \Pi_{\mathcal{F},{S}_t}
    \left(
    \boldsymbol{\theta}_{t-1}
    -
    \eta(t){S}_t^{-1}\boldsymbol{m}_t
    \right)$
\EndWhile

\end{algorithmic}
\end{algorithm}
}
\end{minipage}
\vspace{-10pt}
\end{wrapfigure}

Here, $\mathcal{F}$ denotes the convex feasible set. $\boldsymbol{\theta}\in\mathcal{F}$ is the parameter to be optimal. Define $f(\boldsymbol{\theta})$ as a real-valued scalar function to minimize. $S_t$ is a diagonal matrix where $S_{t,i,i}\coloneqq\psi_t(\boldsymbol{g}_{1:t,i})$. $\psi_t$ is the pre-conditioner function. $\prod_{\mathcal{F},S}(y) = \mathrm{argmin}_{x \in \mathcal{F}} \| S^{1/2} (x-y) \|$ denotes the projection of $y$ onto $\mathcal{F}$ under the scaling matrix $S$. The scheduler $\eta$ controls the learning rate at each step, which can be constant or scheduled via strategies such as cosine annealing \citep{loshchilov2016sgdr}. $\boldsymbol{g}_t$ is the gradient at step $t$. $\boldsymbol{m}_t$ is a vector where $\boldsymbol{m}_{t,i}\coloneqq\phi_t(\boldsymbol{g}_{1:t,i})$. The momentum operator $\phi_t: \mathbb{R}^t \rightarrow \mathbb{R}$ is typically implemented as a moving average of past gradients. The two common variants are:
\vspace{-4mm}


\vspace{-1mm}

\begingroup
\small
\begin{align}
\textit{EMA}(\boldsymbol{x}_{1:t};\beta) = \frac{1-\beta}{(1-\beta^t)}\sum_{i=1}^{t}{\beta^{t-i}x_{i}}, 
{\textit{M}}(\boldsymbol{x}_{1:t};\beta) = \sum_{i=1}^{t}\beta^{t-i}x_{i},
\label{eq:phi}
\end{align}
\endgroup
\vspace{-4mm}

where \textit{EMA} denotes the exponential moving average with bias correction. \textit{M} refers to the classical momentum without normalization. Both operators serve to smooth the gradient history. 

\textbf{Since the smoothing behavior of $\phi$ is similar across optimizers, the key differentiator lies in the design of the pre-conditioner $\psi_t: \mathbb{R}^t \rightarrow \mathbb{R}_+$, which directly modulates the effective step size and direction of each parameter and shapes the optimizer's convergence, stability, and adaptability to heterogeneous parameter landscapes.} Thus, we focus our subsequent analysis on the properties and effects of $\psi$. We summarize the designs of $\phi$ and $\psi$ for representative optimizers in Table~\ref{tab:pre-conditioner}.

\begin{table}[t]
    \centering
    \caption{
    Summary of momentum functions and pre-conditioners for representative optimizers \citep{polyak1964some, luo2019adaptive, zhuang2020adabelief}. For full expressions of complex terms ($A_t^\text{AMSGrad}$, $A_t^\text{AdaBound}$, $A_t^\text{AdaBelief}$, $A_t^\text{Anon}$), please refer to Table~\ref{suptab:adaptivity} of Appendix~\ref{supsec:adaptivity}.
    }
    \vspace{2pt} 
    \setlength{\tabcolsep}{8pt}
    \begin{small}
    \begin{sc}
    \begin{tabular}{cccc}
    \toprule
      $\textbf{Optimizer}$ & $\boldsymbol{\phi}_t({x})$  & $\boldsymbol{\psi_t}({x})$ & $\boldsymbol{A_t(\psi,{x})}$\\ \hline
        SGD & $x_t$ & $1$ & $0$\\ \hline
        SGDM & $\textit{M}(\boldsymbol{x};\beta)$ & $1$ & $0$\\ \hline
        RMSProp & $x_t$ & $\sqrt{\textit{EMA}(\boldsymbol{x}^2;\beta_2)}+\epsilon$ &
        $\frac{1}{1+{\epsilon}/{\sqrt{\textit{EMA}(\boldsymbol{x}^2;\beta_2)}}}$\ ($\approx 1$)
        \\ \hline
        Adam & $\textit{EMA}(\boldsymbol{x};\beta_1)$ & $\psi_t^\text{RMSProp}$ & $A_t^\text{RMSProp}$\\ \hline
        AMSGrad & $\phi_t^\text{Adam}$ & $\max_{i\in[t]}\{\psi_i^\text{RMSProp}\}$ & $[0,1]$\ ($\approx 1$) \\ \hline
        AdaBound & $\phi_t^\text{Adam}$ & \makecell{$\text{Clip}(\psi_t^\text{RMSProp},$\\ $f_l(t),f_u(t))$} & \makecell{$[0,1]$\\($1\rightarrow 0$)}\\ \hline
        AdaBelief & $\phi_t^\text{Adam}$ & \makecell[l]{$\Big(\textit{EMA}((\boldsymbol{x}-\boldsymbol{\phi}^{\text{Adam}})^2$ \\ $+\epsilon/(1-\beta_2);\beta_2)\Big)^{1/2}+\epsilon$} & $[0,1]$\ ($\approx 1$)\\ \hline
        Anon & $\phi_t^\text{Adam}$ & $\psi_t^\text{Anon}$ \eqref{eq:Anon} & $\approx \gamma$\\ \bottomrule
    \end{tabular}
    \end{sc}
    \end{small}
    \vspace{-4mm} 
    \label{tab:pre-conditioner}
\end{table}

As shown in Table~\ref{tab:pre-conditioner}, if the bias correction factor $1/(1-\beta^t)$ is omitted in Exponential Moving Average (EMA), it essentially reduces to a classical momentum $M$ scaled by a constant coefficient $1-\beta$. Therefore, for the remainder of this paper, we will primarily focus on analyzing the characteristics of the preconditioner $\psi$, assuming a shared momentum $\phi$ across all optimizers unless otherwise specified.

Extensive empirical evidence has shown that SGD and SGDM often generalize better than Adam on classical architectures such as ResNet \citep{he2016deep}, whereas Adam typically outperforms SGD on more complex architectures such as transformers. Many hypotheses have been proposed, from Adam escaping saddle points more efficiently \citep{staib2019escaping} to SGD's bias toward flatter minima \citep{wilson2017marginal}, but a unified measure of this scaling behavior is lacking. \textbf{We therefore formally define an adaptivity metric to mathematically distinguish the scaling mechanisms of SGD and Adam and guide the design of extrapolative optimizers.} We study how adaptivity affects optimization in \S~\ref{sec:Ainfluence}, after giving its formal definition.

\subsection{The Adaptivity of Existing Optimizers}
\label{sec:adaptivity}

We formalize the concept of adaptivity based on the framework described in Algorithm~\ref{algo:optimizers}.
\begin{definition}
\label{definition:adaptivity}
Suppose the pre-conditioner $\psi_n$ is continuous and strictly positive, i.e., $\psi_n: \mathbb{R}^n \rightarrow \mathbb{R}_{>0}$. For any optimizer following Algorithm~\ref{algo:optimizers}, we define the adaptivity $A$ of its pre-conditioner $\psi$ as
\begin{align}
A_n(\psi,\boldsymbol{x}_{1:n}) = \nabla_k \ln \psi_n(k\boldsymbol{x}_{1:n})\big|_{k=1}.
\end{align}

Furthermore, we define two pre-conditioners $\psi$ and $\psi'$ to be equivalent if and only if $A_n(\psi,\boldsymbol{x}_{1:n}) = A_n(\psi',\boldsymbol{x}_{1:n})$ for all $\boldsymbol{x}_{1:n} \in \mathbb{R}^n$ and $n \in \mathbb{N}_+$.
\end{definition}

\paragraph{Intuition Explanation} $A$ measures the pre-conditioner's ``response'' to changes in the gradient's scale: $A=0$ (like SGD) means the pre-conditioner is ``non-reactive'' and ignores the overall gradient scale, while $A \approx 1$ (like Adam) means it is ``compensatory'', adjusting itself with strength $1$ to offset gradient-scale changes. Notably, $A$ depends not only on the functional form of $\psi$, but also on the gradient history $\boldsymbol{g}_{1:t}$: pre-conditioning is inherently dynamic, and even for a fixed $\psi$ its adaptivity can vary as the gradient distribution evolves. We also introduce an equivalence notion: two optimizers using different $\psi$ functions may still be essentially equivalent from an adaptivity perspective.

\begin{theorem}
\label{theorem:class}
If $\psi$ and $\psi'$ are from the same equivalence class, there exists $f: \mathbb{N}_+\rightarrow \mathbb{R}_+$ such that $\psi_n(\boldsymbol{x}_{1:n})=\psi'_n(\boldsymbol{x}_{1:n})f(n)$ for any $\boldsymbol{x}_{1:n}\in\mathbb{R}^n$ and $n\in\mathbb{N}_+$.
\end{theorem}

\paragraph{Decoupling from Scheduler}
Theorem~\ref{theorem:class} shows that if two pre-conditioners yield the same adaptivity for any input, then they are equivalent. Specifically, if there exists a scheduler adjustment that can eliminate the difference between two pre-conditioners (e.g., $\psi' = k\psi$ corresponds to $\eta'(t) = k\eta(t)$), we regard them as equivalent strategies. The proof of Theorem~\ref{theorem:class} is deferred to Appendix~\ref{supproof:1}.

Based on these definitions, we can characterize the adaptivity of several widely used optimizers: For SGD(M), the adaptivity is $A=0$ in all dimensions, indicating no explicit scaling of the loss landscape. In contrast, for Adam and its variants (e.g., RMSProp, AdaBelief), the adaptivity is approximately $A=1$, as the contribution of the small $\epsilon$ term is negligible compared to the accumulated gradient statistics most of the time. A more intricate case is AdaBound \citep{luo2019adaptive}, whose adaptivity transitions dynamically from $A\approx1$ toward $A\approx0$ as training proceeds. Specifically, AdaBound clamps the pre-conditioner $\psi_t$ between shrinking bounds $\eta_l(t)$ and $\eta_u(t)$:

\vspace{-4mm}
{
\small
\begin{align}
\label{eq:adabound}
    A_t(\psi^\text{AdaBound},\boldsymbol{x})=
    \left\{
        \begin{array}{ll}
            A_t^\text{RMSProp}, & \text{if } \eta_l < \psi_t^\text{RMS} < \eta_u, \\
            0, & \text{otherwise}.
\end{array}
    \right.
\end{align}
}
\vspace{-4mm}

As the bounds tighten over time, AdaBound behaves increasingly like SGD. This is supported by both evidence from \citet{zhuang2020adabelief} and our experiments (Table~\ref{table:diffusion}), which indicates that AdaBound struggles in tasks such as GAN and diffusion model training, where high adaptivity is critical. These observations suggest the following: Optimizers with $A=0$ (e.g., SGD) tend to generalize better on classical architectures such as CNNs, while those with $A=1$ (e.g., Adam) perform better in complex modern architectures. However, whether $A=0$, $A=1$, or other values yield better performance remains an open question, which we explore in the next section.

\subsection{The Optimal Adaptivity for Tasks}
\label{dilemmaofA}
We argue that the optimal adaptivity is inherently task-dependent and often lies beyond the standard range. While mainstream optimizers typically operate at $A \approx 1$ (Adam) or $A \approx 0$ (SGD), different tasks, ranging from sharp-minima Transformers to flat-minima ResNets, may favor adaptivity values far outside the $[0, 1]$ interval (i.e., $A < 0$ or $A > 1$).

\begin{figure}[t]
\centering
    \includegraphics[width=0.97\linewidth]{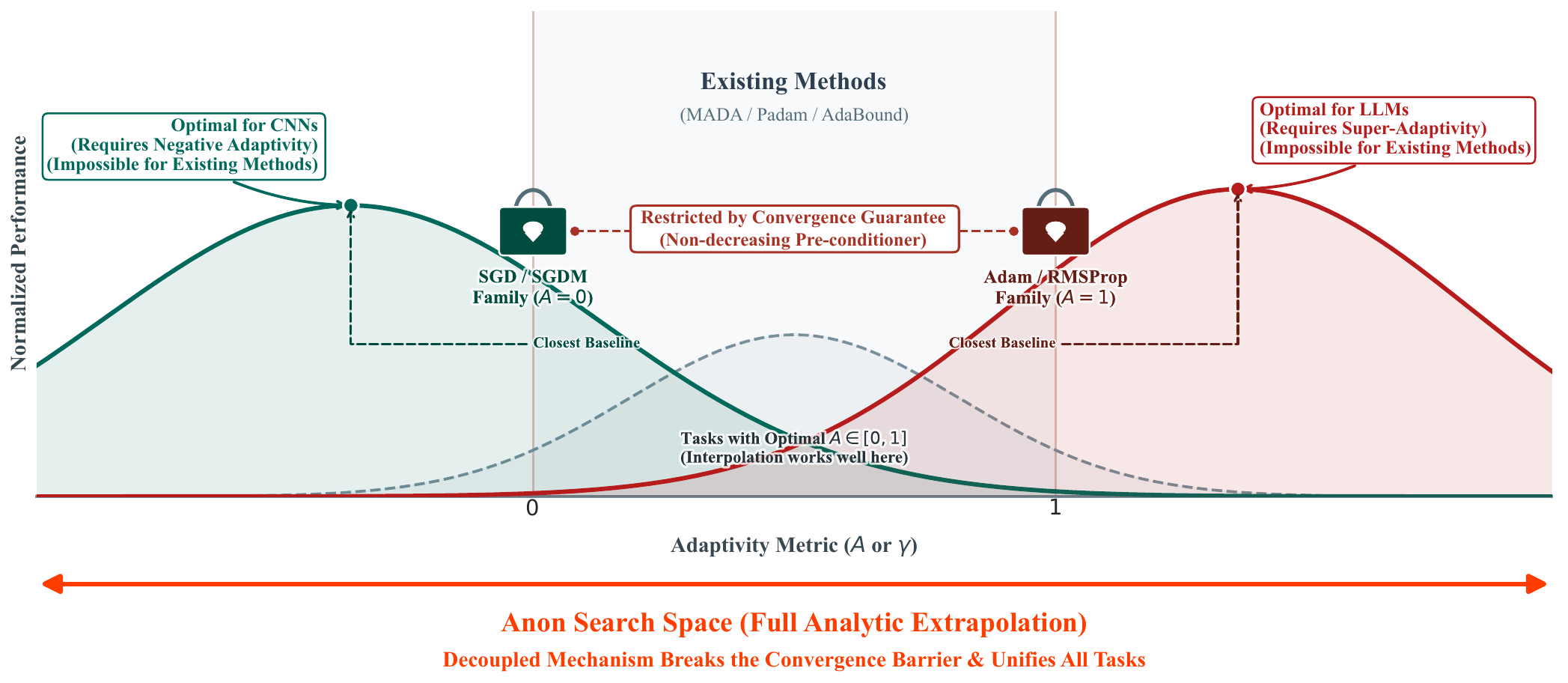}
    \caption{Visualizing the Dilemma: The distribution of existing optimizers' adaptivity and the preference of different tasks. Convergence constraints lock existing optimizers in suboptimal regimes.}
    \label{fig:existing method}
\end{figure}

However, as we show in Figure~\ref{fig:existing method}, existing methods are theoretically constrained to interpolation. Approaches like Padam \citep{chen2018closing} (whose exponent $p$ is merely the value our measure reads out on its pre-conditioner, $A\approx2p$, as listed in Appendix~\ref{supsec:adaptivity}) and MADA \citep{ozkara2024mada} attempt to tune adaptivity, but they are fundamentally restricted to the convex hull of SGD and Adam. Specifically, Padam's exponent is confined to $p\le 1/2$ (i.e., $A\in[0,1]$), since $p>1/2$ breaks its convergence condition and diverges in practice \citep{chen2018convergence}; MADA tunes the mixing coefficients of base optimizers via hyper-gradient descent, remaining within their convex hull by construction. Since the endpoints of this interval (SGDM and Adam) are already covered by our baselines, we do not benchmark these interpolants separately. Extending these methods for extrapolation is non-trivial because the convergence of most adaptive optimizers relies on the critical non-decreasing assumption:

\vspace{-5mm}{
\small
\begin{align}\frac{\psi_{t+1}(g_{1:t+1,i})}{\eta(t+1)} \geq \frac{\psi_t(g_{1:t,i})}{\eta(t)}, \quad \forall i \in [d], \forall t \in \mathbb{N}_+,\label{assumption:convergence}\end{align}
}

which guarantees stability even in worst-case scenarios. Naively constructing an optimizer with negative adaptivity (e.g., via negative powers $\psi = (\psi^\text{Adam})^\gamma$ with $\gamma < 0$) typically causes the pre-conditioner value to decrease over time, violating Eq.~\ref{assumption:convergence} and leading to divergence. Prior works \citep{chen2018closing,chen2018convergence} further show that Padam suffers from divergence in practice when adaptivity extends beyond $[0,1]$. Consequently, the literature still lacks a unified framework that stably extrapolates adaptivity to match diverse task geometries without breaking convergence guarantees.

\section{Extend to All Real Numbers}
\label{sec:extend}
\subsection{Adaptivity Tunable Optimizer and Beyond}
\label{sec:Anon}

\begin{wrapfigure}{r}{0.58\textwidth}
\centering
\vspace{-20pt}
\begin{minipage}{0.56\textwidth}
\begin{algorithm}[H]
\caption{The Anon Optimizer}
\label{algo:anon}
\small
\begin{algorithmic}[1]
\Require $\eta$, $\beta_1$, $\beta_2$, $\epsilon$, $\gamma$
\State \textbf{Initialize} $\boldsymbol{\theta}_0$, 
$\boldsymbol{m}_0 \gets \boldsymbol{0}$, 
$\boldsymbol{s}_0 \gets \boldsymbol{0}$, 
$t \gets 0$, 
$k \gets -1$

\While{$\boldsymbol{\theta}_t$ not converged}
    \State $t \gets t+1$
    \State $\boldsymbol{g}_t \gets \nabla f_t(\boldsymbol{\theta}_{t})$
    \State $\boldsymbol{m}_t \gets 
    \beta_1\boldsymbol{m}_{t-1}
    +(1-\beta_1)\boldsymbol{g}_t$
    \State $\widehat{\boldsymbol{m}}_t \gets 
    \frac{\boldsymbol{m}_t}{1-\beta_1^t}$
    \State $\boldsymbol{s}_t \gets 
    \beta_2\boldsymbol{s}_{t-1}
    +(1-\beta_2)\boldsymbol{g}_t^2$

    \If{$k+1=\log_2 t$}
        \State $k \gets k+1$
        \State $\boldsymbol{\sigma}_k \gets 
        \boldsymbol{s}_t/(1-\beta_2^{\max(t/2,1)})+\epsilon$
        \State $\boldsymbol{v}_k \gets 
        \begin{cases}
        \sqrt{
        2/
        \left(
        \frac{1}{\boldsymbol{v}_{k-1}^2}
        +\boldsymbol{\sigma}_k^\gamma
        \right)
        },
        & \text{if } k>0,\\
        \boldsymbol{\sigma}_k^{-\gamma/2},
        & \text{otherwise}
        \end{cases}$
        \State $\boldsymbol{s}_t \gets \boldsymbol{0}$
        \State $\boldsymbol{V}_k \gets 
        \operatorname{diag}(v_{k,1},\ldots,v_{k,d})$
    \EndIf

    \State $\boldsymbol{\theta}_t \gets 
    \Pi_{\mathcal{F},\boldsymbol{V}_k^{-1}}
    \left(
    \boldsymbol{\theta}_{t-1}
    -
    \eta(t)\boldsymbol{V}_k\widehat{\boldsymbol{m}}_t
    \right)$
\EndWhile

\end{algorithmic}
\end{algorithm}
\end{minipage}
\vspace{-15pt}
\end{wrapfigure}

In \S\ref{sec:adaptivity} and \S\ref{dilemmaofA}, we have shown that extending adaptivity beyond $[0,1]$ could be beneficial. However, achieving tunable adaptivity across all real numbers while ensuring convergence remains challenging. We propose a new technique called \textit{incremental delay update (IDU)}, which can ensure the convergence of an optimizer regardless of the value of its adaptivity. We will elaborate the technique in \S\ref{sec:IDU}. Leveraging this technique, we design a novel optimizer \textit{Anon} (\textbf{A}daptivity \textbf{N}on-restricted \textbf{O}ptimizer with a \textbf{N}ovel convergence technique) with tunable adaptivity and extend the allowable range of adaptivity to all real numbers. 

The pseudocode of Anon is presented in Algorithm~\ref{algo:anon}, and all the operations are element-wise. Here, $\widehat{\boldsymbol{m}}_t$ corresponds to $\boldsymbol{m}_t$ in Algorithm \ref{algo:optimizers}. $\mV_k$ corresponds to $\mS_t^{-1}$ in Algorithm \ref{algo:optimizers}. $\boldsymbol{s}_t,\boldsymbol{\sigma}_k,\boldsymbol{v}_k$, and $k$ are intermediate variables. $\gamma$ is a hyperparameter to adjust adaptivity $A$. $\epsilon$ is a small hyperparameter to avoid division by 0. $\beta_1$ and $\beta_2$ are hyperparameters for \textit{EMA}, $0\le \beta_1,\beta_2 < 1$, typically set as 0.9 and 0.999. Let $\{a_n\}$ be an increasing sequence and $a_1=1$ (specifically, let $a_0=0$). Let ${\tilde{a}_n}=\sum_{i>0}\1_{a_i\le n}$, so $\tilde{a}_1=1$. The pre-conditioner of Anon can be written as \eqref{eq:Anon} ($\beta_3=0.5, a_n=2^{n-1}$).

\vspace{-5mm}
{
\small
\begin{equation}    
\label{eq:Anon}
\psi^\text{Anon}_t(\boldsymbol{x}) = \Bigg( \sum_{j=1}^{\tilde{a}_t}\beta_3^{\tilde{a}_t-j}(1-\beta_3\1_{j> 1}) \cdot\textit{EMA}^\gamma(\boldsymbol{x}_{{a_{j-1}+1}:{a_j}}^2+\epsilon;\beta_2) \Bigg)^{1/2} .
\end{equation}
}
\vspace{-2mm}

\begin{theorem}
    \label{theorem:AofAnon}
    For the optimizer Anon described in Algorithm~\ref{algo:anon}, the adaptivity of Anon in i-th dimension is $\in[\gamma(1-k),\gamma)$, or $(\gamma,\gamma(1-k)]$ if $\gamma<0$, where $k=\epsilon/(\min_{j\in[\tilde{a}_t]} \textit{EMA}(\vg^2_{a_{j-1}+1:a_j,i};\beta_2)+\epsilon)\in(0,1]$.
\end{theorem}

According to Theorem~\ref{theorem:AofAnon}, since we also set a small $\epsilon$ by default, we can adjust the adaptivity $A$ of Anon by adjusting the hyperparameter $\gamma$ ($A\approx\gamma$). The proof of Theorem~\ref{theorem:AofAnon} is shown in Appendix~\ref{supproof:2}.

\subsection{How Adaptivity Influences Behaviors of Optimizers}
\label{sec:Ainfluence}

\begin{figure}[t]
\centering
    \begin{subfigure}[b]{0.31\textwidth}
       \includegraphics[width=\linewidth]{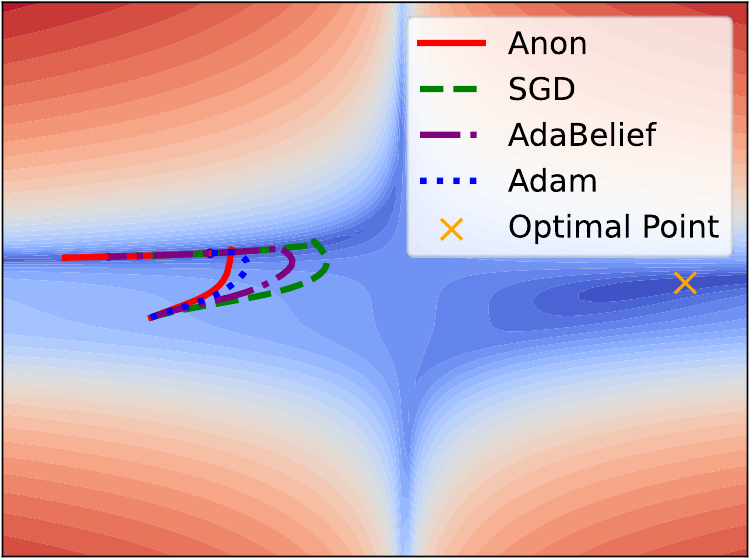}
        \caption {\small{$\gamma=1.5$}}
    \end{subfigure}
    \hfill
    \begin{subfigure}[b]{0.31\textwidth}
        \includegraphics[width=\linewidth]{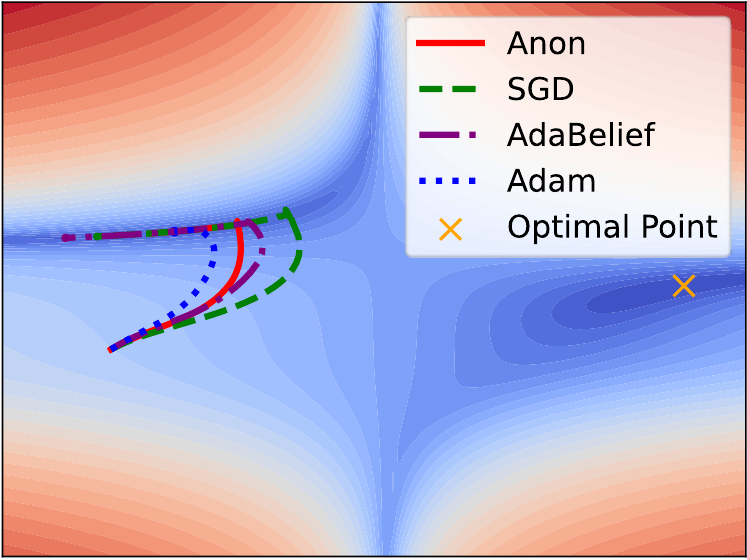}
        \caption{\small{$\gamma=0.5$}}
    \end{subfigure}
    \hfill
    \begin{subfigure}[b]{0.31\textwidth}
        \includegraphics[width=\linewidth]{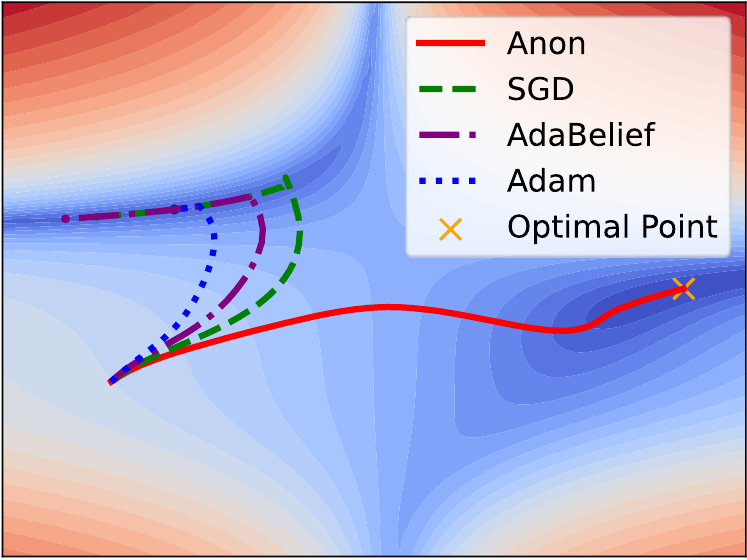}
        \caption{\small{$\gamma=-0.5$}}
    \end{subfigure}
    \caption{Trajectories of {SGDM}, {Adam}, {AdaBelief}, and {Anon}. The color change from deep red to deep blue represents the loss from high to low. And the loss landscape displayed is the result of scaling by Anon. More empirical experiments are shown in Appendix~\ref{sup:effct of A} and~\ref{supsec:empirical}.}
    \label{fig:toy}
\vspace{-5mm}
\end{figure}

\paragraph{Empirical Validations} To show how adaptivity influences the behaviors of optimizers, we conduct a simple experiment in the loss function $f(x,y)=\ln(1+\text{Beale}(x,y))/10$, where Beale \citep{beale1955minimizing} is a commonly used function to test optimizer performance. We apply appropriate learning rates for SGDM, Adam, AdaBelief, and Anon, and draw the optimization trajectories. We also show the loss landscapes in the view of Anon by scaling the loss landscape according to the pre-conditioner of Anon in epoch 100. The trajectories and loss landscapes after scaling are shown in Figure~\ref{fig:toy}.

\paragraph{Effect of Scaling} By changing $\gamma$ from $1.5$ to $-0.5$, the adaptivity also changes from $1.5$ to $-0.5$ referring to Theorem~\ref{theorem:AofAnon}. When $\gamma=1.5$, Anon takes a shorter path to descend along the y-axis; when $\gamma=0.5$, the path lies between Adam and SGDM; and when $\gamma=-0.5$, Anon descends along the x-axis and arrives at the optimal point. As $\gamma$ decreases, the x-axis is increasingly scaled down relative to the y-axis, so that Anon can choose the right path to reach the optimum. This implies that Anon's optimization path in deep learning training may differ greatly from other optimizers, helping reach a new parameter region that achieves better performance.

\paragraph{The Meaning of Negative Adaptivity} Positive adaptivity typically reduces step sizes for large gradients to help escape saddle points. In contrast, negative adaptivity adopts the opposite strategy by increasing step sizes when gradients are large, which enables the optimizer to escape from sharp minima. Intuitively, higher adaptivity drives the optimization toward steeper minima, whereas lower adaptivity favors flatter regions. {\bf{}Thus, adaptivity influences the optimizer not only through its path but also by altering its preference for specific minima geometries.} This perspective implies that restricting adaptivity to fixed points like A=0 or A=1 is insufficient. Empirically, we find that negative adaptivity is more effective for classical models, while positive adaptivity remains suitable for complex architectures.

\subsection{Incremental Delay Update}
\label{sec:IDU}
As we state in \S~\ref{dilemmaofA}, it is challenging to guarantee the convergence when adaptivity is allowed to take any value. So we propose a new technique \textit{incremental delay update} (IDU), which can be seen as using a new function $U(\boldsymbol{x};\psi^\text{old})$ to replace the old pre-conditioner function $\psi$, with smoothing factor $\beta_3\in(0,1)$ over delay intervals:

\vspace{-5mm}
{
\small
\begin{equation}
    U_t(\boldsymbol{x};\psi_t^\text{old},\{a_n\},\beta_3) =  
    \Bigg(\sum_{j=1}^{\tilde{a}_t}\beta_3^{\tilde{a}_t-j}\Big(1-\beta_3\1_{j>1}\Big)
    \cdot\Big({{\psi_{a_j-a_{j-1}}^\text{old}}}(\boldsymbol{x}_{a_{j-1}+1:a_j})\Big)^2\Bigg)^{1/2}\text{ .}
\end{equation}
}
\vspace{-2mm}

Lines 9\textasciitilde15 of Algorithm~\ref{algo:anon} are the recursive formulas for IDU used in Anon where $\beta_3=0.5$, $a_n=2^{n-1}$ and $\psi^{\text{old}}=\textit{EMA}^\gamma(\vx^2+\epsilon;\beta_2)$. IDU updates the pre-conditioner using accumulated gradient information {\bf{only at specific, delayed steps}}. Intuitively, IDU holds the pre-conditioner fixed within geometrically growing blocks, bounding its cumulative variation by $\sum_k\sqrt{a_{k+1}}=O(\sqrt{T})$; per-step updates would instead incur an $O(T^{3/2})$ drift and break the bound. This strategy confines unpredictable oscillations within a manageable range, thereby ensuring theoretical convergence while still permitting the pre-conditioner to change non-monotonically. Theorems~\ref{theorem:CDU_convex} (convex) and \ref{theorem:CDU_nonconvex} (non-convex) establish the convergence of Anon, with proofs in Appendices~\ref{supproof:3} and \ref{supproof:4}.

\begin{theorem}{(Convergence analysis for online convex optimization)}
\label{theorem:CDU_convex}
Let $\{\theta_t\}$ and $\{v_k\}$ be the sequence obtained by Algorithm~\ref{algo:anon}, $\gamma\in\mathbb{R}$, $\beta_1 \in [0,1)$, $\beta_2 \in [0,1)$, $\beta_{1,t+1}\in [0,\beta_{1,t}]$, $\beta_{1,1}=\beta_1$, $\eta(t) = \frac{\eta_0}{\sqrt{t}}$, for $\forall t \in [T]$. Assume that $\Vert x-y \Vert_\infty \leq D_\infty$ for $\forall x,y\in \mathcal{F}$. Suppose $f(\theta)$ is a convex function, $ \Vert  g_t  \Vert_\infty \leq G_\infty$ , for $\forall t \in [T]$, $\theta \in \mathcal{F}$. Let $C_l=\min(G_\infty^{-\gamma},\epsilon^{-\gamma/2})$, $C_u=\max(G_\infty^{-\gamma},\epsilon^{-\gamma/2})$, where $\epsilon\in\mathbb{R_+}$ is a small constant set in Algorithm~\ref{algo:anon}. The optimal point of $f$ is denoted as $\theta^*$. For $\{ \theta_t \}$ generated by Anon, there is a bound on the regret: 
\vspace{-7mm}

\begingroup
\begin{small} 
\begin{align}
\sum_{t=1}^T [f_t(\theta_t) - f_t(\theta^*)] 
 &\leq \frac{dD_\infty^2C_l^{-1}}{(1 - \beta_{1})\eta_0} \bigg(\sqrt{T}+ \sum_{k=1}^{\tilde{a}_T-1} \sqrt{a_{k+1}} \bigg)
 + \frac{dD_\infty^2}{2 C_l\eta_{1} (1 - \beta_{1})} + \frac{dD_\infty G_{\infty}}{1-\beta_{1}}\sum_{t=1}^T{\beta_{1,t}} 
  \nonumber\\
 &\quad  + \frac{dG_\infty^2 C_u\eta_0}{1-\beta_{1}}\sqrt{T} 
  + \sum_{t=1}^{T-1} \frac{\beta_{1,t+1}{\mathbb{I}}_{\beta_{1,t+1}>\beta_{1,t}}dD_\infty^2}{2 C_l\eta_{t+1} (1 - \beta_{1})^2} 
\end{align}
\end{small}
\endgroup

\end{theorem} 

\begin{corollary}
\label{corollary:CDU_1}
Suppose $\beta_{1,t} = \beta_1 \lambda^t,\ \ 0<\lambda<1$ in Theorem~\ref{theorem:CDU_convex}, then we have:
\vspace{-7mm}

\begin{small}
\begin{equation}
\sum_{t=1}^T [f_t(\theta_t) - f_t(\theta^*)] 
 \leq \frac{dD_\infty^2C_l^{-1}}{(1 - \beta_{1})\eta_0}\bigg(\sqrt{T}+ \sum_{k=1}^{\tilde{a}_T-1} \sqrt{a_{k+1}}\bigg)  \\
 + \frac{dD_\infty^2}{2 C_l\eta_{1} (1 - \beta_{1})} + \frac{dD_\infty G_{\infty}\beta_{1}}{(1-\beta_{1})(1-\lambda)} 
 + \frac{dG_\infty^2 C_u\eta_0}{1-\beta_{1}}\sqrt{T}  
\end{equation}
\end{small}


It implies the regret of Anon is upper-bounded by $O(\sqrt{T})$ for convex case when $a_n=2^{n-1}$.
\end{corollary}

\begin{theorem}{(Convergence analysis for non-convex stochastic optimization)}
\label{theorem:CDU_nonconvex}
The update of $\theta_t$ can be described as $\theta_{t+1}=\theta_t-\eta_{t}V_{\tilde{a}_t}m_t$, and $m_t=\beta_1m_{t-1}+(1-\beta_1)g_t$. Under the assumptions:
\begin{itemize}[leftmargin=*, itemsep=0mm, topsep=0pt]
    \item $f$ is differentiable and $f^*\leq f \leq F$. $\nabla f(x)$ is L-Lipschitz continuous, $i.e.\ \ \Vert \nabla f(x) - \nabla f(y)  \Vert \leq L  \Vert x-y  \Vert,\ \forall x,y$.
    \item The noisy gradient is unbiased and its infinity norm is bounded by N, $i.e.\ \ \mathbb{E} g_t = \nabla f(x)$, $\Vert g_t \Vert_\infty \leq N$.
\end{itemize}
The hyperparameters are set as: $\eta_{t}=\eta_0t^{-p}$, $\eta_0>0$, $p\in (0,1)$ where the bounds are $C_lI\preceq V_{\tilde{a}_t}\preceq C_uI$, and $0<C_l<C_u$ ($A\preceq B$ means $B-A$ is a positive semi-definite matrix). And $\epsilon$ and $N$ ensure $C_l$ and $C_u$ exist. Here we consider the unconstrained case $\mathcal{F}=\mathbb{R}^d$. For sequence $\{\theta_t\}$ generated by Anon, we have: 
\vspace{-5mm}

\begin{small}
\begin{equation}
\frac{1}{T} \sum_{t=1}^T \mathbb{E}\Big \Vert \nabla f(x_t)\Big \Vert^2 
 \leq \frac{1}{\eta_0 C_l} T^{p -1}\cdot \\
 \Big(F - f^*  
 \quad + K \int_{1}^{T} t^{-2p}\, \mathrm{d} t + J +K \Big),
\end{equation}
where $x_t\equiv\theta_t$ denotes the iterates, the expectation is taken over the stochastic gradient noise, and
\begin{align}
J = \frac{\beta_1^2d N^2}{4L (1-\beta_1)^2} + \frac{3dN^2 \eta_0 C_u}{1-\beta_1} \sum_{k=1}^{\tilde{a}_T} {(a_k-\1_{k\ne 1})}^{-p}, 
K = \left( \frac{1}{1-\beta_1}+\frac{1}{2} \right) L \eta_0^2 N^2 C_u^2d.
\end{align}
\end{small}
\end{theorem}

Theorem~\ref{theorem:CDU_nonconvex} shows when $p=0.5$ and $a_n=2^{n-1}$, Anon has a convergence rate of $O(\ln T/\sqrt{T})$ for non-convex cases. Note that the convergence rates shown in Theorem~\ref{theorem:CDU_convex} and  Theorem~\ref{theorem:CDU_nonconvex} are the same as mainstream adaptive optimizers under the strong assumption \eqref{assumption:convergence} or using the technique of AMSGrad. And the assumptions and boundedness conditions are standard in the literature and consistent with those adopted in previous works like \citet{luo2019adaptive} and \citet{zhuang2020adabelief}.

\stitle{Better Noise Robustness} 
Other convergence-guarantee techniques must enforce \eqref{assumption:convergence} by alternative means, so early-stage noise can greatly influence their performance and prevent them from using the latest gradient information. In contrast, IDU is the first technique that guarantees convergence while allowing \eqref{assumption:convergence} to be violated, offering Anon (IDU) better noise robustness and flexibility. To verify this, we compare Anon (IDU) with AMSGrad. Slightly different from Table~\ref{tab:pre-conditioner}, AMSGrad is usually implemented in practice in the form: $\max_{i\in[t]}\{\psi_i^\text{RMSProp}\sqrt{1-\beta_2^i}\}/\sqrt{1-\beta_2^t}$ (we apply in experiments). Regardless of the implementation form, AMSGrad's strategy of persistently applying the max operation is highly susceptible to noise interference. We conduct empirical experiments to prove it, and the relevant function settings include:
\begin{align}
\label{eq:FunctionOfAMSGrad}
    f_t(x)=
    \left\{
        \begin{array}{ll}
            1010x ,& \text{if }t\text{ mod } 101=1 \\
            -10x,   & \text{otherwise}
        \end{array}
    \right.
    ,
    N_t=
    \left\{
        \begin{array}{ll}
            500/\mathrm{e}^{t-1} ,& \text{if }t\text{ mod } 2=1 \\
            -500/\mathrm{e}^{t-1},   & \text{otherwise}
        \end{array}
    \right.
\end{align}

with the constraint set $\mathcal{F} = [-1, 1]$. The $f_t(x)$ is the example provided in \citet{reddi2019convergence}, which can make Adam diverge. And $N_t$ is the noise added to the gradients $g_t$. We can observe that the noisy gradient is unbiased and its influence on gradients approaches $0$ with the increase of $t$. The results of experiments are shown in Figure~\ref{fig:noise}. Note that we set $\gamma=1$ to make the adaptivity of Anon equivalent to AMSGrad and Adam, and their other hyperparameters are the same. Therefore, we can compare the performances of the two convergence guarantee techniques fairly.

\begin{figure}[t]
\centering
    \begin{subfigure}[b]{0.235\columnwidth}
       \includegraphics[width=\linewidth]{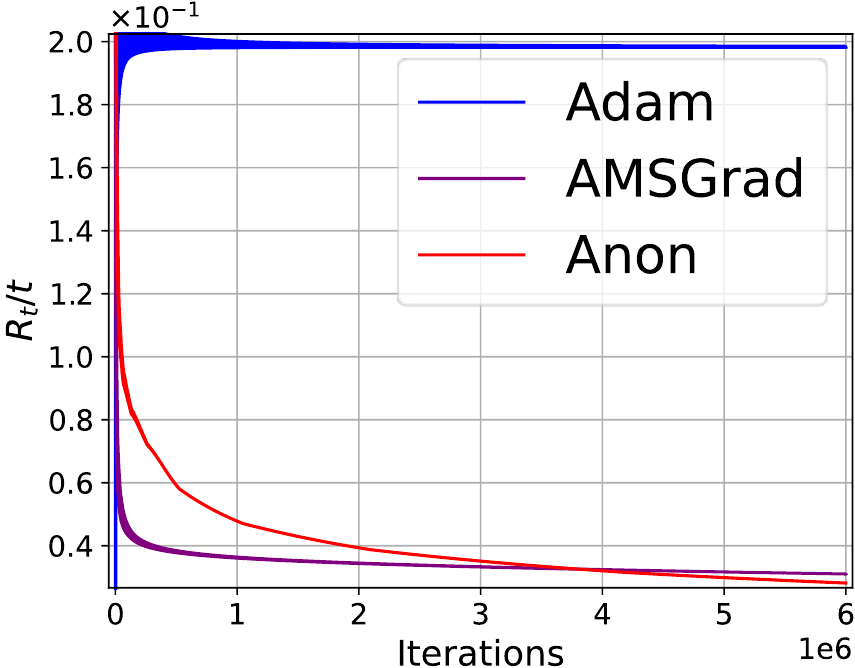}
        \caption {\small{$\beta_1=0.5, \beta_2=0.75$}}
    \end{subfigure}
    \hfill
    \begin{subfigure}[b]{0.235\columnwidth}
        \includegraphics[width=\linewidth]{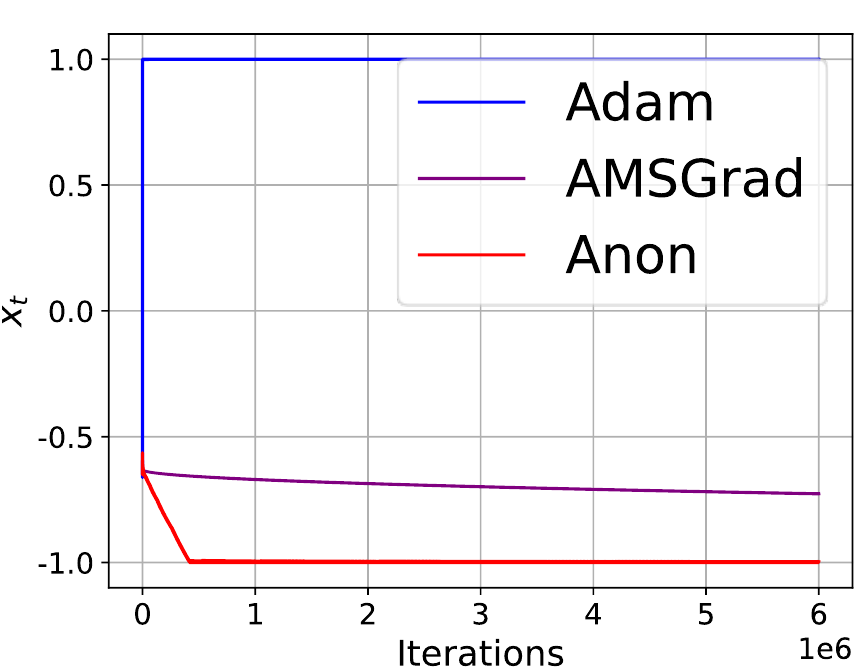}
        \caption{\small{$\beta_1=0.5, \beta_2=0.75$}}
    \end{subfigure}
    \hfill
    \begin{subfigure}[b]{0.235\columnwidth}
        \includegraphics[width=\linewidth]{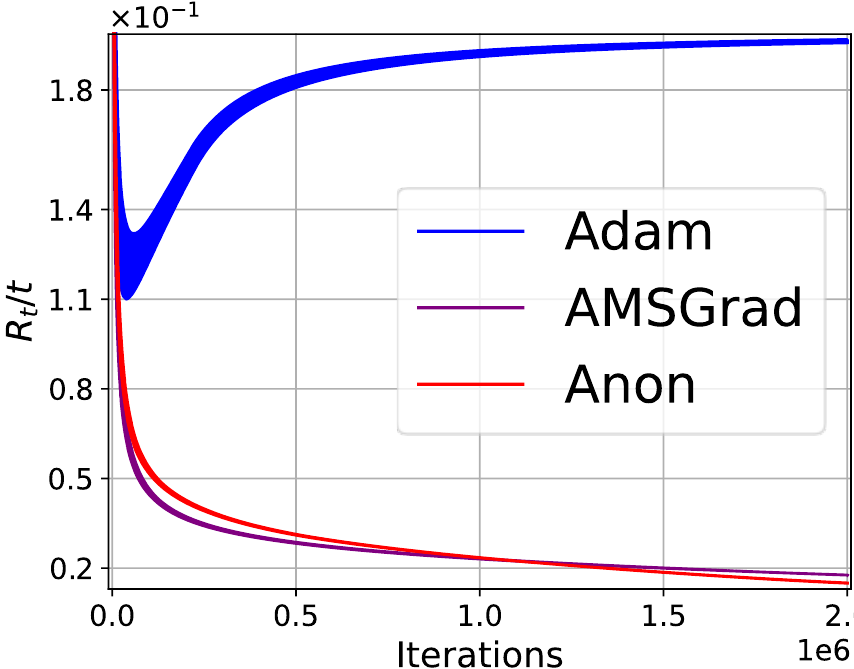}
        \caption{\small{$\beta_1=0.9, \beta_2=0.99$}}
    \end{subfigure}
    \hfill
    \begin{subfigure}[b]{0.235\columnwidth}
        \includegraphics[width=\linewidth]{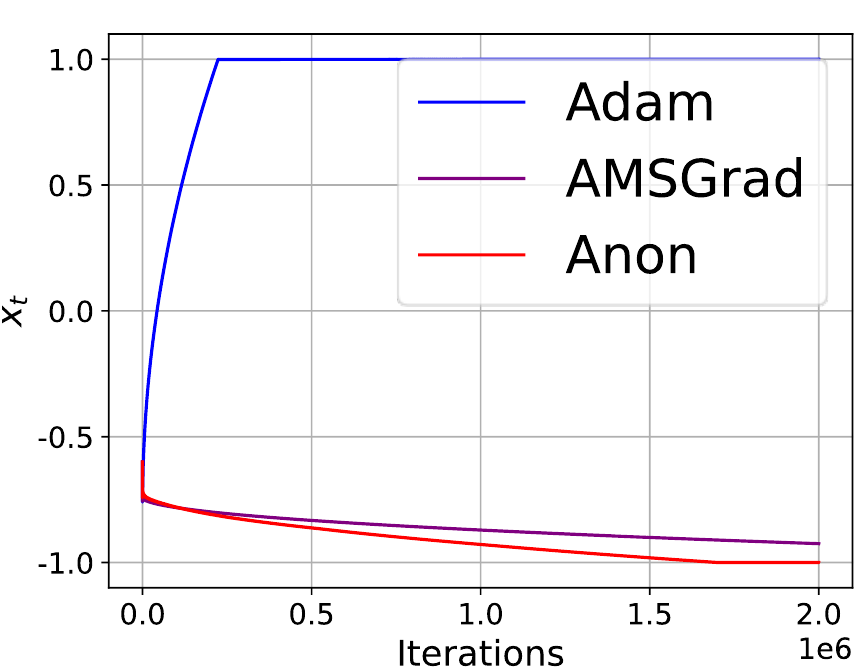}
        \caption{\small{$\beta_1=0.9, \beta_2=0.99$}}
    \end{subfigure}
    \caption{Comparison of Adam, AMSGrad, and Anon on a simple convex problem with noise. The setting of hyperparameters follows $\beta_1<\sqrt{\beta_2}$ and $\eta(t)=0.1/\sqrt{t}$ \citep{reddi2019convergence}.}
    \label{fig:noise}
\vspace{-3mm}
\end{figure}

In Figure~\ref{fig:noise}(a)(c), the regrets divided by $t$ of Anon and AMSGrad approach $0$ (converging), while Adam's approaches a constant (diverging). Figure~\ref{fig:noise}(b)(d) further shows that Anon reaches the optimal point $x=-1$ fast, whereas AMSGrad converges much slower under noise, especially when $\beta_2$ is small, proving the better noise robustness of Anon (IDU).

\section{Experiments}
\label{sec:expriments}
In this section, we compare Anon with 13 baseline optimizers (SGD(M), Adam, AdamW \citep{loshchilov2017decoupled}, Yogi \citep{zaheer2018adaptive}, AdaBound, RAdam \citep{liu2019variance}, SWA \citep{PavelIzmailov2018AveragingWL}, Lookahead \citep{zhang2019lookahead}, AdaBelief, Adai \citep{xie2022adaptive}, Lookaround \citep{zhang2023Lookaround}, Sophia \citep{liu2023sophia} and HVAdam \citep{zhang2025hvadam}) on image classification with ResNet, image generation with diffusion models, and language modeling with LLMs. For the diffusion experiments, all optimizers are trained by us under an identical environment and recipe. For the remaining tasks, following the standard evaluation protocol of prior optimizer studies \citep{luo2019adaptive,zhuang2020adabelief,xie2022adaptive}, we source the baseline metrics from their original publications, thereby comparing Anon against each baseline at its best result tuned by the original authors. Throughout, our conclusions are grounded in controlled comparisons within each table, varying only $\gamma$ or the IDU mechanism under an identical environment and recipe; a fully controlled multi-seed benchmark on CIFAR-10 (Appendix~\ref{supsec:multiseed}) further confirms the ordering with statistical significance.

\paragraph{Image Classification with CNN} We conduct experiments on ImageNet \citep{russakovsky2015imagenet} with ResNet18 and ResNet50, using the official implementations of AdaBound, AdaBelief and LookAround for exact replication. Tables~\ref{table:resnet50} and \ref{table:resnet18} report the top-1 accuracy for ResNet50 and ResNet18, respectively. We set a learning rate of $1$ for Anon, corresponding to the $0.1$ learning rate and $0.9$ momentum of SGDM, since $\textit{EMA}(\boldsymbol{x};0.9)\approx \textit{M}(\boldsymbol{x};0.9)/10$ \eqref{eq:phi}. With $\gamma=-0.1$ ($A=-0.1$), Anon surpasses SGDM ($A=0$), confirming that negative adaptivity suits classical models like CNNs. 
And the detailed setting is shown in Appendix~\ref{supsec:details}.

\vspace{-10pt}
\begin{table}[htbp]
\small
\centering
\caption{Top-1 accuracy (\%) of ResNet18 on ImageNet. AMSGradW denotes AMSGrad with decoupled weight decay. $\dag$ from~\cite{luo2019adaptive}, $\ddag$ from~\cite{liu2019variance}, $\ast$ from~\cite{zhuang2021momentum} (MSVAG, \citep{balles2018dissecting}).}
\label{table:resnet18}
\scalebox{0.95}{
\begin{tabular}{@{}l*{8}{c}@{}}
\toprule
\textbf{Anon}  & SGDM & AMSGradW & AdaBelief & AdaBound$^\dag$  & Yogi$^\dag$  & Adam$^\ddag$ & MSVAG$^\ast$ & RAdam$^\ddag$   \\ 
\midrule
\textbf{70.06} & 69.94 & 68.78 & 69.42 & 68.13 & 68.23 & 66.54 & 65.99 & 67.62  \\ 
\bottomrule
\end{tabular}
}
\end{table}


\vspace{-15pt}
\begin{table}[htbp]
\small
\centering
\caption{Top-1 accuracy (\%) of ResNet50 on ImageNet. $\dag$ from~\cite{xie2022adaptive}, $\ddag$ from~\cite{zhang2023Lookaround}, $\ast$ from~\cite{zhang2025hvadam}.}
\label{table:resnet50}
\begin{tabular}{@{}l*{7}{c}@{}}
\toprule
\textbf{Anon} & SGDM &Lookaround & Adam$^\dag$ & Adai$^\dag$ & SWA$^\ddag$ & Lookahead$^\ddag$ & HVAdam$^\ast$  \\
\midrule
\textbf{77.25} & 76.23 &76.77 & 72.87 & 76.80 & 76.78 & 76.52 & 77.22   \\ 
\bottomrule
\end{tabular}
\end{table}
\vspace{-15pt}


\paragraph{Image Generation with Diffusion Model} We conduct image generation experiments on CIFAR-10 \citep{krizhevsky2009learning} with a diffusion model, searching the learning rate in $\{0.1,0.01, ...,0.00001\}$ for AdamW, AMSGrad, Anon, SGDM, and AdaBound. The code and other hyperparameters follow the official implementation of \citet{nichol2021improved}. The results are reported in Table~\ref{table:diffusion}. Using a standard learning rate of 0.0001 (the optimal default for Adam-class methods) and $\gamma=1.01$, Anon achieves a highly competitive FID of 7.75. A slightly extrapolated adaptivity ($\gamma>1$) thus benefits complex generative models.

\begin{table}[t]
\small
\centering
\caption{FID scores of diffusion models on CIFAR-10 (lower is better).}
\label{table:diffusion}
\setlength{\tabcolsep}{6pt}
\begin{tabular}{@{}lccccc@{}}
\toprule
Adam & AMSGrad & SGDM & AdaBound & Anon$_{\gamma=1}$ & Anon$_{\gamma=1.01}$  \\ 
\midrule
9.11 & 8.12 & 12.84 & 12.13  & 8.03 & \textbf{7.75} \\ 
\bottomrule
\end{tabular}
\end{table}

\begin{table}[t]
\small
\centering
\caption{Validation loss and time on OpenWebText (Anon: $\gamma=1.1$ on GPT2-small, $\gamma=1$ on GPT2-medium).}
\label{table:llms}
\setlength{\tabcolsep}{5pt}
\begin{tabular}{@{}l ccc ccc@{}}
\toprule
& \multicolumn{3}{c}{Validation Loss $\downarrow$} & \multicolumn{3}{c}{Time (h) $\downarrow$} \\
\cmidrule(lr){2-4}\cmidrule(lr){5-7}
Model & Anon & AdamW & Sophia-G & Anon & AdamW & Sophia-G \\
\midrule
GPT2-small & \textbf{2.93283} & 2.95614 & 2.95143 & \textbf{26.17554} & 26.88118 & 28.98702 \\
GPT2-medium & \textbf{2.69017} & 2.70994 & 2.70653 & 36.91487 & \textbf{36.83633} & 41.02486 \\
\bottomrule
\end{tabular}
\end{table}

\paragraph{Language Modeling} 
We train autoregressive models on OpenWebText \citep{Gokaslan2019OpenWeb} using the official Sophia implementation \citep{liu2023sophia} and its exact experimental setup, setting $\gamma\ge1$ for Anon with the other optimizers' learning rate settings. As shown in Table~\ref{table:llms}, Anon obtains the lowest validation losses on both GPT2-small and GPT2-medium, demonstrating strong performance on LLM training. Many Adam variants are slower than Adam due to extra calculations; in contrast, Anon achieves comparable wall-clock training efficiency (Table~\ref{table:llms}). In theory, as iterations approach infinity, the average per-iteration time cost satisfies

\vspace*{-10pt}
\begin{equation}
    E{(t^\text{Adam}-t^\text{Anon})} \approx t^{\textit{vector-Div}}+t^{\textit{vector-Sqrt}}-t^{\textit{vector-Mul}} \\
     -C\frac{\log_2Iters}{Iters}>0 \ (Iters \rightarrow \infty)\ ,
    \label{eq:timecost}
\end{equation}
 and $C$ is the time cost of the operations in line 9\textasciitilde15 of Algorithm~\ref{algo:anon} per iteration. Eq.~\eqref{eq:timecost} implies a lower arithmetic cost per iteration for Anon, since vector division and square root are slower than multiplication; in practice this advantage is largely masked by memory-bandwidth and kernel-launch overheads, so the measured wall-clock times of Anon and AdamW are comparable (Table~\ref{table:llms}). IDU also amortizes the expensive vector power operation related to $\gamma\in\mathbb{R}$ (covered in $C$) to near zero.

\paragraph{Comprehensive Analysis and Robustness}
\begin{wrapfigure}{r}{0.42\textwidth}
\vspace{-10pt}
\centering
\includegraphics[width=0.95\linewidth]{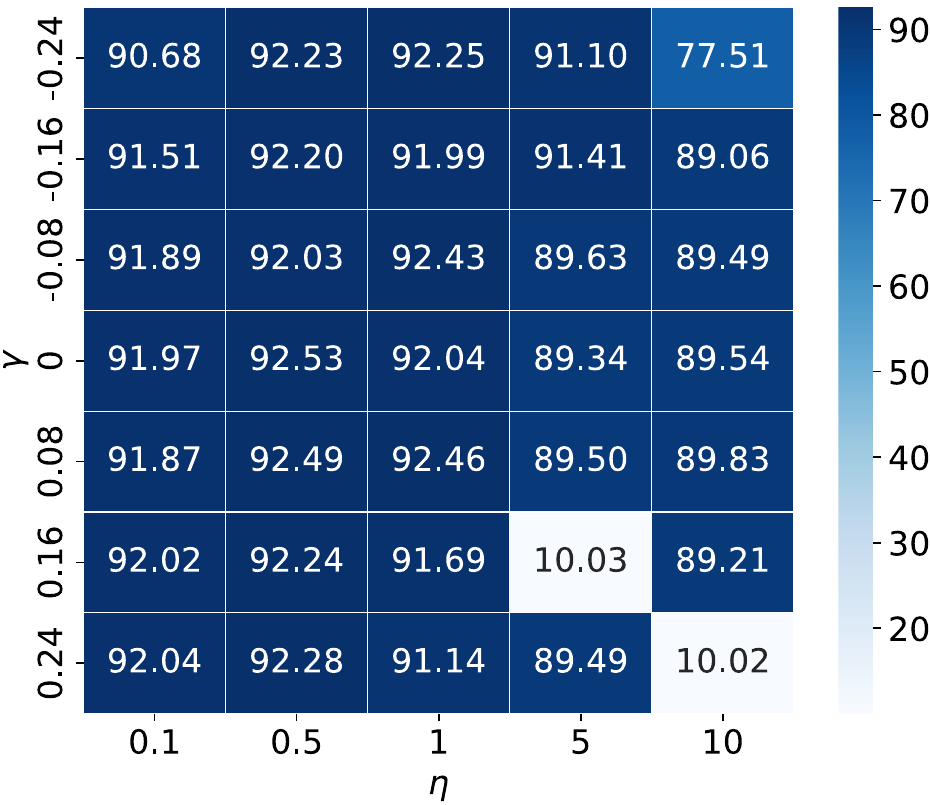}
\vspace{-5 pt}
\caption{Hyperparameter sensitivity analysis of ResNet20 on CIFAR-10\strut}
\label{fig:ablation}
\vspace{-8 pt}
\end{wrapfigure}

On CNNs, the SGDM-equivalent learning rate paired with a negative adaptivity yields better generalization and accuracy. In contrast, the Adam-equivalent learning rate with a positive adaptivity ($\gamma \geq 1$) achieves SOTA results in diffusion models and LLMs, consistent with our analysis in Section~\ref{dilemmaofA}. Results with $\gamma=1$ (Adam's adaptivity) in Tables~\ref{table:diffusion} and \ref{table:llms} differ from Adam only by IDU, isolating the \textbf{improvement brought by IDU}. Furthermore, Figure~\ref{fig:ablation} shows Anon maintains high performance across broad ranges of learning rates and $\gamma$. Notably, Anon fixes $\beta_1,\beta_2,\epsilon$ ($0.9$, $0.999$, $10^{-16}$) across all experiments with no per-task tuning, yet consistently achieves SOTA, validating its robustness. Finally, while individual margins can be modest, Anon ranks first on every task we evaluate, each time against a different runner-up, a consistency that would be unlikely if Anon carried no systematic advantage over the per-task second-best methods.

\section{Conclusion}
We propose Anon, integrating continuously tunable adaptivity with the provably stable IDU mechanism. Extensive evaluations establish Anon as a robust and competitive baseline among element-wise optimizers, supporting our hypothesis that task-specific adaptivity extrapolation, whose optima fall just outside the previously reachable $[0,1]$ interval, is a critical dimension of optimization design. IDU matches AMSGrad's convergence guarantees with better noise robustness and efficiency, and Anon offers a versatile, scalable framework compatible with standard training techniques; extensions to matrix-level and higher-order optimizers are future work.


\bibliography{example_paper}
\bibliographystyle{iclr2027_conference}

\appendix

\newpage
\section*{{\LARGE Appendix}}

\section{Limitations and Future Work} 
\label{supsec:limitation}
While we have demonstrated that adaptivity is a critical attribute for first-order optimizers, our current Adaptivity Definition~\ref{definition:adaptivity} covers the generic framework outlined in Algorithm~\ref{algo:optimizers}. A small subset of first-order optimizers, such as HVAdam, deviate from this framework and are therefore not currently covered. We aim to propose a more generalized definition of adaptivity in future work to encompass these cases. 

Additionally, due to computational resource constraints, our hyperparameter search for Anon was not exhaustive. For instance, in our diffusion model experiments, we observed that a learning rate of $0.0001$ combined with an adaptivity of $1.02$ hindered the reduction of training loss in the early stages. Conversely, a lower learning rate of $10^{-5}$ allowed for higher adaptivity values, such as $1.15$. This coupling arises because the pre-conditioner scales as $(\textit{EMA}(\vg^2)+\epsilon)^{\gamma/2}$, so increasing $\gamma$ amplifies the scale compensation and effectively enlarges the step-size magnitude whenever gradient statistics deviate from unity. As a practical rule of thumb, we recommend pairing a larger learning rate with a smaller $\gamma$ and a smaller learning rate with a larger $\gamma$, which we found to be more stable. Time constraints prevented a deeper investigation into these observations; thus, further research is required to fully exploit Anon's potential across different regimes.

When $\gamma<0$, coordinates with vanishing gradients rely on $\epsilon$ to keep the pre-conditioner $(\textit{EMA}(\vg^2)+\epsilon)^{\gamma/2}$ bounded; we use $\epsilon=10^{-16}$, which is near the precision limit of FP32. We observed no numerical instability in any experiment, but a systematic study of the interaction between $\epsilon$ and negative $\gamma$ is left to future work. We note that the condition number of the bound, $C_u/C_l$, grows as $(G_\infty^2/\epsilon)^{|\gamma|/2}$, so the theoretical constants deteriorate exponentially as $|\gamma|$ grows and the guarantees are most informative for moderate extrapolation, which matches the small-to-moderate $|\gamma|$ values that worked best in our experiments.

In terms of memory, Anon maintains three parameter-sized optimizer states (the same count as AMSGrad) versus two for Adam, i.e., 12 vs.\ 8 bytes per parameter in FP32 optimizer states. Since activations dominate end-to-end training memory in our workloads, the increase in total training memory is far below 50\%.

Regarding the choice of $\gamma$, our results across all tasks provide a directional search prior: for conventional CNNs, start at $\gamma=0$ and move to smaller values; for diffusion and language modeling, start at $\gamma=1$ and move to larger values; if validation performance does not improve in that direction, fall back to the conventional $[0,1]$ interval. Since $\beta_1$, $\beta_2$ and $\epsilon$ are fixed throughout all experiments, $\gamma$ and the learning rate are the only searched quantities, so the added budget is one extra axis explored in a known direction rather than a sweep over $\mathbb{R}$.

Finally, we clarify the relation to lazy or stale second-moment updates: IDU differs in that the pre-conditioner is held fixed over geometrically growing blocks and the block values are combined through the $\beta_3$-EMA recurrence, which is precisely what yields the $O(\sqrt{T})$ bound on cumulative pre-conditioner variation (Theorem 3.2). A naive stale update without block structure or EMA smoothing does not admit such a bound, and our ablation (Table~\ref{suptab: ablation study}) shows the block schedule and smoothing factor both matter empirically.

We also hope this work contributes to the broader exploration of deep learning model design. Our experiments reveal that different model architectures exhibit distinct preferences for specific adaptivity levels. This suggests that some architectural modifications previously deemed ``ineffective'' might simply have been mismatched with the optimizer's fixed adaptivity (typically confined to the [0, 1] range). In such scenarios, Anon's capability for extensive adaptivity tuning could potentially unlock the latent capabilities of these architectures.

\section{Adaptivity of Optimizers}
\subsection{Summary of Adaptivity}
\label{supsec:adaptivity}
We present the detailed adaptivity expressions for the optimizers discussed in the main paper in Table~\ref{suptab:adaptivity}.

\begin{table}[h]
    \centering
    \caption{Summary of adaptivity expressions for representative optimizers.}
    \vspace{5pt}
    \setlength{\tabcolsep}{8pt}
    \renewcommand{\arraystretch}{1.5} 
    \begin{tabular}{cc}
    \toprule
      \textbf{Optimizer} & $\boldsymbol{A_t(\psi,{x})}$\\ \hline
        SGD & $0$\\ \hline
        SGDM & $0$\\ \hline
        RMSProp &
        $\displaystyle \frac{1}{1+{\epsilon}/{\sqrt{\textit{EMA}(\boldsymbol{x}^2;\beta_2)}}}$
        \\ \hline
        Adam & Same as RMSProp\\ \hline
        AMSGrad & $\displaystyle \frac{1}{1+{\epsilon}\big/\max_{i\in[t]}{\sqrt{\textit{EMA}(\boldsymbol{x}_{1:i}^2;\beta_2)}}}$ \\ \hline
        Padam & $\displaystyle \frac{2p}{1+{\epsilon}\big/\max_{i\in[t]}{\sqrt{\textit{EMA}(\boldsymbol{x}_{1:i}^2;\beta_2)}}}$ \\ \hline
        AdaBound & $\left\{
        \begin{array}{ll}
            A_t^\text{RMSProp}, & \text{if } \eta_l(t) < \psi_t^\text{RMSProp} < \eta_u(t), \\
            0, & \text{otherwise}.
        \end{array}
    \right.$\\ \hline
        AdaBelief & $\displaystyle \frac{1}{1+\epsilon\cdot \frac{\frac{1}{1-\beta_2}+\sqrt{\textit{EMA}((\boldsymbol{x}-\boldsymbol{\phi}^{\text{Adam}})^2+\frac{\epsilon}{1-\beta_2};\beta_2)}}{\textit{EMA}((\boldsymbol{x}-\boldsymbol{\phi}^{\text{Adam}})^2;\beta_2)}}$\\ \hline
        Anon & Derived in \eqref{supeq:AofAnon} ($\approx \gamma$)\\ \bottomrule
    \end{tabular}
    \label{suptab:adaptivity}
\end{table}

For Algorithm~\ref{algo:anon}, we provide an equivalent formulation in Algorithm~\ref{algo:anon2}. While this formulation yields no speedup, it offers a clearer representation of the underlying mechanism. Note that the accumulator $w_k$ in Algorithm~\ref{algo:anon2} corresponds to $1/v_k^2$ in Algorithm~\ref{algo:anon}, i.e., an arithmetic average of $\sigma^\gamma$ over delay intervals (the debiasing term $1-\beta_2^{\max(t/2,1)}$ in Algorithm~\ref{algo:anon} equals $1-\beta_2^{2^k-a}$ at update steps $t=2^k$); accordingly, $V_k=\operatorname{diag}(\sqrt{w_k})$ here is exactly the inverse of the update matrix $V_k$ in Algorithm~\ref{algo:anon}, so the two formulations produce identical iterates.

\begin{algorithm}[H]
\caption{The Anon Optimizer (Equivalent Formulation)}
\label{algo:anon2}
\begin{algorithmic}[1]
\Require $\eta$, $\beta_1$, $\beta_2$, $\epsilon$, $\gamma$
\State \textbf{Initialize} $\boldsymbol{\theta}_0$, 
$\boldsymbol{m}_0 \gets \boldsymbol{0}$, 
$\boldsymbol{s}_0 \gets \boldsymbol{0}$, 
$t \gets 0$, 
$k \gets 0$, 
$a \gets 0$,

\While{$\boldsymbol{\theta}_t$ not converged}
    \State $t \gets t+1$
    \State $\boldsymbol{g}_t \gets \nabla f_t(\boldsymbol{\theta}_{t})$
    \State $\boldsymbol{m}_t \gets 
    \beta_1\boldsymbol{m}_{t-1}
    +(1-\beta_1)\boldsymbol{g}_t$
    \State $\widehat{\boldsymbol{m}}_t \gets 
    \frac{\boldsymbol{m}_t}{1-\beta_1^t}$
    \State $\boldsymbol{s}_t \gets 
    \beta_2\boldsymbol{s}_{t-1}
    +(1-\beta_2)\boldsymbol{g}_t^2$

    \If{$t=2^k$}
        \State $\boldsymbol{\sigma}_k \gets 
        \boldsymbol{s}_t/(1-\beta_2^{2^{k}-a})+\epsilon$
        \State $a \gets 
        2^k$
        \State $\boldsymbol{w}_k \gets 
        \begin{cases}
        \frac{\boldsymbol{w}_{k-1}
        +\boldsymbol{\sigma}_k^\gamma}{2}
        ,
        & \text{if } k>0,\\
        \boldsymbol{\sigma}_k^{\gamma},
        & \text{otherwise}
        \end{cases}$
        \State $\boldsymbol{s}_t \gets \boldsymbol{0}$
        \State $\boldsymbol{V}_k \gets 
        \operatorname{diag}(\sqrt{w_{k,1}},\ldots,\sqrt{w_{k,d}})$
        \State $k \gets k+1$
    \EndIf

    \State $\boldsymbol{\theta}_t \gets 
    \Pi_{\mathcal{F},\boldsymbol{V}_{k-1}}
    \left(
    \boldsymbol{\theta}_{t-1}
    -
    \eta(t)\boldsymbol{V}_{k-1}^{-1}\widehat{\boldsymbol{m}}_t
    \right)$
\EndWhile

\end{algorithmic}
\end{algorithm}

\subsection{The Effect of Adaptivity}
\label{sup:effct of A}
To intuitively illustrate the impact of adaptivity, we present a visualization in Figure~\ref{fig:arrow}.

\begin{figure}[htpb]
    \centering
    \includegraphics[width=0.8\textwidth]{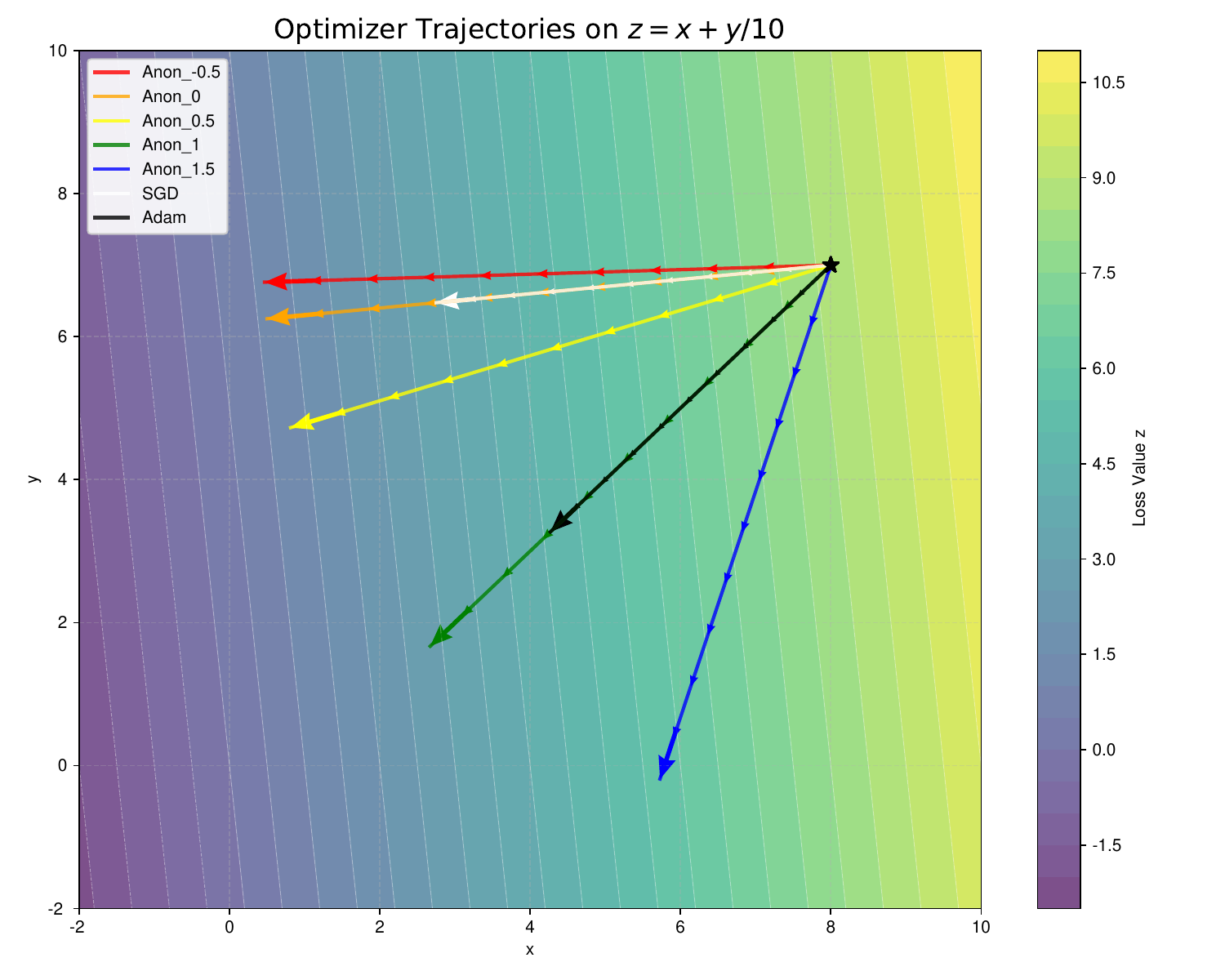}
    \caption{Optimization trajectories of {SGDM}, {Adam}, and {Anon} with varying $\gamma$. The gradient from yellow to purple indicates decreasing loss values. Different learning rates are applied to clearly visualize the distinct update directions.}
    \label{fig:arrow}
\end{figure}

\label{sup:more results}
We further demonstrate that varying adaptivity drives optimizers toward distinct regions of the parameter space by training GPT2-small on OpenWebText. The validation loss results are presented in Table~\ref{table:llmsup}. Additionally, we analyze the cosine similarity of the trained model parameters in Table~\ref{table:sim}. Notably, Anon with $\gamma=1.15$ exhibits the lowest similarity when compared to Adam, Lion~\citep{chen2023symbolic}, and Muon~\citep{jordan2024muon}, suggesting it discovers a unique solution.

\begin{table}[htpb]
    \centering
    \caption{Validation loss on OpenWebText. The results in this table were measured on a different hardware platform from those of Table~\ref{table:llms} in the main text; all values within each table were measured under an identical hardware environment.}
    \label{table:llmsup}
    \begin{tabular}{@{}ccccccc@{}}
        \toprule
         & Anon$_{\gamma=1}$ & Anon$_{\gamma=1.1}$ & Anon$_{\gamma=1.15}$ & Adam & Lion & Muon \\ 
        \midrule
        Loss & 2.937 & \textbf{2.927} & 2.932 & 2.934 & 2.992 & 3.092 \\ 
        \bottomrule
    \end{tabular}
\end{table}

\begin{table}[htpb]
    \centering
    \caption{Cosine similarity on OpenWebText. The upper number in each cell represents the cosine similarity of \textbf{all parameters}, while the lower number represents the cosine similarity of the \textbf{weights only}.}
    \vspace{5pt}
    \label{table:sim}
    \begin{tabular}{@{}ccccccc@{}}
        \toprule
         & Anon$_{\gamma=1}$ & Anon$_{\gamma=1.1}$ & Anon$_{\gamma=1.15}$ & Adam & Lion & Muon \\ 
        \midrule
        \multirow{2}{*}{Anon$_{\gamma=1}$} 
         & 1 & 0.597 & 0.355 & 0.914 & 0.857 & 0.901 \\
         & 1 & 0.317 & 0.201 & 0.511 & 0.161 & 0.194 \\
        \midrule
        \multirow{2}{*}{Anon$_{\gamma=1.1}$} 
         & 0.597 & 1 & 0.425 & 0.573 & 0.522 & 0.539 \\
         & 0.317 & 1 & 0.328 & 0.248 & 0.088 & 0.098 \\
        \midrule
        \multirow{2}{*}{Anon$_{\gamma=1.15}$} 
         & 0.355 & 0.425 & 1 & 0.355 & 0.299 & 0.306 \\
         & 0.201 & 0.328 & 1 & 0.156 & 0.056 & 0.063 \\
        \midrule
        \multirow{2}{*}{Adam} 
         & 0.914 & 0.573 & 0.355 & 1 & 0.880 & 0.923 \\
         & 0.511 & 0.248 & 0.156 & 1 & 0.191 & 0.222 \\
        \midrule
        \multirow{2}{*}{Lion} 
         & 0.857 & 0.522 & 0.299 & 0.880 & 1 & 0.925 \\
         & 0.161 & 0.088 & 0.056 & 0.191 & 1 & 0.132 \\
        \midrule
        \multirow{2}{*}{Muon} 
         & 0.901 & 0.539 & 0.306 & 0.923 & 0.925 & 1 \\
         & 0.194 & 0.098 & 0.063 & 0.222 & 0.132 & 1 \\
        \bottomrule
    \end{tabular}
\end{table}

\section{Experimental Details and Additional Results}
\label{supsec:details}
\subsection{Image Classification}



\paragraph{ResNet18} We report results from the sources cited in the main paper. We adopted the experimental setup from the official AdaBelief implementation\footnote{https://github.com/juntang-zhuang/Adabelief-Optimizer} and reproduced the SGDM and AdaBelief results using their recommended hyperparameters. For AMSGrad (with decoupled weight decay), we searched the learning rate in \{0.1, 0.01, 0.001\}, finding 0.01 to be optimal. For Anon, we set the learning rate to 1 and searched $\gamma$ in \{-0.1, -0.05, 0, 0.05\}, with -0.1 yielding the best performance.

\paragraph{ResNet50} We report results from the sources cited in the main paper. We used the setup from the official LookAround implementation\footnote{https://github.com/Ardcy/Lookaround} and reproduced SGDM and LookAround results. Due to computational constraints, we performed a limited search for Anon, setting $\eta=1$ and $\gamma=-0.1$.

\subsection{Image Generation}
\paragraph{Diffusion Model}
We adopted the experimental setup from the official implementation\footnote{https://github.com/openai/improved-diffusion} (Unconditional CIFAR-10 with L\_hybrid objective and cosine noise schedule). We searched the learning rate in \{0.1, 0.01, ..., 0.00001\} for all optimizers. For Anon, we additionally searched $\gamma \in \{1, 1.1, 1.01\}$. The optimal configuration found was $\eta=0.0001$ and $\gamma=1.01$. 

\subsection{Language Modeling}
\paragraph{GPT2} We followed the setup in the official Sophia implementation\footnote{https://github.com/Liuhong99/Sophia}\footnote{https://github.com/karpathy/nanoGPT}, setting `nproc\_per\_node=4` due to resource limits. We observed that applying the Sophia learning rate scheduler to AdamW on GPT2-medium resulted in lower loss, so we applied this setting to both AdamW and Anon. We set $\gamma=1.1$ for GPT2-small and $\gamma=1$ for GPT2-medium. All optimizers used decoupled weight decay.

\subsection{Feasibility Study on 1B-Parameter LLMs: Hybrid Optimization with Muon}
\label{sec:appendix_qwen2_muon}

To evaluate the compatibility of Anon with state-of-the-art matrix-level optimizers on Transformer architectures, we conduct a preliminary feasibility study training a $\sim$1B parameter language model (Qwen2 \citep{yang2024qwen2}) on the OpenWebText dataset based on \citet{liu2025muonscalablellmtraining}.

Recent optimizers like Muon \citep{jordan2024muon} strictly update 2D parameters via orthogonalization but structurally rely on standard element-wise optimizers (typically AdamW) as a fallback mechanism for 1D parameters (e.g., biases, normalization layers, and embeddings). Because Anon is intrinsically designed as an element-wise optimizer, it is mathematically orthogonal to Muon's 2D updates. In this experiment, we deploy Anon as a direct, drop-in replacement for the 1D AdamW fallback.

To ensure a strictly controlled comparison, we utilize the default recommended hyperparameters for the Muon+AdamW baseline. For our method, we apply a default adaptivity of $\gamma = 1.1$ for Anon without further tuning. Due to the computational constraints associated with 1B-scale training, this experiment is executed as a single run to serve as a proof-of-concept rather than a comprehensive statistical benchmark.

\begin{table}[h]
    \centering
    \caption{Validation loss of the 1B-parameter LLM (Qwen2) trained on OpenWebText. Anon serves as the 1D parameter optimizer in the hybrid setup.}
    \label{tab:qwen2_muon}
    \begin{tabular}{lc}
    \toprule
    \textbf{Hybrid Optimizer Configuration} & \textbf{Validation Loss} $\downarrow$ \\
    \midrule
    Muon (2D) + AdamW (1D) [Baseline] & 3.81 \\
    \textbf{Muon (2D) + Anon (1D) [Ours]} & \textbf{3.79} \\
    \bottomrule
    \end{tabular}
\end{table}

As presented in Table \ref{tab:qwen2_muon}, the hybrid Muon+Anon configuration achieves a validation loss of 3.79, compared to 3.81 for the standard Muon+AdamW baseline. These preliminary results suggest two empirical observations: (1) Anon maintains optimization stability in 1B-parameter regimes; and (2) Anon can function as a structurally compatible 1D fallback to enhance existing matrix-level optimizers.

\subsection{Ablation Study on IDU Hyperparameters}
We conducted an ablation study to investigate the impact of the hyperparameters $\{a_n\}$ and $\beta_3$ in IDU. Experiments were performed using ResNet20 on CIFAR-10, with results summarized in Table~\ref{suptab: ablation study}. These results demonstrate that IDU is robust to hyperparameter variations; indeed, certain configurations (e.g., $\beta_3=0.3, a_n=4^{n-1}$) even outperform our default setting ($\beta_3=0.5, a_n=2^{n-1}$).

\begin{table}[H]
    \caption{Ablation study on the hyperparameters $\{a_n\}$ and $\beta_3$ of IDU.}
    \vspace{5pt}
    \label{suptab: ablation study}
    \centering
    \setlength{\tabcolsep}{6pt}
    \begin{tabular}{lccccc}
    \toprule
     & $\beta_3=0.1$  & $\beta_3=0.3$ & $\beta_3=0.5$  & $\beta_3=0.7$      & $\beta_3=0.9$\\
    \midrule
    $a_n=2^{n-1}$  & 91.76 & 91.98 & \underline{92.42} & 92.43 & 92.16 \\
    $a_n=3^{n-1}$ & 92.26 & 92.23 & 92.28 & 92.34 & 92.38 \\
    $a_n=4^{n-1}$  & 91.97 &\bf 92.44 & 92.25 & 92.12 & 92.31 \\
    \bottomrule
    \end{tabular}
\end{table}

\subsection{Controlled Multi-Seed Benchmark on CIFAR-10}
\label{supsec:multiseed}
To validate the central comparisons with statistical significance, we run a fully controlled benchmark on CIFAR-10 in a single codebase (the official AdaHessian repository \citep{yao2021adahessian}, which the AGD paper \citep{yue2023agd} identifies for its CIFAR experiments), with three independent runs per configuration. Each baseline uses the best configuration reported by its own source for this benchmark. Anon is selected over the grid already shown in Figure~\ref{fig:ablation} ($\gamma\in\{-0.24,-0.16,\ldots,0.24\}\times\eta\in\{0.1,0.5,1,5,10\}$), choosing $(\gamma=-0.08,\eta=1)$ for ResNet20 and $(\gamma=-0.16,\eta=0.5)$ for ResNet32. Results are reported in Table~\ref{suptab:multiseed}.

\begin{table}[H]
    \caption{Top-1 accuracy (\%, mean $\pm$ std over 3 runs) on CIFAR-10. Every entry except the last row is our own run in a single codebase.}
    \vspace{5pt}
    \label{suptab:multiseed}
    \centering
    \setlength{\tabcolsep}{8pt}
    \begin{tabular}{lcc}
    \toprule
    \textbf{Method} & \textbf{ResNet20} & \textbf{ResNet32} \\
    \midrule
    SGDM & $92.33\pm0.09$ & $92.95\pm0.04$ \\
    AdamW & $91.96\pm0.15$ & $92.62\pm0.16$ \\
    AMSGradW & $92.13\pm0.10$ & $92.84\pm0.19$ \\
    AGD (our reproduction) & $92.25\pm0.18$ & $93.02\pm0.13$ \\
    Anon w/o IDU & $92.23\pm0.29$ & $93.14\pm0.06$ \\
    Anon w/ IDU & $\mathbf{92.47\pm0.05}$ & $\mathbf{93.20\pm0.08}$ \\
    \midrule
    AGD, as reported by \citet{yue2023agd} & $92.35\pm0.24$ & $93.12\pm0.18$ \\
    \bottomrule
    \end{tabular}
\end{table}

Anon with IDU attains the highest mean on both architectures and the tightest spread on ResNet20 (std $0.05$, against $0.09$--$0.29$ for the others), so the ordering is consistent across seeds rather than resting on a single run. Our AGD reproduction matches the published value to within one standard deviation on both architectures ($92.25$ vs.\ $92.35$; $93.02$ vs.\ $93.12$), providing direct evidence that environment differences do not materially shift cross-paper numbers on this benchmark; this supports the baseline-sourcing protocol used in Tables~\ref{table:resnet18} and \ref{table:resnet50}.

\subsection{Ablation of IDU across Operating Regimes}
\label{supsec:idulablation}
IDU targets a specific failure mode: persistent non-monotone decrease of the pre-conditioner, which is what breaks Adam's convergence in the construction of \citet{reddi2019convergence}. Figure~\ref{fig:noise} identifies when this regime is entered, namely small $\beta_2$ with noisy gradients. We therefore ablate IDU in that regime first, and then across four adaptivity regimes. Throughout, ``without IDU'' means the pre-conditioner state is updated at every step instead of at the delayed block boundaries; everything else is unchanged, and $(\gamma,\text{lr})$ is held fixed within each row so that the IDU switch is the only difference. Three runs per cell, ResNet20 on CIFAR-10.

Table~\ref{suptab:idu_beta2} reports the regime IDU is designed for: holding $\gamma=1.1$ and $\text{lr}=0.01$ fixed while reducing $\beta_2$. IDU improves the mean at all three reduced values of $\beta_2$ ($+0.12$ to $+0.32$). Notably, with IDU the accuracy stays essentially flat as the second-moment window is shortened from $0.99$ to $0.75$ ($91.84/91.79/91.89$), whereas without IDU it sits $0.1$--$0.3$ points lower at every value.

\begin{table}[H]
    \caption{Top-1 accuracy (\%, mean $\pm$ std over 3 runs) with and without IDU at reduced $\beta_2$ ($\gamma=1.1$, lr $=0.01$, ResNet20).}
    \vspace{5pt}
    \label{suptab:idu_beta2}
    \centering
    \setlength{\tabcolsep}{8pt}
    \begin{tabular}{lccc}
    \toprule
    $(\gamma,\beta_2,\text{lr})$ & w/o IDU & w/ IDU & diff \\
    \midrule
    $(1.1,\ 0.99,\ 0.01)$ & $91.58\pm0.01$ & $91.84\pm0.19$ & $+0.26$ \\
    $(1.1,\ 0.90,\ 0.01)$ & $91.67\pm0.04$ & $91.79\pm0.31$ & $+0.12$ \\
    $(1.1,\ 0.75,\ 0.01)$ & $91.57\pm0.16$ & $91.89\pm0.12$ & $+0.32$ \\
    \bottomrule
    \end{tabular}
\end{table}

Table~\ref{suptab:idu_gamma} reports the same ablation at the standard $\beta_2=0.999$ over four fixed diagnostic points spanning the negative, interpolation, Adam-endpoint and super-adaptive regimes. Here the differences are small and mixed in sign: none of the four settings encountered the worst case that IDU exists to prevent, namely a divergence or a persistent non-monotone collapse of the pre-conditioner, so IDU has nothing to correct at these points. As an implementation check, Anon at $\gamma=1$ without IDU is algorithmically AdamW and obtains $92.04\pm0.14$, consistent with the $91.96\pm0.15$ we measure for AdamW under this protocol.

\begin{table}[H]
    \caption{Top-1 accuracy (\%, mean $\pm$ std over 3 runs) with and without IDU at $\beta_2=0.999$ across adaptivity regimes (ResNet20).}
    \vspace{5pt}
    \label{suptab:idu_gamma}
    \centering
    \setlength{\tabcolsep}{8pt}
    \begin{tabular}{lccc}
    \toprule
    $(\gamma,\text{lr})$ & w/o IDU & w/ IDU & diff \\
    \midrule
    $(-0.1,\ 1)$ & $92.10\pm0.09$ & $92.15\pm0.18$ & $+0.05$ \\
    $(0.5,\ 0.1)$ & $92.17\pm0.13$ & $92.07\pm0.11$ & $-0.10$ \\
    $(1.0,\ 0.01)$ & $92.04\pm0.14$ & $92.20\pm0.23$ & $+0.16$ \\
    $(1.1,\ 0.01)$ & $91.91\pm0.23$ & $91.87\pm0.05$ & $-0.04$ \\
    \bottomrule
    \end{tabular}
\end{table}

The same regime-dependence is characteristic of other convergence-protection mechanisms: under this protocol AMSGradW differs from AdamW by $+0.17/+0.22$ on ResNet20/32, while on the diffusion task (Table~\ref{table:diffusion}) AMSGrad improves over Adam by $0.99$ FID. IDU is therefore best understood as a safeguard whose benefit becomes visible when its protected failure mode is approached, and as the mechanism that makes the whole range $\gamma\in\mathbb{R}$ usable with a convergence guarantee, rather than as an accuracy improvement in every regime. Since the protected pathology already exists at $\gamma=1$, IDU is also applicable as a stabilization mechanism for Adam-type optimizers beyond Anon.

\subsection{Training Dynamics on GPT-2}
\label{supsec:gpt2curves}
Table~\ref{suptab:gpt2curve} reports the GPT-2 Small validation loss at 20k-step intervals through the end of training. From the first reported checkpoint onward the ordering is unchanged, with no crossings, and the margin stays between $0.022$ and $0.035$ over AdamW and between $0.019$ and $0.047$ over Sophia-G. The final value is thus the endpoint of a separation established well before the end of training, rather than a lead built in the closing checkpoints. The same ordering also holds for the training loss and across $\gamma$: at every reported checkpoint, both the training and the validation loss are ordered $\gamma=1.1 < \gamma=1.0 < \gamma=0.9$, with smoothly decreasing trajectories, indicating that the validation gains of larger $\gamma$ are consistent with the optimization behavior during training rather than arising from a late validation-curve crossing.

\begin{table}[H]
    \caption{GPT-2 Small validation loss on OpenWebText at 20k-step intervals.}
    \vspace{5pt}
    \label{suptab:gpt2curve}
    \centering
    \setlength{\tabcolsep}{8pt}
    \begin{tabular}{lccc}
    \toprule
    \textbf{Steps} & \textbf{Anon ($\gamma=1.1$)} & \textbf{AdamW} & \textbf{Sophia-G} \\
    \midrule
    20k & 3.19062 & 3.22548 & 3.23539 \\
    40k & 3.08144 & 3.11093 & 3.12880 \\
    60k & 3.01662 & 3.04006 & 3.04837 \\
    80k & 2.96273 & 2.98518 & 2.99619 \\
    100k & \textbf{2.93283} & 2.95614 & 2.95143 \\
    \bottomrule
    \end{tabular}
\end{table}

\section{Benchmark Function Visualization}
\label{supsec:empirical}

\begin{figure}[H]
    \centering
    \begin{subfigure}[b]{0.24\textwidth}
        \includegraphics[width=\linewidth]{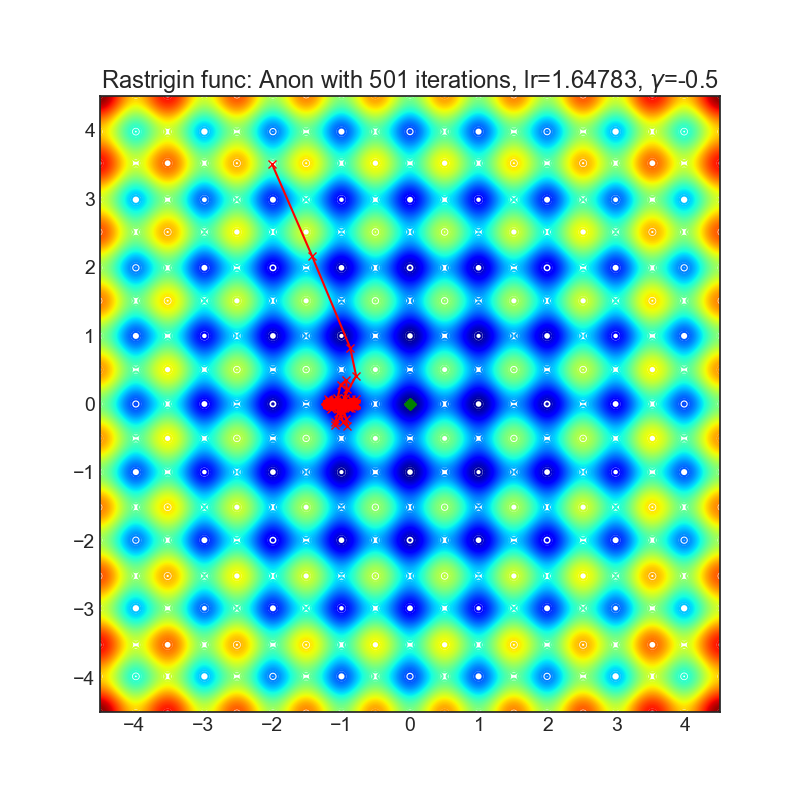}
    \end{subfigure}
    \hfill
    \begin{subfigure}[b]{0.24\textwidth}
        \includegraphics[width=\linewidth]{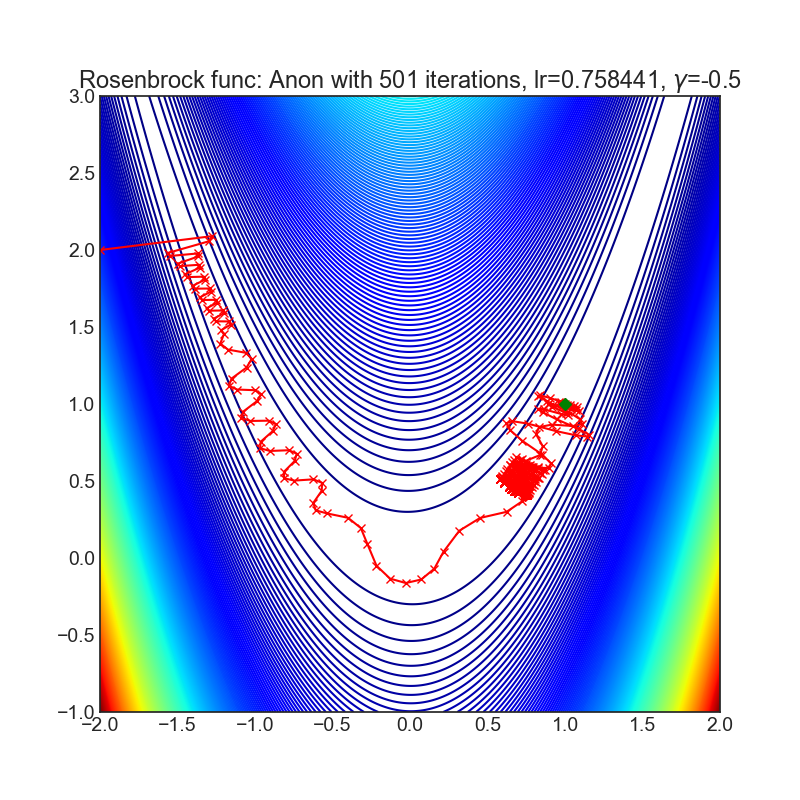}
    \end{subfigure}
    \hfill
    \begin{subfigure}[b]{0.24\textwidth}
        \includegraphics[width=\linewidth]{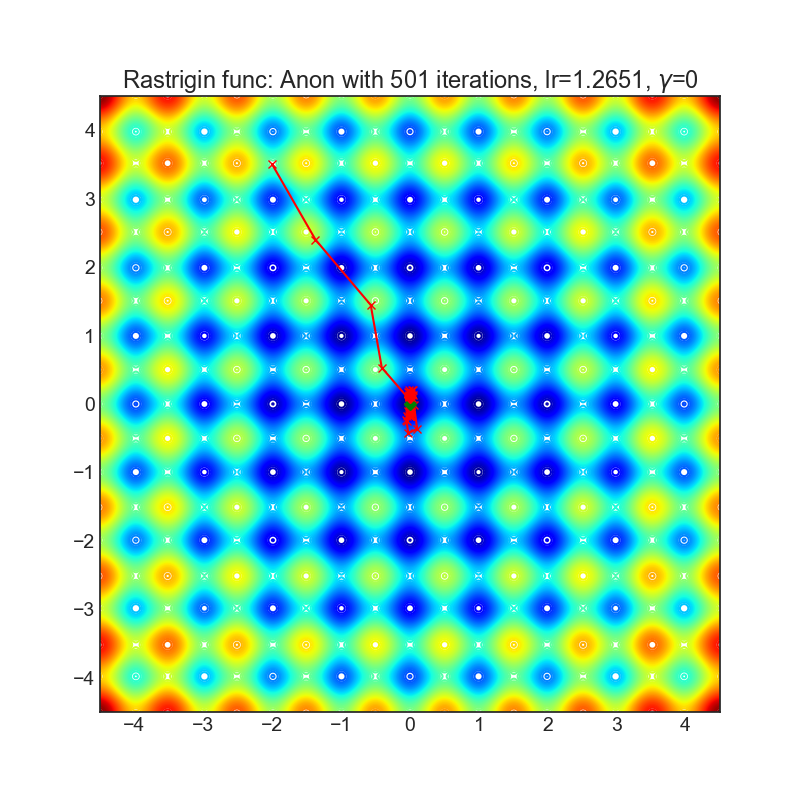}
    \end{subfigure}
    \hfill
    \begin{subfigure}[b]{0.24\textwidth}
        \includegraphics[width=\linewidth]{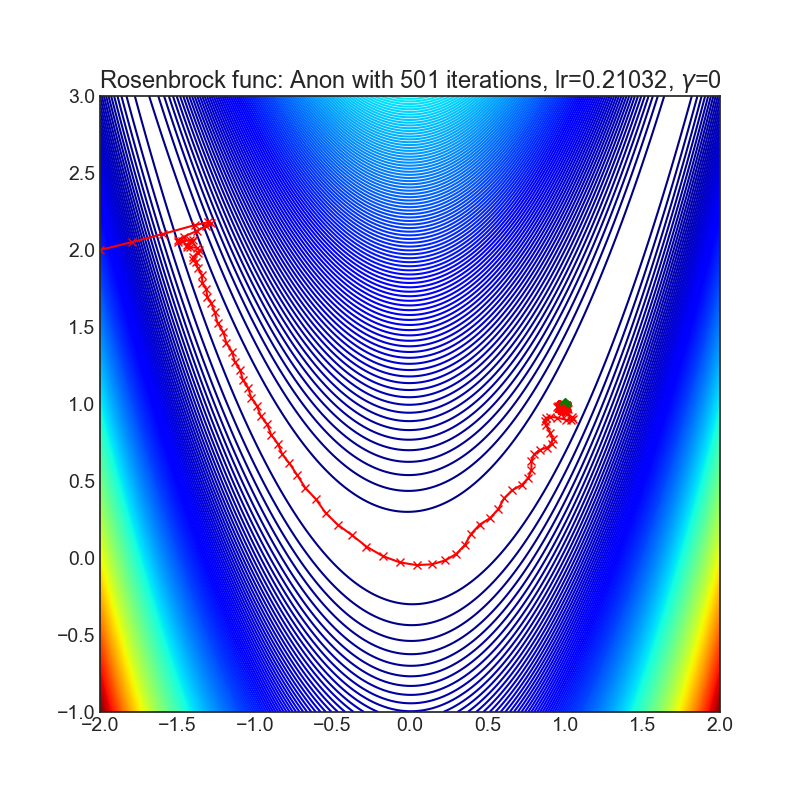}
    \end{subfigure}
    
    \vspace{-5mm}
    {\small \makebox[0.49\textwidth]{Anon ($\gamma=-0.5$)} \hfill \makebox[0.49\textwidth]{Anon ($\gamma=0$)}}
    \vspace{-0.5mm}

    \begin{subfigure}[b]{0.24\textwidth}
        \includegraphics[width=\linewidth]{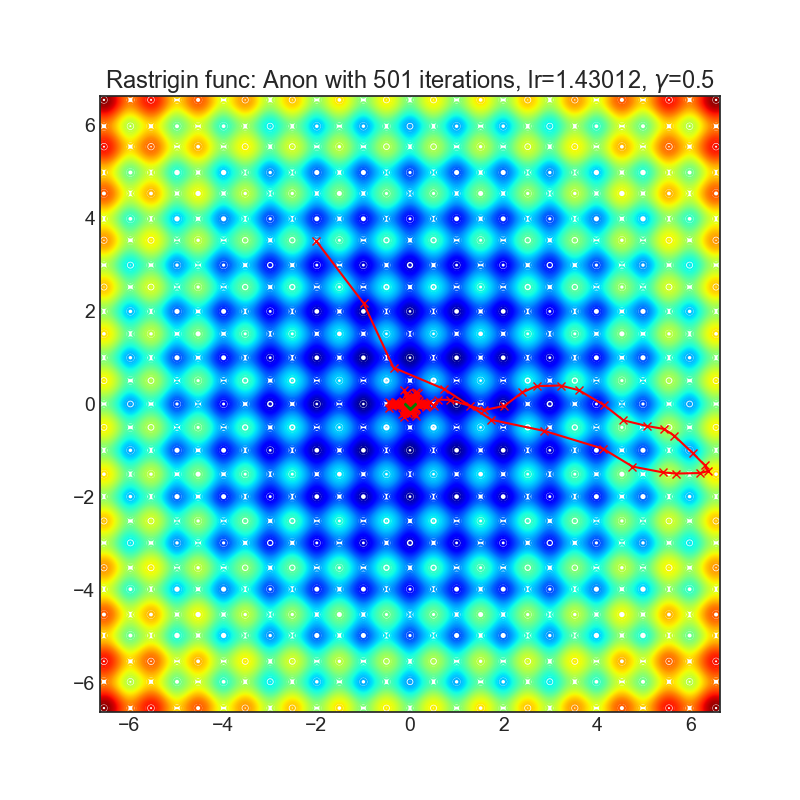}
    \end{subfigure}
    \hfill
    \begin{subfigure}[b]{0.24\textwidth}
        \includegraphics[width=\linewidth]{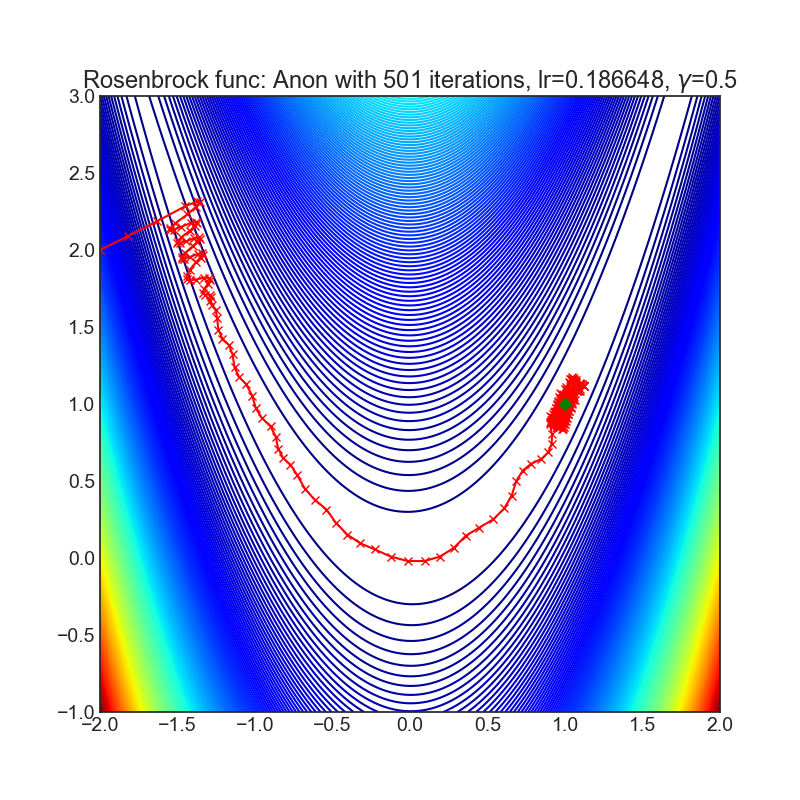}
    \end{subfigure}
    \hfill
    \begin{subfigure}[b]{0.24\textwidth}
        \includegraphics[width=\linewidth]{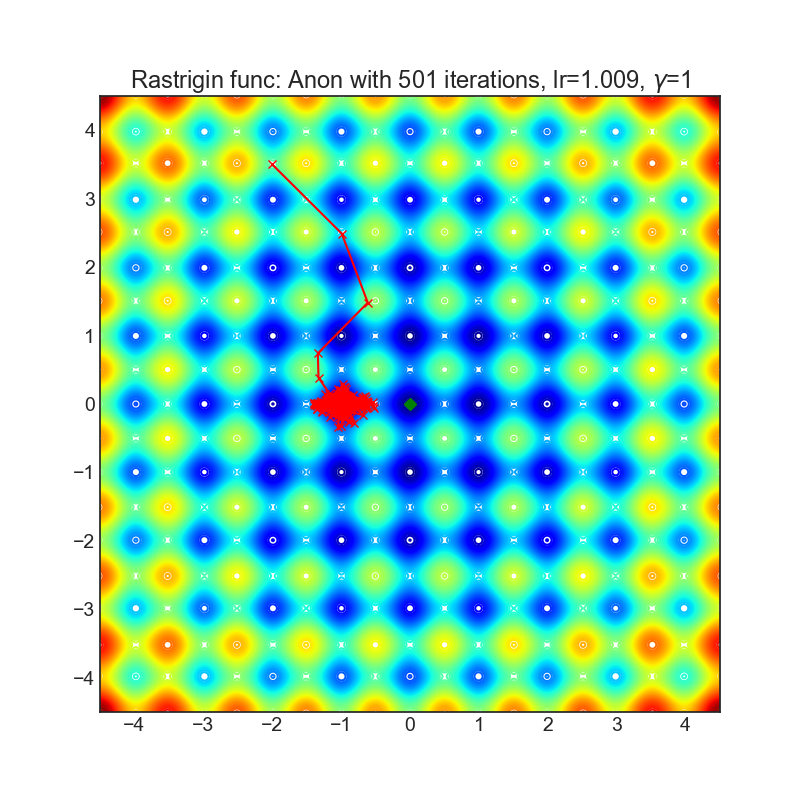}
    \end{subfigure}
    \hfill
    \begin{subfigure}[b]{0.24\textwidth}
        \includegraphics[width=\linewidth]{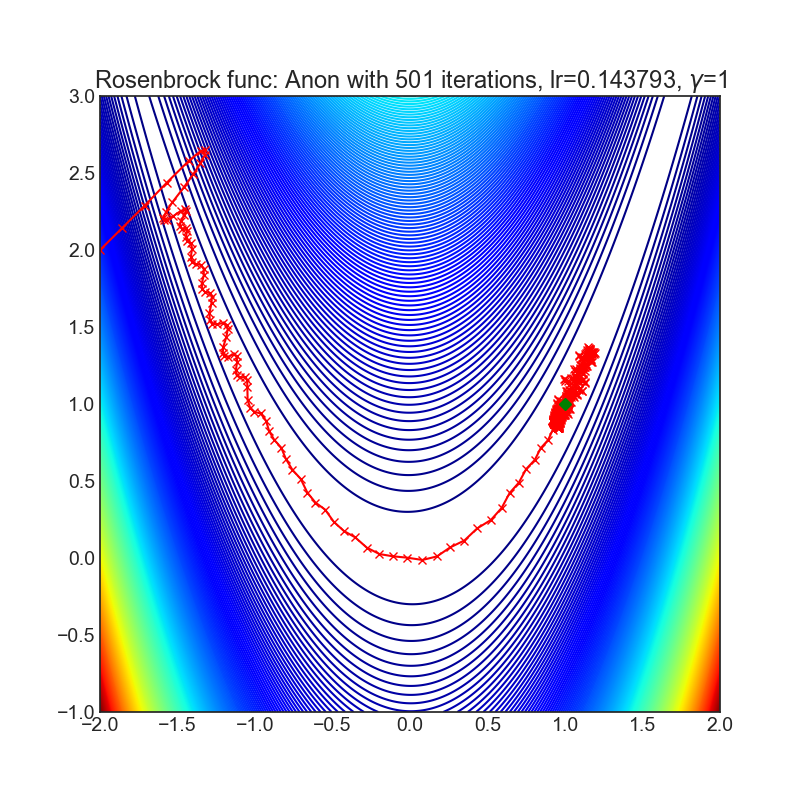}
    \end{subfigure}

    \vspace{-5mm}
    {\small \makebox[0.49\textwidth]{Anon ($\gamma=0.5$)} \hfill \makebox[0.49\textwidth]{Anon ($\gamma=1.0$)}}
    \vspace{-0.5mm}

    \begin{subfigure}[b]{0.24\textwidth}
        \includegraphics[width=\linewidth]{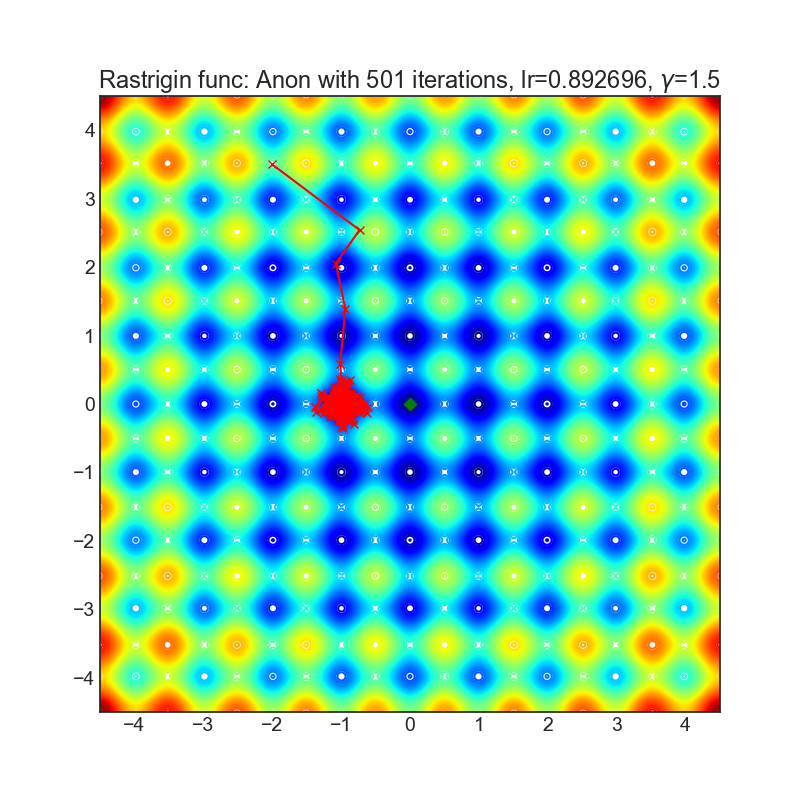}
    \end{subfigure}
    \hfill
    \begin{subfigure}[b]{0.24\textwidth}
        \includegraphics[width=\linewidth]{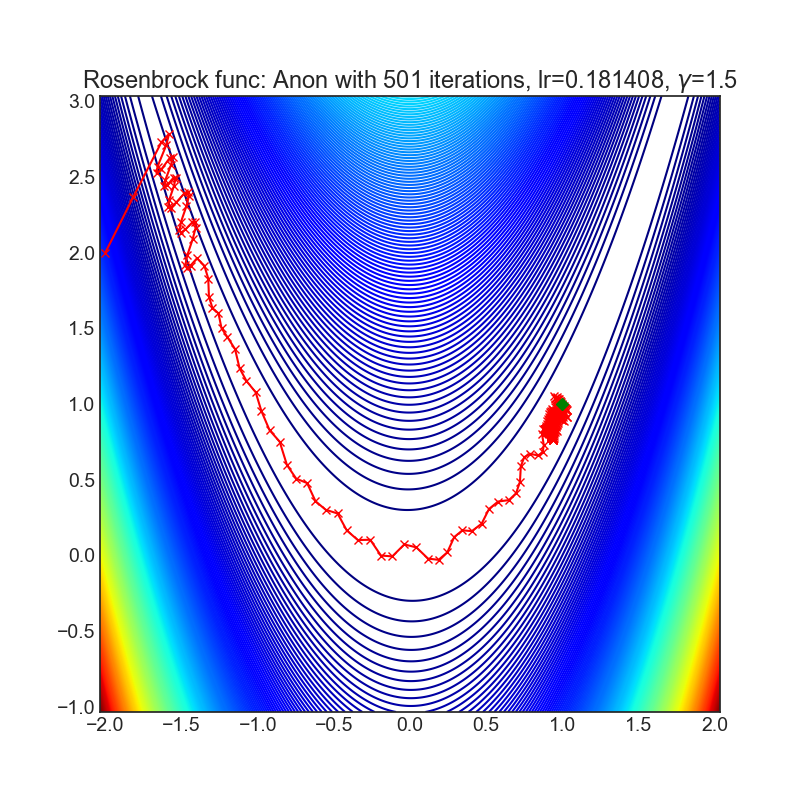}
    \end{subfigure}
    \hfill
    \begin{subfigure}[b]{0.24\textwidth}
        \includegraphics[width=\linewidth]{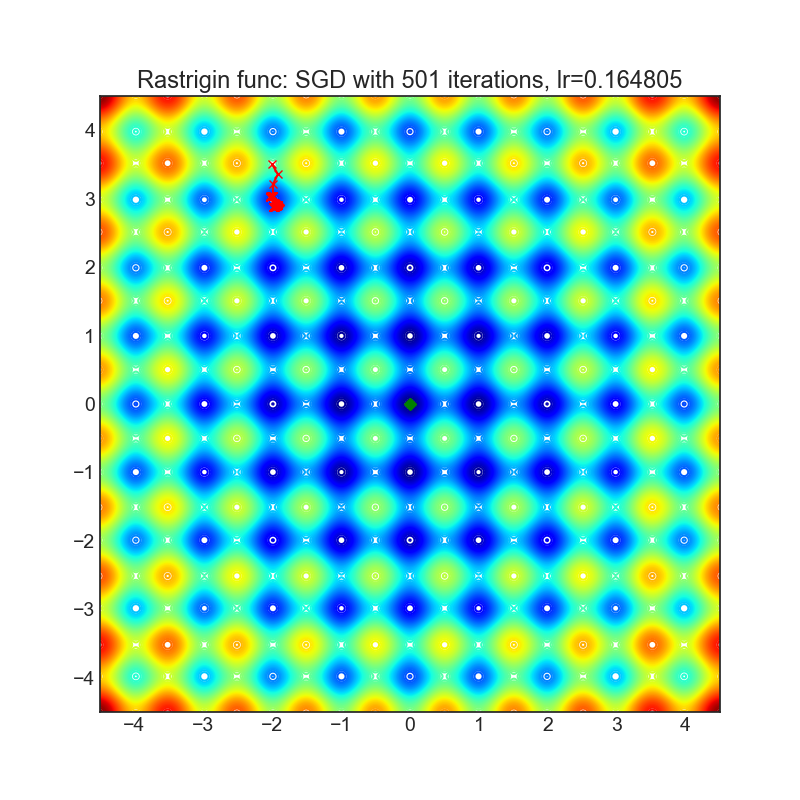}
    \end{subfigure}
    \hfill
    \begin{subfigure}[b]{0.24\textwidth}
        \includegraphics[width=\linewidth]{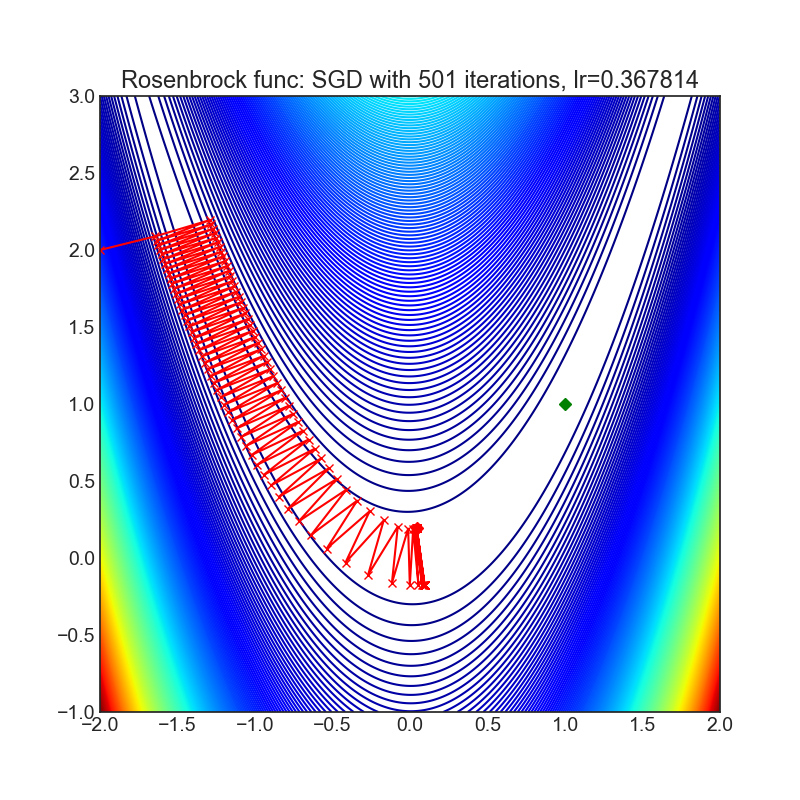}
    \end{subfigure}

    \vspace{-5mm}
    {\small \makebox[0.49\textwidth]{Anon ($\gamma=1.5$)} \hfill \makebox[0.49\textwidth]{SGD}}
    \vspace{-0.5mm}

    \begin{subfigure}[b]{0.24\textwidth}
        \includegraphics[width=\linewidth]{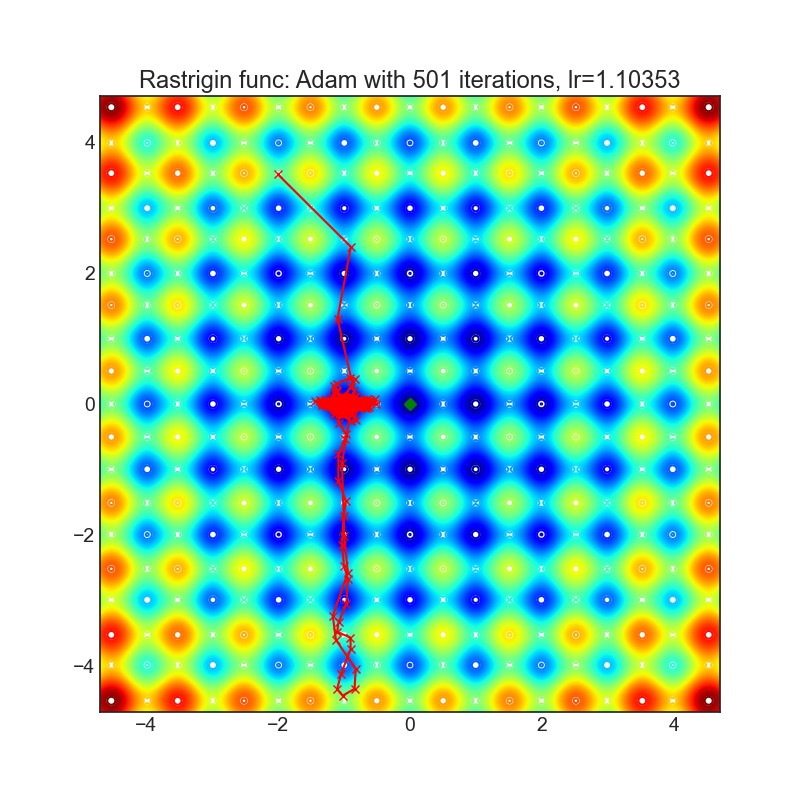}
    \end{subfigure}
    \hfill
    \begin{subfigure}[b]{0.24\textwidth}
        \includegraphics[width=\linewidth]{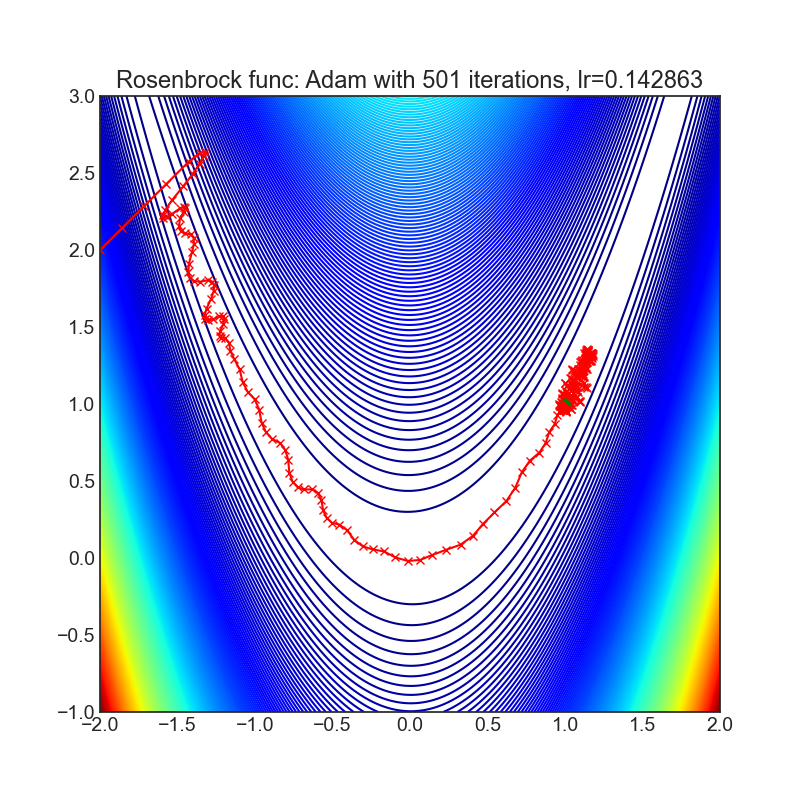}
    \end{subfigure}
    \hfill
    \begin{subfigure}[b]{0.24\textwidth}
        \includegraphics[width=\linewidth]{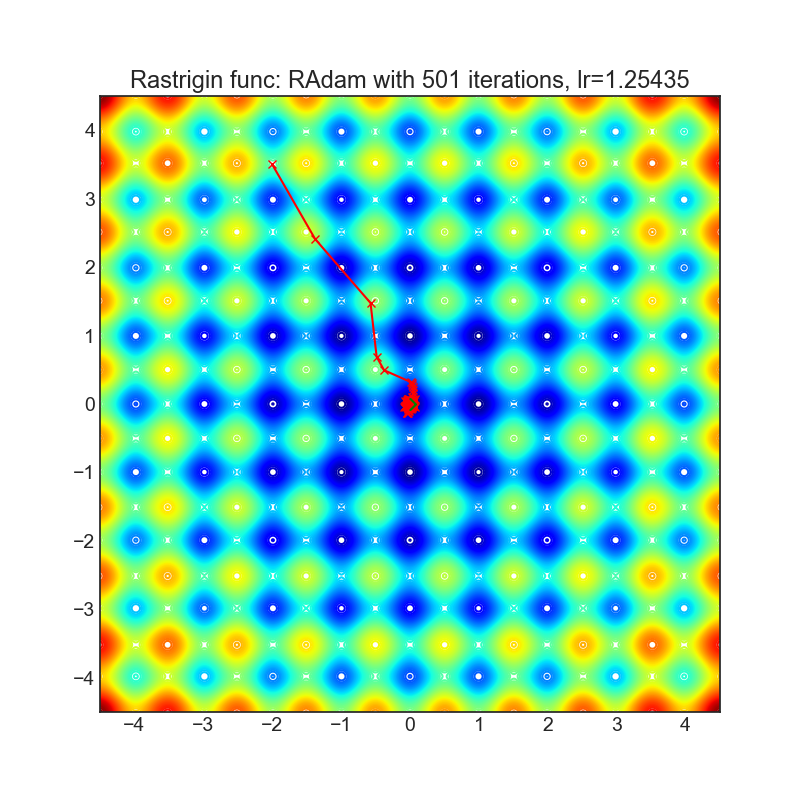}
    \end{subfigure}
    \hfill
    \begin{subfigure}[b]{0.24\textwidth}
        \includegraphics[width=\linewidth]{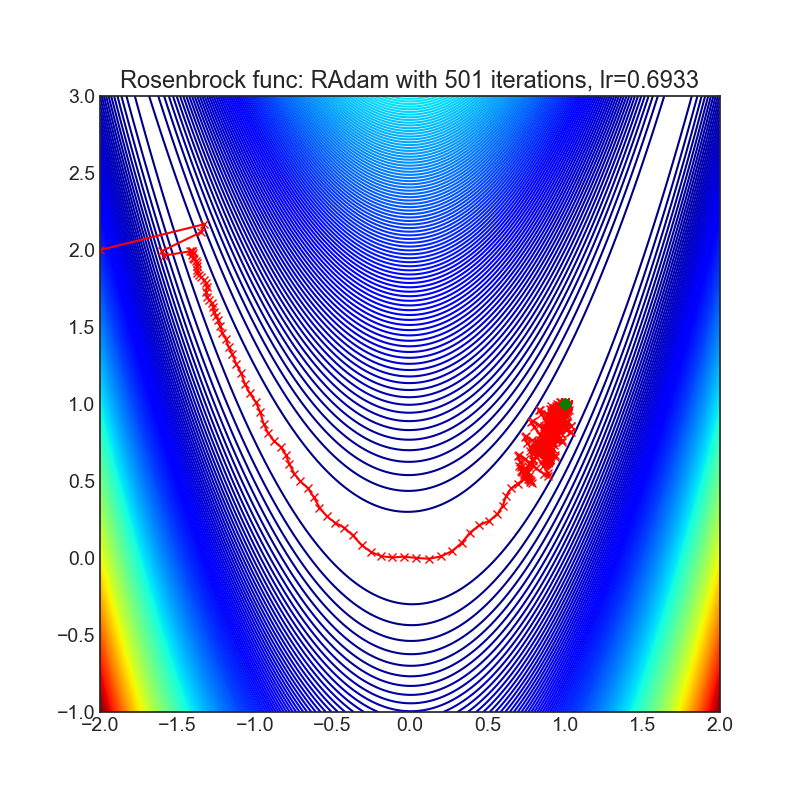}
    \end{subfigure}

    \vspace{-5mm}
    {\small \makebox[0.49\textwidth]{Adam} \hfill \makebox[0.49\textwidth]{RAdam}}
    \vspace{-0.5mm}

    \begin{subfigure}[b]{0.24\textwidth}
        \includegraphics[width=\linewidth]{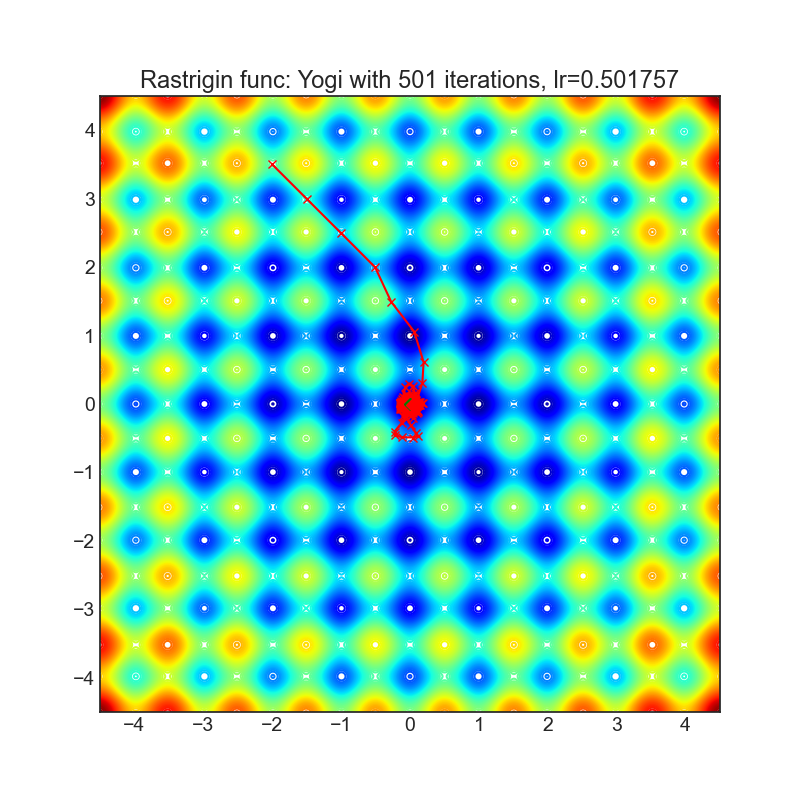}
    \end{subfigure}
    \hfill
    \begin{subfigure}[b]{0.24\textwidth}
        \includegraphics[width=\linewidth]{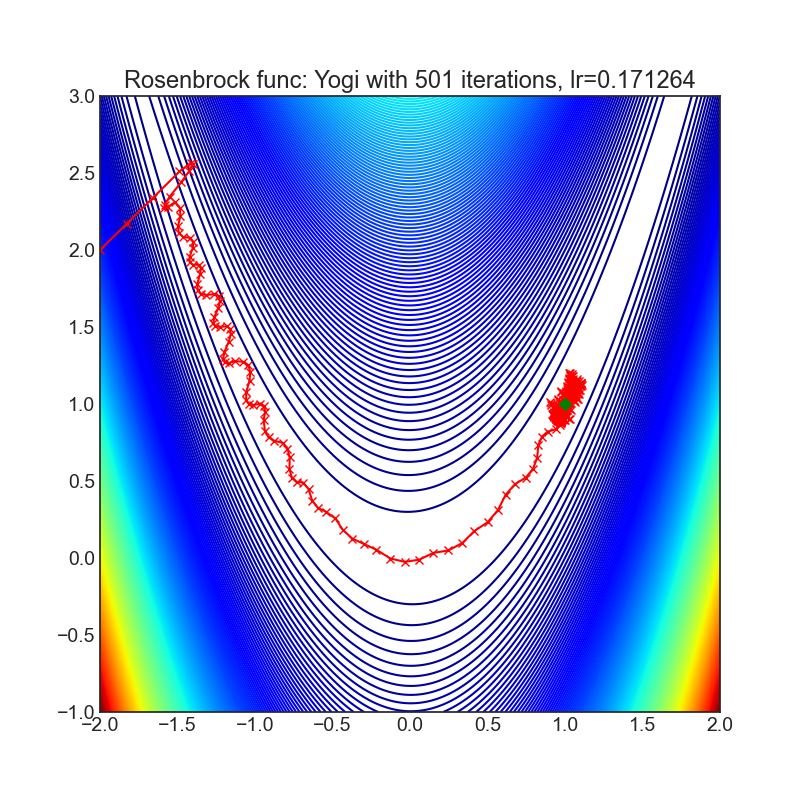}
    \end{subfigure}
    \hfill
    \begin{subfigure}[b]{0.24\textwidth}
        \includegraphics[width=\linewidth]{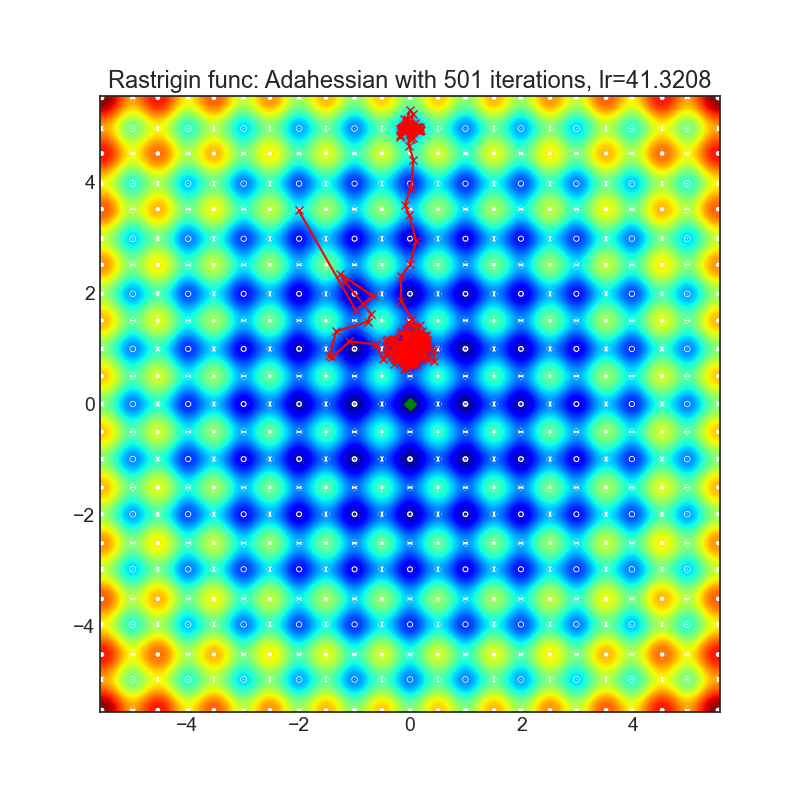}
    \end{subfigure}
    \hfill
    \begin{subfigure}[b]{0.24\textwidth}
        \includegraphics[width=\linewidth]{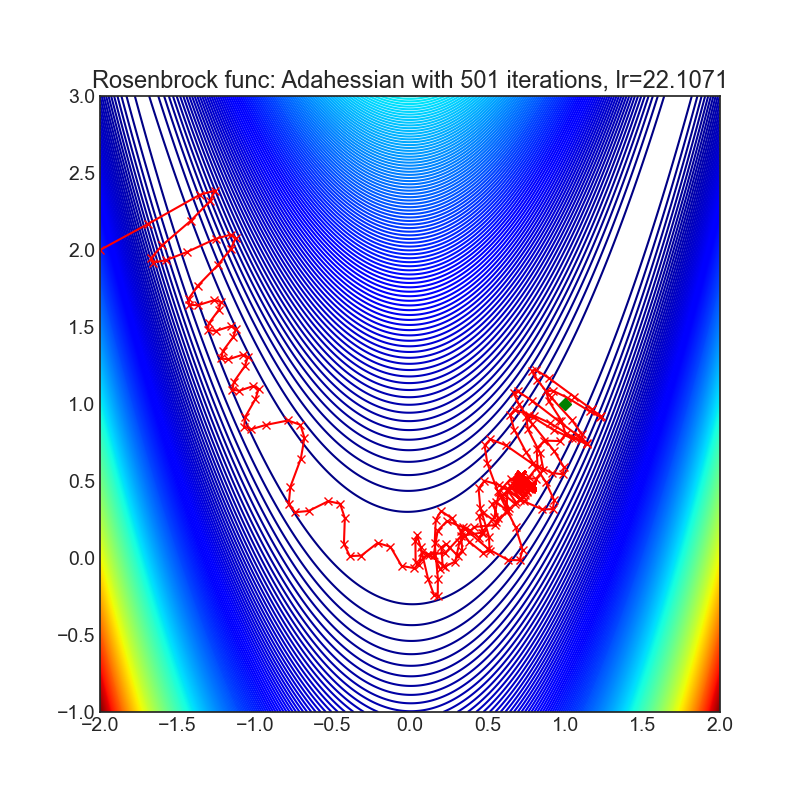}
    \end{subfigure}

    \vspace{-5mm}
    {\small \makebox[0.49\textwidth]{Yogi} \hfill \makebox[0.49\textwidth]{AdaHessian}}
    \vspace{-0.5mm}

    \begin{subfigure}[b]{0.24\textwidth}
        \includegraphics[width=\linewidth]{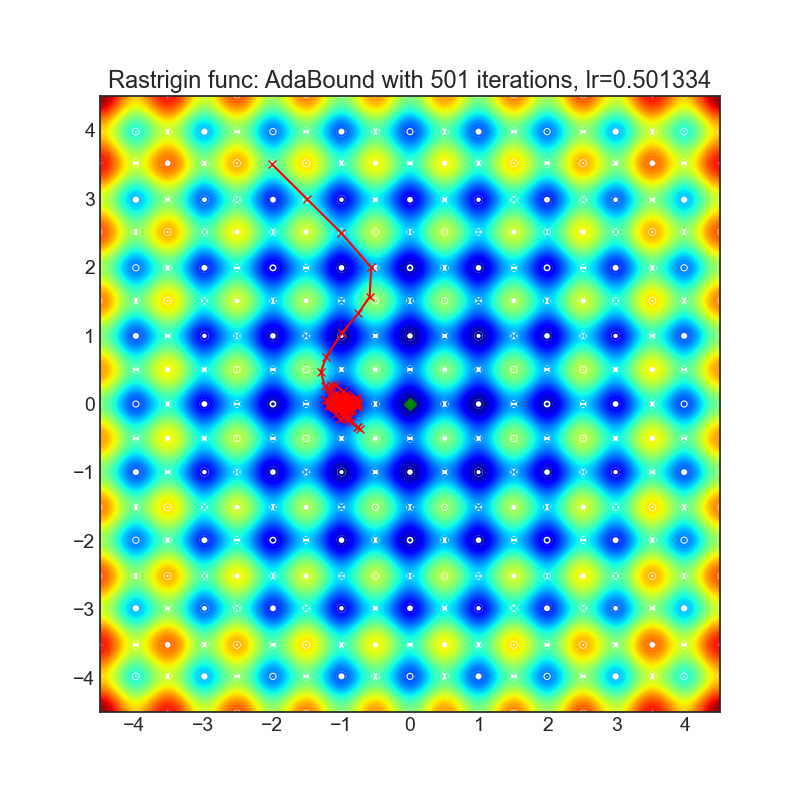}
    \end{subfigure}
    \hfill
    \begin{subfigure}[b]{0.24\textwidth}
        \includegraphics[width=\linewidth]{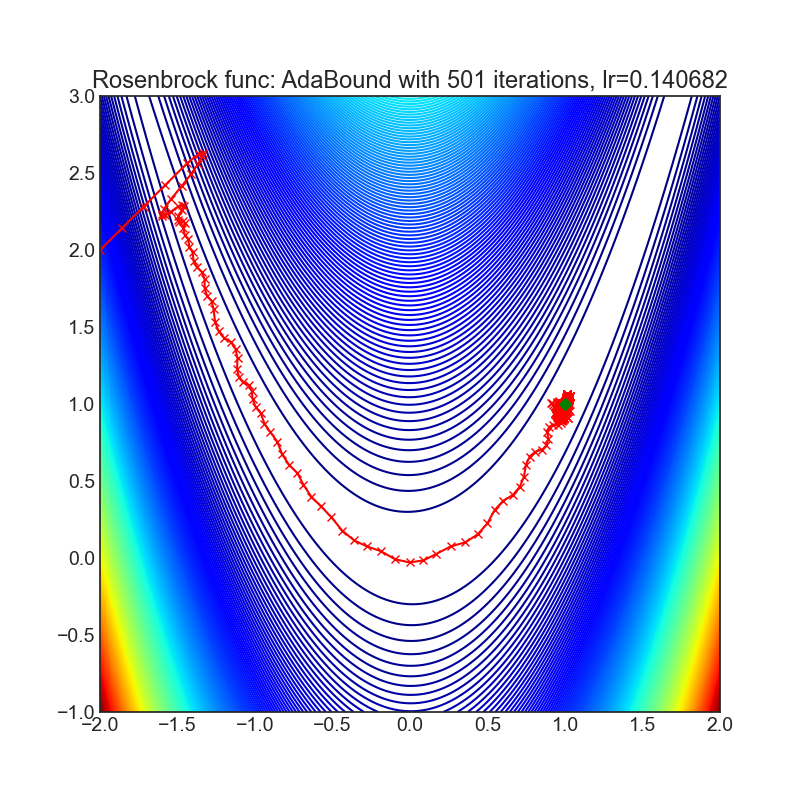}
    \end{subfigure}
    \hfill
    \begin{subfigure}[b]{0.24\textwidth}
        \includegraphics[width=\linewidth]{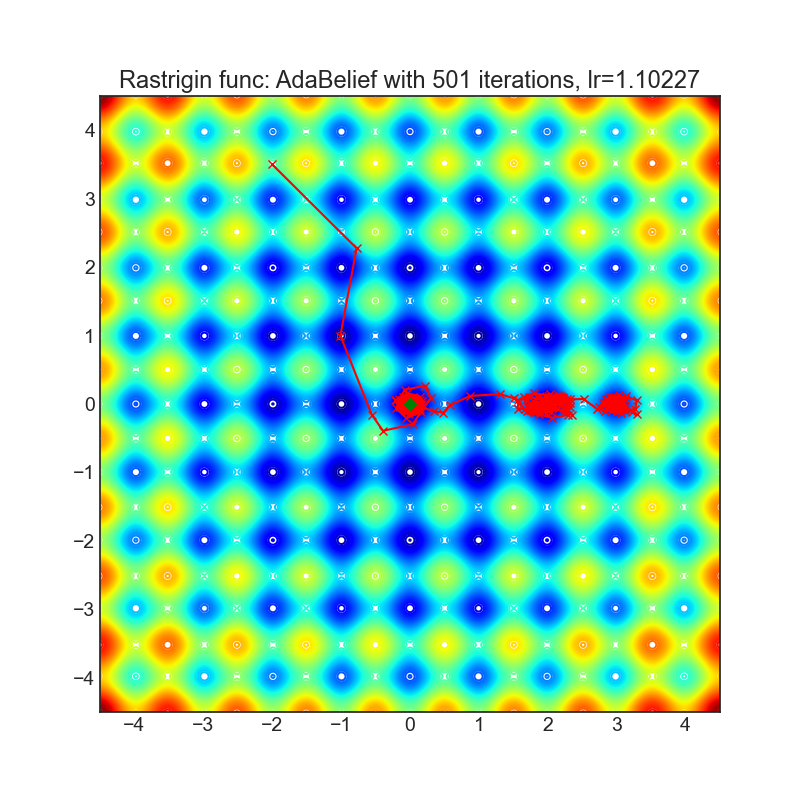}
    \end{subfigure}
    \hfill
    \begin{subfigure}[b]{0.24\textwidth}
        \includegraphics[width=\linewidth]{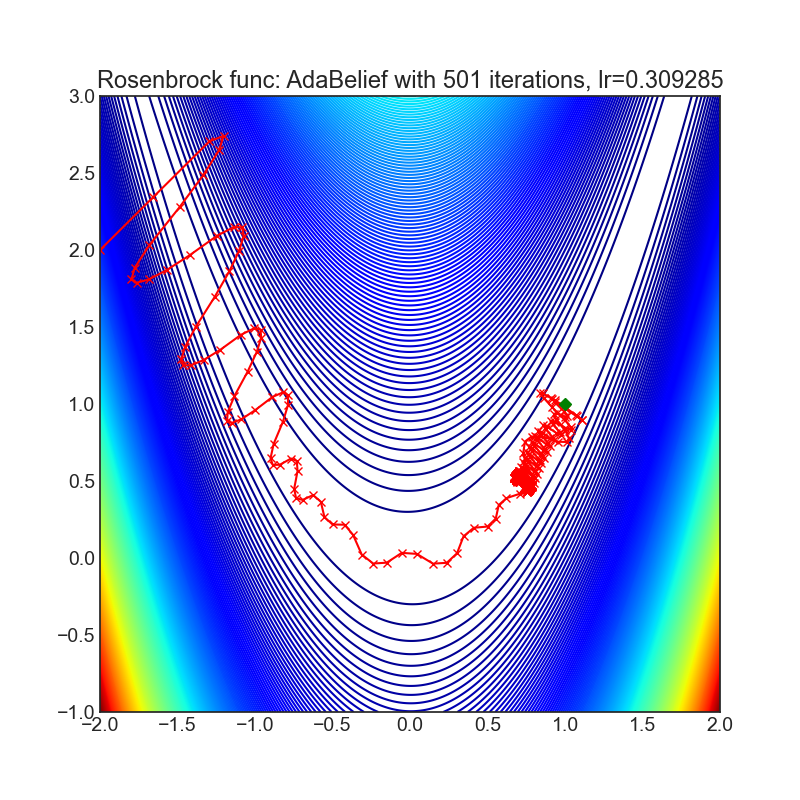}
    \end{subfigure}

    \vspace{-5mm}
    {\small \makebox[0.49\textwidth]{AdaBound} \hfill \makebox[0.49\textwidth]{AdaBelief}}
    \vspace{-1mm}
    
    \caption{Optimization trajectory comparison using searched hyperparameters. The grid shows Rastrigin (columns 1, 3) and Rosenbrock (columns 2, 4) functions. Anon ($\gamma \leq 0$) tends to explore flatter regions, while Anon ($\gamma \geq 1$) and adaptive methods converge quickly to sharp minima.}
    \label{fig:trajectory_comparison}
\end{figure}

To better understand optimizer behavior in complex landscapes, we visualize their trajectories on two classical benchmark functions: Rosenbrock and Rastrigin. Rosenbrock tests the ability to follow narrow, curved valleys to a global minimum at (1, 1), evaluating stability and curvature sensitivity. Rastrigin challenges optimizers with a rugged landscape filled with deceptive local minima, with the global minimum at (0, 0).

\begin{figure}[htpb]
\vspace{-3mm}
    \centering
    \begin{subfigure}[b]{0.45\textwidth}
        \centering
        \includegraphics[width=0.8\linewidth]{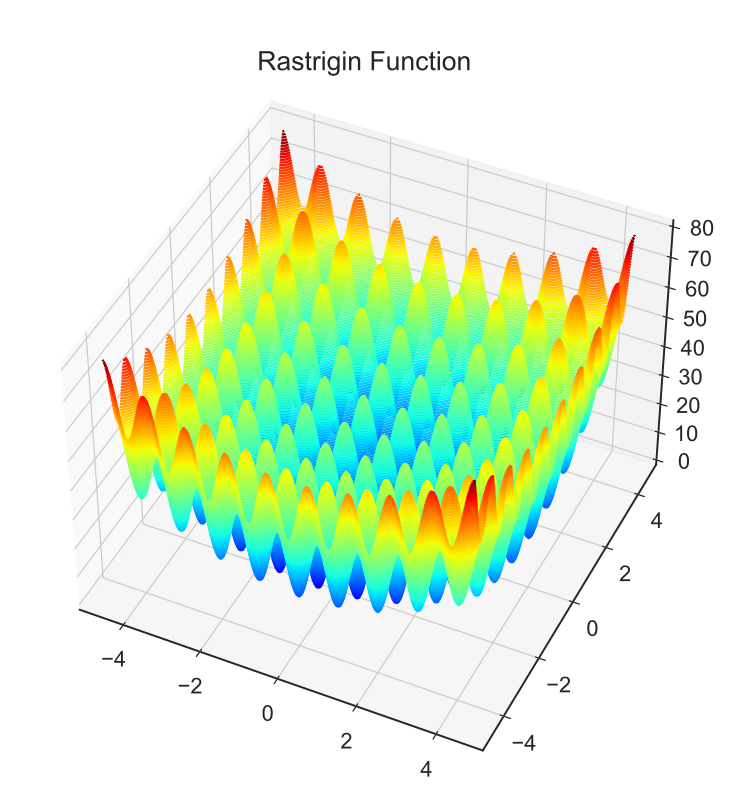}
        \caption{Rastrigin Function}
        \label{fig:rastrigin}
    \end{subfigure}
    \hfill
    \begin{subfigure}[b]{0.45\textwidth}
        \centering
        \includegraphics[width=0.8\linewidth]{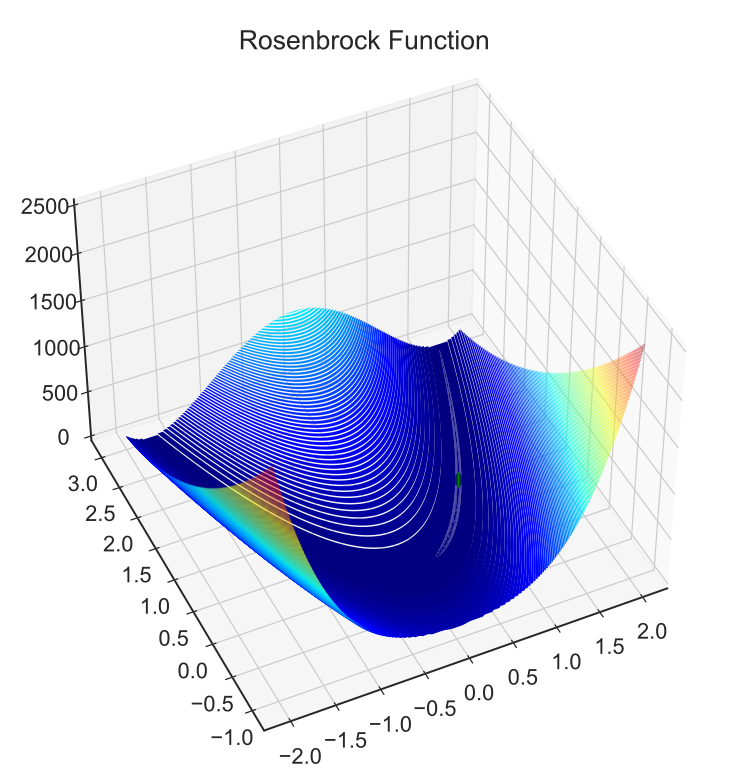}
        \caption{Rosenbrock Function}
        \label{fig:rosenbrock}
    \end{subfigure}
    \caption{3D visualization of the benchmark functions.}
    \label{fig:3d_visualization}
\end{figure} 
\allowdisplaybreaks
\section{Theorem 1 in main paper}
\label{supproof:1}
\begin{theorem}
If $\psi$ and $\psi'$ are from the same equivalence class, there is a function $f: \mathbb{N}_+\rightarrow \mathbb{R}_+$ that makes $\psi_n(\boldsymbol{x}_{1:n})=\psi'_n(\boldsymbol{x}_{1:n})f(n)$ for any $\boldsymbol{x}_{1:n}\in\mathbb{R}^n$ and any $n\in\mathbb{N}_+$.
\end{theorem}
\begin{proof}
Let $h(k;\vg_{1:n})=\ln\psi_n(k\vg_{1:n})-\ln\psi'_n(k\vg_{1:n}), h:\mathbb{R}\rightarrow\mathbb{R}$. Because $\psi_n$ and $\psi'_n$ are continuous, $h$ is continuous.

When $k\ne0$, we have
\begin{align}
h^{'}(k;\vx_{1:n})=&\lim_{\Delta k\rightarrow0}\frac{\ln\psi_n((k+\Delta k)\vx_{1:n})-\ln\psi'_n((k+\Delta k)\vx_{1:n})-\ln\psi_n(k\vx_{1:n})+\ln\psi'_n(k\vx_{1:n})}{\Delta k}\nonumber\\
    =&\frac{1}{k}\lim_{\Delta k\rightarrow0}\frac{\ln\psi_n((1+\Delta k/k)k\vx_{1:n})-\ln\psi'_n((1+\Delta k/k)k\vx_{1:n})-\ln\psi_n(k\vx_{1:n})+\ln\psi'_n(k\vx_{1:n})}{\Delta k/k}\nonumber\\
    =&\frac{1}{k}\lim_{\Delta k\rightarrow0}\frac{[\ln\psi_n((1+\Delta k/k)k\vx_{1:n})-\ln\psi_n(k\vx_{1:n})]-[\ln\psi'_n((1+\Delta k/k)k\vx_{1:n})-\ln\psi'_n(k\vx_{1:n})]}{\Delta k/k}\nonumber\\
    =&\frac{1}{k}\Big[A_n(\psi,\vx_{1:n})-A_n(\psi',\vx_{1:n})\Big]\nonumber\\
    =&\frac{1}{k}\cdot0 \,\,\,\,\,\,\,\,\Big(Since\  \psi\ and\  \psi'\ are\ in \ the \ same \ class\Big)\nonumber\\
    =&0
\end{align}
So $h(k;\vx_{1:n})=C_1$ when $k>0$, $h(k;\vx_{1:n})=C_2$ when $k<0$. And because $h$ is continuous, we have $C_1=C_2=h(0;\vx_{1:n})=\ln\frac{\psi_n(0)}{\psi'_n(0)}$.

Therefore, we have $\frac{\psi_n(k\vx_{1:n})}{\psi'_n(k\vx_{1:n})}=\frac{\psi_n(0)}{\psi'_n(0)}$ for $\forall k\in\mathbb{R}$.

And since $\vx_{1:n}$ can be any vector $\in\mathbb{R}^n$ and any $n\in\mathbb{N_+}$, we have $\frac{\psi_n(\vx_{1:n})}{\psi_n'(\vx_{1:n})}=\frac{\psi_n(0)}{\psi'_n(0)}$ for $\forall \vx_{1:n}\in\mathbb{R}^n,\ \forall n\in\mathbb{N_+}$.

Let $f(n)=\frac{\psi_n(0)}{\psi'_n(0)}$, we have $\psi_n(\boldsymbol{x}_{1:n})=\psi'_n(\boldsymbol{x}_{1:n})f(n)$ for any $\boldsymbol{x}_{1:n}\in\mathbb{R}^n$ and any $n\in\mathbb{N}_+$.
\end{proof}

\section{Theorem 2 in main paper}
\label{supproof:2}
\begin{theorem}
    For the optimizer Anon described in Algorithm~\ref{algo:anon}, the adaptivity of Anon in i-th dimension is $\in[\gamma(1-k),\gamma)$, or $(\gamma,\gamma(1-k)]$ if $\gamma<0$, where $k=\epsilon/(\min_{j\in[\tilde{a}_t]} \textit{EMA}(\vg^2_{a_{j-1}+1:a_j,i};\beta_2)+\epsilon)\in(0,1]$.
\end{theorem}
\begin{proof}
We let $f_{n,\gamma}(\boldsymbol{x})=\beta_3^{-n}(1-\beta_3\1_{n>1})\textit{EMA}^\gamma(\boldsymbol{x}_{{a_{n-1}+1}:{a_n}}^2+\epsilon;\beta_2)$, so we have
\begin{align}
    A(\psi,\boldsymbol{g}_{1:t,i})=&\ \nabla_k\ln{\Bigg({\sum_{j=1}^{\tilde{a}_t}\beta_3^{\tilde{a}_t}f_{j,\gamma}(k\boldsymbol{g}_{1:t,i})}\Bigg)}^{1/2}\Bigg|_{k=1}\nonumber\\
    =&\ \frac{\gamma\sum_{j=1}^{\tilde{a}_t}\beta_3^{\tilde{a}_t}f_{j,\gamma-1}(\boldsymbol{g}_{1:t,i})\textit{EMA}(\boldsymbol{g}_{a_{j-1}+1:a_j,i}^2;\beta_2)}{{\sum_{j=1}^{\tilde{a}_t}\beta_3^{\tilde{a}_t}f_{j,\gamma}(\boldsymbol{g}_{1:t,i})}}\nonumber\\
    =&\ \frac{\gamma\sum_{j=1}^{\tilde{a}_t}\beta_3^{\tilde{a}_t}f_{j,\gamma-1}(\boldsymbol{g}_{1:t,i})[\textit{EMA}(\boldsymbol{g}_{a_{j-1}+1:a_j,i}^2+\epsilon;\beta_2)-\epsilon]}{{\sum_{j=1}^{\tilde{a}_t}\beta_3^{\tilde{a}_t}f_{j,\gamma}(\boldsymbol{g}_{1:t,i})}}\nonumber\\
    =&\ \frac{\gamma\sum_{j=1}^{\tilde{a}_t}\beta_3^{\tilde{a}_t}f_{j,\gamma}(\boldsymbol{g}_{1:t,i})-\gamma\epsilon\sum_{j=1}^{\tilde{a}_t}\beta_3^{\tilde{a}_t}f_{j,\gamma-1}(\boldsymbol{g}_{1:t,i})}{{\sum_{j=1}^{\tilde{a}_t}\beta_3^{\tilde{a}_t}f_{j,\gamma}(\boldsymbol{g}_{1:t,i})}}\nonumber\\
    =&\ \gamma \Bigg(1-\epsilon\cdot\frac{\sum_{j=1}^{\tilde{a}_t}\beta_3^{\tilde{a}_t}f_{j,\gamma-1}(\boldsymbol{g}_{1:t,i})}{{\sum_{j=1}^{\tilde{a}_t}\beta_3^{\tilde{a}_t}f_{j,\gamma}(\boldsymbol{g}_{1:t,i})}}\Bigg)\nonumber\\
    =&\ \gamma \Bigg(1-\epsilon\cdot\frac{\sum_{j=1}^{\tilde{a}_t}f_{j,\gamma-1}(\boldsymbol{g}_{1:t,i})}{{\sum_{j=1}^{\tilde{a}_t}f_{j,\gamma}(\boldsymbol{g}_{1:t,i})}}\Bigg)\label{supeq:AofAnon}
\end{align}
Since $k={\epsilon}/(\min_{j\in[\tilde{a}_t]} \textit{EMA}(\vg^2_{a_{j-1}+1:a_j,i};\beta_2)+\epsilon)\in(0,1]$, we have $A(\psi,\boldsymbol{g}_{1:t,i})\in[\gamma(1-k),\gamma)$, or $(\gamma,\gamma(1-k)]$ if $\gamma<0$.
\end{proof}

\section{Theorem 3 in main paper}
\label{supproof:3}
For simplicity, we omit the debiasing step in theoretical analysis as in \citet{reddi2019convergence}. It is easy to prove that the analysis also applies to the de-biased version. 
\begin{lemma}{\citep{mcmahan2010adaptive}}
\label{suplemma:1}
For any $Q \in S^d_{+}$ and convex feasible set $\mathcal{F} \subset \mathbb{R}^d$, suppose $u_1 = \operatorname*{min}_{x \in \mathcal{F}} \Big \Vert Q^{1/2}(x-z_1) \Big \Vert$ and $u_2 = \operatorname*{min}_{x \in \mathcal{F}} \Big \Vert Q^{1/2}(x - z_2) \Big \Vert$, then we have $ \Big \Vert Q^{1/2}(u_1 - u_2)  \Big \Vert \leq \Big \Vert Q^{1/2} (z_1 - z_2)  \Big \Vert$.
\end{lemma}

\begin{theorem}
\label{suptheorem:convex}
{(Convergence analysis for online convex optimization)}
Let $\{\theta_t\}$ and $\{v_k\}$ be the sequence obtained by Algorithm~\ref{algo:anon}, $\gamma\in\mathbb{R}$, $\beta_1 \in [0,1)$, $\beta_2 \in [0,1)$, $\beta_{1,t+1}\in [0,\beta_{1,t}]$, $\beta_{1,1}=\beta_{1}$, $\eta(t) = \frac{\eta_0}{\sqrt{t}}$, for $\forall t \in [T]$. 
 Assume that $\Vert x-y \Vert_\infty \leq D_\infty$ for $\forall x,y\in \mathcal{F}$. Suppose $f(\theta)$ is a convex function, $ \Vert  g_t  \Vert_\infty \leq G_\infty$ , for $\forall t \in [T]$, $\theta \in \mathcal{F}$. Let $C_l=\min(G_\infty^{-\gamma},\epsilon^{-\gamma/2})$, $C_u=\max(G_\infty^{-\gamma},\epsilon^{-\gamma/2})$, where $\epsilon\in\mathbb{R_+}$ is a small constant set in Algorithm~\ref{algo:anon}. The optimal point of $f$ is denoted as $\theta^*$. For $\{ \theta_t \}$ generated by Anon, there is a bound on the regret: 
\resizebox{1.0\linewidth}{!}{
\begin{minipage}{\linewidth}
\begin{align*}
\sum_{t=1}^T [f_t(\theta_t) - f_t(\theta^*)] \leq&\frac{dD_\infty^2C_l^{-1}}{(1 - \beta_{1})\eta_0}\bigg(\sqrt{T}+ \sum_{k=1}^{\tilde{a}_T-1} \sqrt{a_{k+1}}\bigg)+ \sum_{t=1}^{T-1} \Bigg[\frac{\beta_{1,t+1}{\mathbb{I}}_{\beta_{1,t+1}>\beta_{1,t}}dD_\infty^2}{2 C_l\eta_{t+1} (1 - \beta_{1})^2}  \Bigg] \nonumber \\
        &+ \frac{dD_\infty^2}{2 C_l\eta_{1} (1 - \beta_{1})} + \frac{dD_\infty G_{\infty}}{1-\beta_{1}}\sum_{t=1}^T{\beta_{1,t}} + \frac{dG_\infty^2 C_u\eta_0}{1-\beta_{1}}\sqrt{T}    
\end{align*}
\end{minipage}
}
\end{theorem} 
\begin{proof}
\begin{align}
    \boldsymbol{v}_k&=\sqrt{2/(\frac{1}{\boldsymbol{v}_{k-1}^2}+\boldsymbol{\sigma}_k^\gamma)} \text{ if } k>0 \text{ else } \boldsymbol{\sigma}_k^{-\gamma/2} \nonumber \\
    \frac{1}{\boldsymbol{v}_k^2}&=\frac{\frac{1}{\boldsymbol{v}_{k-1}^2}+\boldsymbol{\sigma}_k^\gamma}{2}\text{ if } k>0 \text{ else } \boldsymbol{\sigma}_k^{\gamma}\nonumber \\
    \frac{1}{\boldsymbol{v}_k^2}&=\sum_{i=0}^k\frac{\boldsymbol{\sigma}_{i}^\gamma}{2^{\min(k-i+1,k)}} \nonumber \\
    \frac{1}{\boldsymbol{v}_k^2}&=\sum_{i=0}^k\frac{\textit{EMA}^\gamma(\boldsymbol{g}^2_{\lfloor2^{k-1}+1\rfloor:2^{k}}+\epsilon;\beta_2)}{2^{\min(k-i+1,k)}}
    \label{supeq:vk2}
\end{align}
Since $\Vert  g_t  \Vert_\infty \leq G_\infty$, $C_l=\min(G_\infty^{-\gamma},\epsilon^{-\gamma/2})$ and $C_u=\max(G_\infty^{-\gamma},\epsilon^{-\gamma/2})$, from \ref{supeq:vk2}, we have:
\begin{align}
    \frac{1}{v_{k,i}^2}&\in\Bigg[\sum_{i=0}^k\frac{C_u^{-2}}{2^{\min(k-i+1,k)}}, \sum_{i=0}^k\frac{C_l^{-2}}{2^{\min(k-i+1,k)}}\Bigg] \nonumber \\
    \frac{1}{v_{k,i}^2}&\in\Big[C_u^{-2}, C_l^{-2}\Big] \nonumber \\
    {v_{k,i}}&\in[C_l, C_u]
    \label{supeq:boundv}
\end{align}
Let $\eta_t=\eta(t)$.
\begin{equation*}
    \theta_{t+1} = \prod_{\mathcal{F}, V_{\tilde{a}_t}^{-1}}(\theta_t - \eta_tV_{\tilde{a}_t}m_t) = \operatorname*{min}_{\theta \in \mathcal{F}} \Big \Vert V_{\tilde{a}_t}^{-1/2}(\theta - (\theta_t - \eta_tV_{\tilde{a}_t}m_t)) \Big \Vert
\end{equation*}
Note that $\prod_{\mathcal{F}, V_{\tilde{a}_t}^{-1}}(\theta^*) = \theta^*$ since $\theta^* \in \mathcal{F}$. Use $\theta^*_i$ and $\theta_{t,i}$ to denote the i-th dimension of $\theta^*$ and $\theta_t$ respectively.
From lemma \eqref{suplemma:1}, using $u_1 = \theta_{t+1}$ and $u_2 = \theta^*$, we have:\\
\begin{align}
\label{supeq:1}
    \Big \Vert V_{\tilde{a}_t}^{-1/2}(\theta_{t+1} - \theta^*) \Big \Vert^2  \leq& \Big \Vert V_{\tilde{a}_t}^{-1/2}(\theta_t - \eta_{t} V_{\tilde{a}_t}m_t - \theta^*) \Big \Vert^2 \nonumber \\
    =&  \Big \Vert V_{\tilde{a}_t}^{-1/2}(\theta_t - \theta^*) \Big \Vert^2 + \eta_{t}^2 \Big \Vert V^{1/2}_{\tilde{a}_t}m_t \Big \Vert^2 - 2 \eta_t \langle m_t, \theta_t - \theta^* \rangle \nonumber \\
    =& \Big \Vert V_{\tilde{a}_t}^{-1/2}(\theta_t - \theta^*) \Big \Vert^2 + \eta_{t}^2 \Big \Vert V^{1/2}_{\tilde{a}_t}m_t \Big \Vert^2  \nonumber \\
    &- 2 \eta_t \langle \beta_{1,t}m_{t-1} + (1 - \beta_{1,t}) g_t, \theta_t - \theta^* \rangle 
\end{align}
Note that $\beta_1 \in [0,1)$ and $\beta_2 \in [0,1)$, rearranging inequality \eqref{supeq:1}, we have:
\begin{align}
    \langle g_t, \theta_t - \theta^* \rangle  \leq& \frac{1}{2 \eta_{t} (1 - \beta_{1,t})} \Big ( \Big \Vert V_{\tilde{a}_t}^{-1/2}(\theta_t - \theta^*)\Big \Vert^2 - \Big \Vert V_{\tilde{a}_t}^{-1/2} (\theta_{t+1} - \theta^*) \Big \Vert^2 \Big ) \nonumber \\
        &+ \frac{\eta_{t}}{2(1-\beta_{1,t})} \Big \Vert V_{\tilde{a}_t}^{1/2} m_t\Big \Vert^2 + \frac{\beta_{1,t}}{1 - \beta_{1,t}} \langle m_{t-1}, \theta^* -\theta_t \rangle \nonumber \\
    \leq& \frac{1}{2 \eta_{t} (1 - \beta_{1,t})} \Big (  \Big \Vert V_{\tilde{a}_t}^{-1/2}(\theta_t - \theta^*)\Big \Vert^2 - \Big \Vert V_{\tilde{a}_t}^{-1/2} (\theta_{t+1} - \theta^*) \Big \Vert^2  \Big ) \nonumber \\
        &+ \frac{\eta_{t}}{2(1-\beta_{1,t})} \Big \Vert V_{\tilde{a}_t}^{1/2} m_t\Big \Vert^2 + \frac{\beta_{1,t}}{1-\beta_{1,t}} \Big \Vert m_{t-1} \Big\Vert \Big\Vert \theta^*-\theta_{t} \Big\Vert \nonumber \\
        &\Big (\textit{Cauchy-Schwartz's inequality: } \langle u,v \rangle \leq \Big \Vert u \Big \Vert v \Big \Vert \Big )\nonumber \\
    \leq& \frac{1}{2 \eta_{t} (1 - \beta_{1,t})} \Big (  \Big \Vert V_{\tilde{a}_t}^{-1/2}(\theta_t - \theta^*)\Big \Vert^2 - \Big \Vert V_{\tilde{a}_t}^{-1/2} (\theta_{t+1} - \theta^*) \Big \Vert^2  \Big ) \nonumber \\
        &+ \frac{\eta_{t}}{2(1-\beta_{1,t})} \Big \Vert V_{\tilde{a}_t}^{1/2} m_t\Big \Vert^2 + \frac{\beta_{1,t}}{1-\beta_{1,t}} \Big \Vert m_{t-1} \Big\Vert \sqrt{d}D_\infty \nonumber \\
        &\Big (\textit{Since  }\Vert x-y \Vert_\infty \leq D_\infty \textit{, for }\forall x,y\in \mathcal{F} \Big)\nonumber \\
    =& \frac{1}{2 \eta_{t} (1 - \beta_{1,t})} \Big (  \Big \Vert V_{\tilde{a}_t}^{-1/2}(\theta_t - \theta^*)\Big \Vert^2 - \Big \Vert V_{\tilde{a}_t}^{-1/2} (\theta_{t+1} - \theta^*) \Big \Vert^2  \Big ) \nonumber \\
        &+ \frac{\eta_{t}}{2(1-\beta_{1,t})} \Big \Vert V_{\tilde{a}_t}^{1/2} m_t\Big \Vert^2 + \frac{\beta_{1,t}\sqrt{d}D_\infty}{1-\beta_{1,t}} \sqrt{\sum_{i=1}^{d}\textit{EMA}^2(g_{1:t-1,i};\beta_1)}  \nonumber \\
    \leq& \frac{1}{2 \eta_{t} (1 - \beta_{1,t})} \Big (  \Big \Vert V_{\tilde{a}_t}^{-1/2}(\theta_t - \theta^*)\Big \Vert^2 - \Big \Vert V_{\tilde{a}_t}^{-1/2} (\theta_{t+1} - \theta^*) \Big \Vert^2  \Big ) \nonumber \\
        &+ \frac{\eta_{t}}{2(1-\beta_{1,t})} \Big \Vert V_{\tilde{a}_t}^{1/2} m_t\Big \Vert^2 + \frac{\beta_{1,t}\sqrt{d}D_\infty}{1-\beta_{1,t}} \sqrt{\sum_{i=1}^{d}G_{\infty}^2} \nonumber \\
        &\Big (\textit{Since  }\Vert  g_t  \Vert_\infty \leq G_\infty\Big)\nonumber \\
    \leq& \frac{1}{2 \eta_{t} (1 - \beta_{1,t})} \Big (  \Big \Vert V_{\tilde{a}_t}^{-1/2}(\theta_t - \theta^*)\Big \Vert^2 - \Big \Vert V_{\tilde{a}_t}^{-1/2} (\theta_{t+1} - \theta^*) \Big \Vert^2  \Big ) \nonumber \\
        &+ \frac{\eta_{t}}{2(1-\beta_{1,t})} \Big \Vert V_{\tilde{a}_t}^{1/2} m_t\Big \Vert^2 + \frac{\beta_{1,t}dD_\infty}{1-\beta_{1,t}} G_{\infty} \nonumber \\
    =& \frac{1}{2 \eta_{t} (1 - \beta_{1,t})} \Big (  \Big \Vert V_{\tilde{a}_t}^{-1/2}(\theta_t - \theta^*)\Big \Vert^2 - \Big \Vert V_{\tilde{a}_t}^{-1/2} (\theta_{t+1} - \theta^*) \Big \Vert^2  \Big ) \nonumber \\
        &+ \frac{\beta_{1,t}dD_\infty G_{\infty}}{1-\beta_{1,t}} + \frac{\eta_{t}}{2(1-\beta_{1,t})} m_t^{\top} V_{\tilde{a}_t} m_t  \nonumber \\
    =& \frac{1}{2 \eta_{t} (1 - \beta_{1,t})} \Big (  \Big \Vert V_{\tilde{a}_t}^{-1/2}(\theta_t - \theta^*)\Big \Vert^2 - \Big \Vert V_{\tilde{a}_t}^{-1/2} (\theta_{t+1} - \theta^*) \Big \Vert^2  \Big ) \nonumber \\
        &+ \frac{\beta_{1,t}dD_\infty G_{\infty}}{1-\beta_{1,t}} + \frac{\eta_{t}}{2(1-\beta_{1,t})}\sum_{i=1}^d m_{t,i}^{2} v_{\tilde{a}_t,i}  \nonumber \\
    \leq& \frac{1}{2 \eta_{t} (1 - \beta_{1,t})} \Big (  \Big \Vert V_{\tilde{a}_t}^{-1/2}(\theta_t - \theta^*)\Big \Vert^2 - \Big \Vert V_{\tilde{a}_t}^{-1/2} (\theta_{t+1} - \theta^*) \Big \Vert^2  \Big ) \nonumber \\
        &+ \frac{\beta_{1,t}dD_\infty G_{\infty}}{1-\beta_{1,t}} + \frac{\eta_{t}}{2(1-\beta_{1,t})}\sum_{i=1}^d m_{t,i}^{2} C_u  \nonumber \\
        &\Big (\textit{Apply formula}~\eqref{supeq:boundv}\Big)\nonumber \\
    \leq& \frac{1}{2 \eta_{t} (1 - \beta_{1,t})} \Big (  \Big \Vert V_{\tilde{a}_t}^{-1/2}(\theta_t - \theta^*)\Big \Vert^2 - \Big \Vert V_{\tilde{a}_t}^{-1/2} (\theta_{t+1} - \theta^*) \Big \Vert^2  \Big ) \nonumber \\
        &+ \frac{\beta_{1,t}dD_\infty G_{\infty}}{1-\beta_{1,t}} + \frac{dG_\infty^2 C_u\eta_{t}}{2(1-\beta_{1,t})} 
        \label{supeq:2}
\end{align}
By convexity of $f$, we have:
\begin{align}
\label{supeq:F}
\sum_{t=1}^T f_t(\theta_t) - f_t(\theta^*)  \leq& \sum_{t=1}^T\langle g_t, \theta_t - \theta^* \rangle \nonumber \\
    \leq&  \sum_{t=1}^T \Bigg[ \frac{1}{2 \eta_{t} (1 - \beta_{1,t})} \Big (  \Big \Vert V_{\tilde{a}_t}^{-1/2}(\theta_t - \theta^*)\Big \Vert^2 - \Big \Vert V_{\tilde{a}_t}^{-1/2} (\theta_{t+1} - \theta^*) \Big \Vert^2  \Big ) \nonumber \\
        &+ \frac{\beta_{1,t}dD_\infty G_{\infty}}{1-\beta_{1,t}} + \frac{dG_\infty^2 C_u\eta_{t}}{2(1-\beta_{1,t})}  \Bigg] \nonumber \\
	&\Big( \textit{By formula}~\eqref{supeq:2} \Big) \nonumber \\
    \leq&  \sum_{t=1}^T \Bigg[ \frac{1}{2 \eta_{t} (1 - \beta_{1,t})} \Big (  \Big \Vert V_{\tilde{a}_t}^{-1/2}(\theta_t - \theta^*)\Big \Vert^2 - \Big \Vert V_{\tilde{a}_t}^{-1/2} (\theta_{t+1} - \theta^*) \Big \Vert^2  \Big )\Bigg] \nonumber \\
        &+ \frac{1}{1-\beta_{1}}\sum_{t=1}^T\Bigg({\beta_{1,t}dD_\infty G_{\infty}} + \frac{dG_\infty^2 C_u\eta_{t}}{2}\Bigg)   \nonumber \\
	&\Big( \textit{Since } 0\leq\beta_{1,t}\leq\beta_1<1 \Big) \nonumber \\
    =&  \sum_{t=1}^T \Bigg[ \frac{1}{2 \eta_{t} (1 - \beta_{1,t})} \Big (  \Big \Vert V_{\tilde{a}_t}^{-1/2}(\theta_t - \theta^*)\Big \Vert^2 - \Big \Vert V_{\tilde{a}_t}^{-1/2} (\theta_{t+1} - \theta^*) \Big \Vert^2  \Big )\Bigg] \nonumber \\
        &+ \frac{1}{1-\beta_{1}}\sum_{t=1}^T\Bigg({\beta_{1,t}dD_\infty G_{\infty}} + \frac{dG_\infty^2 C_u\eta_0}{2\sqrt{t}}\Bigg)   \nonumber \\
    \leq&  \sum_{t=1}^T \Bigg[ \frac{1}{2 \eta_{t} (1 - \beta_{1,t})} \Big (  \Big \Vert V_{\tilde{a}_t}^{-1/2}(\theta_t - \theta^*)\Big \Vert^2 - \Big \Vert V_{\tilde{a}_t}^{-1/2} (\theta_{t+1} - \theta^*) \Big \Vert^2  \Big )\Bigg] \nonumber \\
        &+ \frac{dD_\infty G_{\infty}}{1-\beta_{1}}\sum_{t=1}^T{\beta_{1,t}} + \frac{dG_\infty^2 C_u\eta_0}{1-\beta_{1}}\int_0^T\frac{1}{2\sqrt{t}}\, \mathrm{d}t   \nonumber \\
	&\Big( \textit{Since } \eta_t=\eta_0/\sqrt{t} \Big) \nonumber \\
    =&  \sum_{t=1}^T \Bigg[ \frac{1}{2 \eta_{t} (1 - \beta_{1,t})} \Big (  \Big \Vert V_{\tilde{a}_t}^{-1/2}(\theta_t - \theta^*)\Big \Vert^2 - \Big \Vert V_{\tilde{a}_t}^{-1/2} (\theta_{t+1} - \theta^*) \Big \Vert^2  \Big )\Bigg] \nonumber \\
        &+ \frac{dD_\infty G_{\infty}}{1-\beta_{1}}\sum_{t=1}^T{\beta_{1,t}} + \frac{dG_\infty^2 C_u\eta_0}{1-\beta_{1}}\sqrt{T}   \nonumber \\
    \leq& \sum_{t=1}^{T-1} \Bigg[ \frac{1}{2 \eta_{t+1} (1 - \beta_{1,t+1})}  \Big \Vert V_{\tilde{a}_{t+1}}^{-1/2}(\theta_{t+1} - \theta^*)\Big \Vert^2 - \frac{1}{2 \eta_{t} (1 - \beta_{1,t})}\Big \Vert V_{\tilde{a}_t}^{-1/2} (\theta_{t+1} - \theta^*) \Big \Vert^2 \Bigg] \nonumber \\
        &+ \frac{1}{2 \eta_{1} (1 - \beta_{1})} \Big \Vert V_{1}^{-1/2}(\theta_1 - \theta^*)\Big \Vert^2 + \frac{dD_\infty G_{\infty}}{1-\beta_{1}}\sum_{t=1}^T{\beta_{1,t}} + \frac{dG_\infty^2 C_u\eta_0}{1-\beta_{1}}\sqrt{T}   \nonumber \\
    =& \sum_{t=1}^{T-1} \Bigg[ \frac{1}{2 \eta_{t+1} (1 - \beta_{1,t})}  \Big \Vert V_{\tilde{a}_{t+1}}^{-1/2}(\theta_{t+1} - \theta^*)\Big \Vert^2 - \frac{1}{2 \eta_{t} (1 - \beta_{1,t})}\Big \Vert V_{\tilde{a}_t}^{-1/2} (\theta_{t+1} - \theta^*) \Big \Vert^2 \nonumber \\
        &+\frac{\beta_{1,t+1}-\beta_{1,t}}{2 \eta_{t+1} (1 - \beta_{1,t})(1 - \beta_{1,t+1})}  \Big \Vert V_{\tilde{a}_{t+1}}^{-1/2}(\theta_{t+1} - \theta^*)\Big \Vert^2 \Bigg]\nonumber\\
        &+ \frac{1}{2 \eta_{1} (1 - \beta_{1})} \Big \Vert V_{1}^{-1/2}(\theta_1 - \theta^*)\Big \Vert^2 + \frac{dD_\infty G_{\infty}}{1-\beta_{1}}\sum_{t=1}^T{\beta_{1,t}} + \frac{dG_\infty^2 C_u\eta_0}{1-\beta_{1}}\sqrt{T}   \nonumber \\
    =& \sum_{t=1}^{T-1} \Bigg\{ \frac{1}{2(1 - \beta_{1,t})}  \Bigg [ {(\theta_{t+1} - \theta^*)}^\top\Bigg(\frac{V_{\tilde{a}_{t+1}}^{-1}}{\eta_{t+1}}-\frac{V_{\tilde{a}_t}^{-1}}{\eta_{t}}\Bigg)(\theta_{t+1} - \theta^*)\Bigg ]  \Bigg\} \nonumber \\
        &+ \sum_{t=1}^{T-1} \Bigg[\frac{\beta_{1,t+1}-\beta_{1,t}}{2 \eta_{t+1} (1 - \beta_{1,t})(1 - \beta_{1,t+1})}  \Big \Vert V_{\tilde{a}_{t+1}}^{-1/2}(\theta_{t+1} - \theta^*)\Big \Vert^2 \Bigg]\nonumber\\
        &+ \frac{1}{2 \eta_{1} (1 - \beta_{1,t})} \Big \Vert V_{1}^{-1/2}(\theta_1 - \theta^*)\Big \Vert^2 + \frac{dD_\infty G_{\infty}}{1-\beta_{1}}\sum_{t=1}^T{\beta_{1,t}} + \frac{dG_\infty^2 C_u\eta_0}{1-\beta_{1}}\sqrt{T}   \nonumber \\
    =& \sum_{k=1}^{\tilde{a}_T}\sum_{t=a_k}^{\min(T,a_{k+1})-1} \Bigg\{ \frac{1}{2(1 - \beta_{1,t})}  \Bigg [ {(\theta_{t+1} - \theta^*)}^\top\Bigg(\frac{V_{\tilde{a}_{t+1}}^{-1}}{\eta_{t+1}}-\frac{V_{\tilde{a}_t}^{-1}}{\eta_{t}}\Bigg)(\theta_{t+1} - \theta^*)\Bigg ]  \Bigg\} \nonumber \\
        &+ \sum_{t=1}^{T-1} \Bigg[\frac{\beta_{1,t+1}-\beta_{1,t}}{2 \eta_{t+1} (1 - \beta_{1,t})(1 - \beta_{1,t+1})}  \Big \Vert V_{\tilde{a}_{t+1}}^{-1/2}(\theta_{t+1} - \theta^*)\Big \Vert^2 \Bigg]\nonumber\\
        &+ \frac{1}{2 \eta_{1} (1 - \beta_{1})} \Big \Vert V_{1}^{-1/2}(\theta_1 - \theta^*)\Big \Vert^2 + \frac{dD_\infty G_{\infty}}{1-\beta_{1}}\sum_{t=1}^T{\beta_{1,t}} + \frac{dG_\infty^2 C_u\eta_0}{1-\beta_{1}}\sqrt{T}   \nonumber\\
    =&\sum_{k=1}^{\tilde{a}_T}\sum_{t=a_k}^{\min(T,a_{k+1})-2} \Bigg\{ \frac{1}{2(1 - \beta_{1,t})}  \Bigg [ {(\theta_{t+1} - \theta^*)}^\top\Bigg(\frac{V_{\tilde{a}_{t+1}}^{-1}}{\eta_{t+1}}-\frac{V_{\tilde{a}_t}^{-1}}{\eta_{t}}\Bigg)(\theta_{t+1} - \theta^*)\Bigg ]  \Bigg\}   \nonumber \\
        &+ \sum_{k=1}^{\tilde{a}_T-1} \Bigg\{ \frac{1}{2(1 - \beta_{1a_{k+1}-1})}  \Bigg [ {(\theta_{a_{k+1}} - \theta^*)}^\top\Bigg(\frac{V_{k+1}^{-1}}{\eta_{a_{k+1}}}-\frac{V_{k}^{-1}}{\eta_{a_{k+1}-1}}\Bigg)(\theta_{a_{k+1}} - \theta^*)\Bigg ]  \Bigg\}\nonumber \\
        &+ \sum_{t=1}^{T-1} \Bigg[\frac{\beta_{1,t+1}-\beta_{1,t}}{2 \eta_{t+1} (1 - \beta_{1,t})(1 - \beta_{1,t+1})}  \Big \Vert V_{\tilde{a}_{t+1}}^{-1/2}(\theta_{t+1} - \theta^*)\Big \Vert^2 \Bigg]\nonumber\\
        &+ \frac{1}{2 \eta_{1} (1 - \beta_{1})} \Big \Vert V_{1}^{-1/2}(\theta_1 - \theta^*)\Big \Vert^2 + \frac{dD_\infty G_{\infty}}{1-\beta_{1}}\sum_{t=1}^T{\beta_{1,t}} + \frac{dG_\infty^2 C_u\eta_0}{1-\beta_{1}}\sqrt{T}   \nonumber\\
    =&\sum_{k=1}^{\tilde{a}_T}\sum_{t=a_k}^{\min(T,a_{k+1})-2} \Bigg\{ \frac{1}{2(1 - \beta_{1,t})}  \Bigg [ {(\theta_{t+1} - \theta^*)}^\top\Bigg(\frac{V_{k}^{-1}}{\eta_{t+1}}-\frac{V_{k}^{-1}}{\eta_{t}}\Bigg)(\theta_{t+1} - \theta^*)\Bigg ]  \Bigg\}   \nonumber \\
        &+ \sum_{k=1}^{\tilde{a}_T-1} \Bigg\{ \frac{1}{2(1 - \beta_{1a_{k+1}-1})}  \Bigg [ {(\theta_{a_{k+1}} - \theta^*)}^\top\Bigg(\frac{V_{k+1}^{-1}}{\eta_{a_{k+1}}}-\frac{V_{k}^{-1}}{\eta_{a_{k+1}-1}}\Bigg)(\theta_{a_{k+1}} - \theta^*)\Bigg ]  \Bigg\}\nonumber \\
        &+ \sum_{t=1}^{T-1} \Bigg[\frac{\beta_{1,t+1}-\beta_{1,t}}{2 \eta_{t+1} (1 - \beta_{1,t})(1 - \beta_{1,t+1})}  \Big \Vert V_{\tilde{a}_{t+1}}^{-1/2}(\theta_{t+1} - \theta^*)\Big \Vert^2 \Bigg]\nonumber\\
        &+ \frac{1}{2 \eta_{1} (1 - \beta_{1})} \Big \Vert V_{1}^{-1/2}(\theta_1 - \theta^*)\Big \Vert^2 + \frac{dD_\infty G_{\infty}}{1-\beta_{1}}\sum_{t=1}^T{\beta_{1,t}} + \frac{dG_\infty^2 C_u\eta_0}{1-\beta_{1}}\sqrt{T}   \nonumber\\
    \leq&\sum_{k=1}^{\tilde{a}_T}\sum_{t=a_k}^{\min(T,a_{k+1})-2} \Bigg\{ \frac{1}{2(1 - \beta_{1})}  \Bigg [ D_\infty e_d^\top \Bigg(\frac{V_{k}^{-1}}{\eta_{t+1}}-\frac{V_{k}^{-1}}{\eta_{t}}\Bigg)D_\infty e_d\Bigg ]  \Bigg\}   \nonumber \\
        &+ \sum_{k=1}^{\tilde{a}_T-1} \Bigg\{ \frac{1}{2(1 - \beta_{1})}  \Bigg [ D_\infty e_d^\top \Bigg(\frac{C_l^{-1}\mI_d}{\eta_{a_{k+1}}}\Bigg)D_\infty e_d \Bigg ]  \Bigg\}\nonumber \\
        &+ \sum_{t=1}^{T-1} \Bigg[\frac{\beta_{1,t+1}-\beta_{1,t}}{2 \eta_{t+1} (1 - \beta_{1,t})(1 - \beta_{1,t+1})}  \Big \Vert V_{\tilde{a}_{t+1}}^{-1/2}(\theta_{t+1} - \theta^*)\Big \Vert^2 \Bigg]\nonumber\\
        &+ \frac{1}{2 \eta_{1} (1 - \beta_{1})} \Big \Vert V_{1}^{-1/2}(\theta_1 - \theta^*)\Big \Vert^2 + \frac{dD_\infty G_{\infty}}{1-\beta_{1}}\sum_{t=1}^T{\beta_{1,t}} + \frac{dG_\infty^2 C_u\eta_0}{1-\beta_{1}}\sqrt{T}   \nonumber \\
	&\Big( \textit{Since } \eta_t=\eta_0/\sqrt{t}\textit{, and }0\leq\beta_{1,t}\leq\beta_1<1 \Big) \nonumber\\
    =&\sum_{k=1}^{\tilde{a}_T-1} \Bigg\{ \frac{1}{2(1 - \beta_{1})}  \Bigg [ D_\infty e_d^\top \Bigg(\frac{V_{k}^{-1}}{\eta_{a_{k+1}-1}}-\frac{V_{k}^{-1}}{\eta_{a_{k}}}\Bigg)D_\infty e_d\Bigg ]  \Bigg\}   \nonumber \\
        &+ \frac{1}{2(1 - \beta_{1})}  \Bigg [ D_\infty e_d^\top \Bigg(\frac{V_{\tilde{a}_T}^{-1}}{\eta_{T}}-\frac{V_{\tilde{a}_T}^{-1}}{\eta_{a_{\tilde{a}_T}}}\Bigg)D_\infty e_d \Bigg ]   \nonumber \\
        &+ \sum_{k=1}^{\tilde{a}_T-1} \Bigg\{ \frac{1}{2(1 - \beta_{1})}  \Bigg [ D_\infty e_d^\top \Bigg(\frac{C_l^{-1}\mI_d}{\eta_{a_{k+1}}}\Bigg)D_\infty e_d \Bigg ]  \Bigg\}\nonumber \\
        &+ \sum_{t=1}^{T-1} \Bigg[\frac{\beta_{1,t+1}-\beta_{1,t}}{2 \eta_{t+1} (1 - \beta_{1,t})(1 - \beta_{1,t+1})}  \Big \Vert V_{\tilde{a}_{t+1}}^{-1/2}(\theta_{t+1} - \theta^*)\Big \Vert^2 \Bigg]\nonumber\\
        &+ \frac{1}{2 \eta_{1} (1 - \beta_{1})} \Big \Vert V_{1}^{-1/2}(\theta_1 - \theta^*)\Big \Vert^2 + \frac{dD_\infty G_{\infty}}{1-\beta_{1}}\sum_{t=1}^T{\beta_{1,t}} + \frac{dG_\infty^2 C_u\eta_0}{1-\beta_{1}}\sqrt{T}   \nonumber\\
    \leq&2\sum_{k=1}^{\tilde{a}_T-1} \Bigg\{ \frac{1}{2(1 - \beta_{1})}  \Bigg [ D_\infty e_d^\top \Bigg(\frac{C_l^{-1}\mI_d}{\eta_{a_{k+1}}}\Bigg)D_\infty e_d \Bigg ]  \Bigg\}\nonumber \\
        &+\frac{1}{2(1 - \beta_{1})}  \Bigg [ D_\infty e_d^\top \Bigg(\frac{C_l^{-1}\mI_d}{\eta_{T}}\Bigg)D_\infty e_d \Bigg ]\nonumber \\
        &+ \sum_{t=1}^{T-1} \Bigg[\frac{\beta_{1,t+1}-\beta_{1,t}}{2 \eta_{t+1} (1 - \beta_{1,t})(1 - \beta_{1,t+1})}  \Big \Vert V_{\tilde{a}_{t+1}}^{-1/2}(\theta_{t+1} - \theta^*)\Big \Vert^2 \Bigg]\nonumber\\
        &+ \frac{1}{2 \eta_{1} (1 - \beta_{1})} \Big \Vert V_{1}^{-1/2}(\theta_1 - \theta^*)\Big \Vert^2 + \frac{dD_\infty G_{\infty}}{1-\beta_{1}}\sum_{t=1}^T{\beta_{1,t}} + \frac{dG_\infty^2 C_u\eta_0}{1-\beta_{1}}\sqrt{T}   \nonumber\\
    \leq&\frac{dD_\infty^2C_l^{-1}}{(1 - \beta_{1})\eta_0}\sum_{k=1}^{\tilde{a}_T-1} \sqrt{a_{k+1}}\nonumber \\
        &+\frac{dD_\infty^2C_l^{-1}}{(1 - \beta_{1})\eta_0} \sqrt{T}\nonumber\nonumber \\
        &+ \sum_{t=1}^{T-1} \Bigg[\frac{\beta_{1,t+1}-\beta_{1,t}}{2 \eta_{t+1} (1 - \beta_{1,t})(1 - \beta_{1,t+1})}  \Big \Vert V_{\tilde{a}_{t+1}}^{-1/2}(\theta_{t+1} - \theta^*)\Big \Vert^2 \Bigg]\nonumber\\
        &+ \frac{1}{2 \eta_{1} (1 - \beta_{1})} \Big \Vert V_{1}^{-1/2}(\theta_1 - \theta^*)\Big \Vert^2 + \frac{dD_\infty G_{\infty}}{1-\beta_{1}}\sum_{t=1}^T{\beta_{1,t}} + \frac{dG_\infty^2 C_u\eta_0}{1-\beta_{1}}\sqrt{T}   \nonumber\\
    \leq&\frac{dD_\infty^2C_l^{-1}}{(1 - \beta_{1})\eta_0}\bigg(\sqrt{T}+ \sum_{k=1}^{\tilde{a}_T-1} \sqrt{a_{k+1}}\bigg)\nonumber \\
        &+ \sum_{t=1}^{T-1} \Bigg[\frac{\beta_{1,t+1}-\beta_{1,t}}{2 \eta_{t+1} (1 - \beta_{1,t})(1 - \beta_{1,t+1})}  \Big \Vert V_{\tilde{a}_{t+1}}^{-1/2}(\theta_{t+1} - \theta^*)\Big \Vert^2 \Bigg]\nonumber\\
        &+ \frac{1}{2 \eta_{1} (1 - \beta_{1})} \Big \Vert V_{1}^{-1/2}(\theta_1 - \theta^*)\Big \Vert^2 + \frac{dD_\infty G_{\infty}}{1-\beta_{1}}\sum_{t=1}^T{\beta_{1,t}} + \frac{dG_\infty^2 C_u\eta_0}{1-\beta_{1}}\sqrt{T}   \nonumber\\
    \leq&\frac{dD_\infty^2C_l^{-1}}{(1 - \beta_{1})\eta_0}\bigg(\sqrt{T}+ \sum_{k=1}^{\tilde{a}_T-1} \sqrt{a_{k+1}}\bigg) + \sum_{t=1}^{T-1} \Bigg[\frac{\beta_{1,t+1}-\beta_{1,t}}{2 \eta_{t+1} (1 - \beta_{1,t})(1 - \beta_{1,t+1})}  \Big \Vert V_{\tilde{a}_{t+1}}^{-1/2}(\theta_{t+1} - \theta^*)\Big \Vert^2 \Bigg]\nonumber \\
        &+ \frac{1}{2 \eta_{1} (1 - \beta_{1})} 
        \Big ( D_\infty e_d^\top V_{1}^{-1}D_\infty e_d \Big ) + \frac{dD_\infty G_{\infty}}{1-\beta_{1}}\sum_{t=1}^T{\beta_{1,t}} + \frac{dG_\infty^2 C_u\eta_0}{1-\beta_{1}}\sqrt{T}   \nonumber\\
    \leq&\frac{dD_\infty^2C_l^{-1}}{(1 - \beta_{1})\eta_0}\bigg(\sqrt{T}+ \sum_{k=1}^{\tilde{a}_T-1} \sqrt{a_{k+1}}\bigg)+ \sum_{t=1}^{T-1} \Bigg[\frac{\beta_{1,t+1}-\beta_{1,t}}{2 \eta_{t+1} (1 - \beta_{1,t})(1 - \beta_{1,t+1})}  \Big \Vert V_{\tilde{a}_{t+1}}^{-1/2}(\theta_{t+1} - \theta^*)\Big \Vert^2 \Bigg]\nonumber \\
        &+ \frac{dD_\infty^2C_l^{-1}}{2 \eta_{1} (1 - \beta_{1})} 
         + \frac{dD_\infty G_{\infty}}{1-\beta_{1}}\sum_{t=1}^T{\beta_{1,t}} + \frac{dG_\infty^2 C_u\eta_0}{1-\beta_{1}}\sqrt{T}  \nonumber\\
    \leq&\frac{dD_\infty^2C_l^{-1}}{(1 - \beta_{1})\eta_0}\bigg(\sqrt{T}+ \sum_{k=1}^{\tilde{a}_T-1} \sqrt{a_{k+1}}\bigg)+ \sum_{t=1}^{T-1} \Bigg[\frac{\beta_{1,t+1}{\mathbb{I}}_{\beta_{1,t+1}>\beta_{1,t}}}{2 \eta_{t+1} (1 - \beta_{1,t})(1 - \beta_{1,t+1})}  \Big \Vert V_{\tilde{a}_{t+1}}^{-1/2}(\theta_{t+1} - \theta^*)\Big \Vert^2 \Bigg] \nonumber \\
        &+ \frac{dD_\infty^2C_l^{-1}}{2 \eta_{1} (1 - \beta_{1})} + \frac{dD_\infty G_{\infty}}{1-\beta_{1}}\sum_{t=1}^T{\beta_{1,t}} + \frac{dG_\infty^2 C_u\eta_0}{1-\beta_{1}}\sqrt{T} \nonumber\\
    \leq&\frac{dD_\infty^2C_l^{-1}}{(1 - \beta_{1})\eta_0}\bigg(\sqrt{T}+ \sum_{k=1}^{\tilde{a}_T-1} \sqrt{a_{k+1}}\bigg)+ \sum_{t=1}^{T-1} \Bigg[\frac{\beta_{1,t+1}{\1}_{\beta_{1,t+1}>\beta_{1,t}}}{2 \eta_{t+1} (1 - \beta_{1})^2}  \Bigg ( D_\infty e_d^\top {C_l^{-1}I_d}D_\infty e_d \Bigg ) \Bigg] \nonumber \\
        &+ \frac{dD_\infty^2C_l^{-1}}{2 \eta_{1} (1 - \beta_{1})} + \frac{dD_\infty G_{\infty}}{1-\beta_{1}}\sum_{t=1}^T{\beta_{1,t}} + \frac{dG_\infty^2 C_u\eta_0}{1-\beta_{1}}\sqrt{T}   \nonumber \\
    \leq&\frac{dD_\infty^2C_l^{-1}}{(1 - \beta_{1})\eta_0}\bigg(\sqrt{T}+ \sum_{k=1}^{\tilde{a}_T-1} \sqrt{a_{k+1}}\bigg)+ \sum_{t=1}^{T-1} \Bigg[\frac{\beta_{1,t+1}{\1}_{\beta_{1,t+1}>\beta_{1,t}}dD_\infty^2}{2 C_l\eta_{t+1} (1 - \beta_{1})^2}  \Bigg] \nonumber \\
        &+ \frac{dD_\infty^2}{2 C_l\eta_{1} (1 - \beta_{1})} + \frac{dD_\infty G_{\infty}}{1-\beta_{1}}\sum_{t=1}^T{\beta_{1,t}} + \frac{dG_\infty^2 C_u\eta_0}{1-\beta_{1}}\sqrt{T}    
\end{align}
\end{proof}

\begin{corollary}
Suppose $\beta_{1,t} = \beta_1 \lambda^t,\ \ 0<\lambda<1$ in Theorem~\ref{suptheorem:convex}, then we have:
\begin{align}
\sum_{t=1}^T f_t(\theta_t) - f_t(\theta^*) \leq&\frac{dD_\infty^2C_l^{-1}}{(1 - \beta_{1})\eta_0}\bigg(\sqrt{T}+ \sum_{k=1}^{\tilde{a}_T-1} \sqrt{a_{k+1}}\bigg)+\frac{dD_\infty^2}{2 C_l\eta_{1} (1 - \beta_{1})}  \nonumber \\
    &+ \frac{dD_\infty G_{\infty}\beta_{1}}{(1-\beta_{1})(1-\lambda)} + \frac{dG_\infty^2 C_u\eta_0}{1-\beta_{1}}\sqrt{T}  \label{supeq:normal1_1}
\end{align}
\end{corollary}
\begin{proof}
    It is easy to prove using:
\begin{align}
    \sum_{t=1}^T \beta_{1,t}=\sum_{t=1}^T \beta_{1}\lambda^{t-1} < \sum_{t=1}^\infty\beta_{1}\lambda^{t-1} \leq \frac{\beta_1}{1-\lambda}
\label{supeq:obvious}
\end{align}
Plugging~\eqref{supeq:obvious} into~\eqref{supeq:F}, we can derive the results above. 
\end{proof}
\begin{corollary}
Suppose $a_n = 2^{n-1},\ \ \beta_3=\frac{1}2{}$ in \eqref{supeq:normal1_1}, then we have:
\begin{align}
\sum_{t=1}^T f_t(\theta_t) - f_t(\theta^*) \leq&\frac{(3+\sqrt{2})dD_\infty^2}{(1 - \beta_{1})C_l\eta_0}\sqrt{T}+\frac{dD_\infty^2}{2 C_l\eta_{1} (1 - \beta_{1})}  \nonumber \\
    &+ \frac{dD_\infty G_{\infty}\beta_{1}}{(1-\beta_{1})(1-\lambda)} + \frac{dG_\infty^2 C_u\eta_0}{1-\beta_{1}}\sqrt{T}  
\end{align}
\end{corollary}
\begin{proof}
    It is easy to prove using:
\begin{align}
    \sum_{t=1}^T a^{t-1}=\frac{1-a^T}{1-a}
\end{align}
\end{proof}

\paragraph{Remark (momentum schedule).} As is standard in Adam-type convergence analyses \citep{kingma2014adam,reddi2019convergence}, the $O(\sqrt{T})$ regret rate requires a non-increasing momentum schedule, e.g., $\beta_{1,t}=\beta_1\lambda^{t-1}$ as in Corollary~G.4. With the constant $\beta_1=0.9$ used in our experiments, the bound retains an additional term of order $\frac{\beta_1}{1-\beta_1}dD_\infty G_\infty T$; this gap between theory and implementation is shared by the existing convergence analyses of Adam and AMSGrad, whose practical implementations likewise use constant momentum, and does not affect the role of IDU, which bounds the cumulative pre-conditioner variation (the $\sum_k\sqrt{a_{k+1}}$ term) independently of the momentum schedule.

\section{Theorem 4 in main paper}
\label{supproof:4}
\begin{lemma}
\label{lemma:mt}
\citep{zhuang2021momentum} Let $m_t = \beta_1 m_{t-1}+(1-\beta_1)g_t$, let $Q_t \in \mathbb{R}^d$, then
\begin{equation}
   \Big \langle Q_t, g_t\Big \rangle = \frac{1}{1-\beta_1} \Big(\Big \langle Q_t, m_t\Big \rangle -\Big \langle Q_{t-1}, m_{t-1}\Big \rangle \Big) +\Big \langle Q_{t-1}, m_{t-1}\Big \rangle + \frac{\beta_1}{ 1 - \beta_1}\Big \langle Q_{t-1} - Q_t, m_{t-1}\Big \rangle
\end{equation}
\end{lemma}
\begin{theorem}{(Convergence analysis for non-convex stochastic optimization)}
\label{suptheorem:CDU_nonconvex}
The update of $\theta_t$ can be described as $\theta_{t+1}=\theta_t-\eta_{t}V_{\tilde{a}_t}m_t$, and $m_t=\beta_1m_{t-1}+(1-\beta_1)g_t$.  \\
Under the assumptions:
\begin{itemize}[leftmargin=*]
    \item $f$ is differentiable and $f^*\leq f \leq F$. $\nabla f(x)$ is L-Lipschitz continuous, $i.e.\ \ \Vert \nabla f(x) - \nabla f(y)  \Vert \leq L  \Vert x-y  \Vert,\ \forall x,y$.
    
    \item The noisy gradient is unbiased and its infinity norm is bounded by N, $i.e.\ \ \mathbb{E} g_t = \nabla f(x)$, $\Vert g_t \Vert_\infty \leq N$.
\end{itemize}
The hyperparameters are set as: $\eta_{t}=\eta_0t^{-p}$, $\eta_0>0$, $p\in (0,1)$ where the bounds are $C_lI\preceq V_{\tilde{a}_t}\preceq C_uI$, and $0<C_l<C_u$ ($A\preceq B$ means $B-A$ is a positive semi-definite matrix). And the $\epsilon$ and $N$ ensure $C_l$ and $C_u$ exist. For sequence $\{\theta_t\}$ generated by Anon, we have: 
\begin{align*}
\scalebox{1.0}{
$
\frac{1}{T} \sum_{t=1}^T \mathbb{E}\Big \Vert \nabla f(x_t)\Big \Vert^2 \leq \frac{1}{\eta_0 C_l} T^{p -1} \Big(F - f^*+ K \int_{1}^{T} t^{-2p}\, \mathrm{d}t + J +K \Big)
 $
 }
\end{align*}
where
\begin{align*}
\scalebox{1.0}{
$
J=  \frac{\beta_1^2d}{4L (1-\beta_1)^2}N^2 + \frac{3dN^2}{1-\beta_1}\eta_0C_u\sum_{k=1}^{\tilde{a}_T} {(a_k-\1_{k\ne 1})}^{-p},\ K=\Big( \frac{1}{1-\beta_1}+\frac{1}{2} \Big) L \eta_0^2 N^2 C_u^2d
 $
 }
\end{align*}
\end{theorem}

\begin{proof}
Let $x_t \equiv \theta_t$ denote the iterates. Let $A_t=V_{\tilde{a}_t},\ Q_t = \eta_t A_t \nabla f(x_t)$ and let $Q_0 = Q_1$, we have
\begin{align}
    \sum_{t=1}^T\Big \langle Q_t, g_{t}\Big \rangle &= \frac{1}{1-\beta_1}\Big \langle Q_T, m_{T}\Big \rangle + \sum_{t=1}^{T}\Big \langle Q_{t-1}, m_{t-1}\Big \rangle + \frac{\beta_1}{1-\beta_1} \sum_{t=1}^T\Big \langle Q_{t-1}-Q_t, m_{t-1}\Big \rangle \nonumber\\
    &= \frac{\beta_1}{1-\beta_1}\Big \langle Q_T, m_{T}\Big \rangle + \sum_{t=1}^{T}\Big \langle Q_{t}, m_{t}\Big \rangle + \frac{\beta_1}{1-\beta_1} \sum_{t=0}^{T-1}\Big \langle Q_{t}-Q_{t+1}, m_{t}\Big \rangle \label{eq:sum_prod}
\end{align}
First we derive a lower bound for  \eqref{eq:sum_prod}.
\begin{align}
   \Big \langle Q_t, g_{t}\Big \rangle =&\Big \langle \eta_t A_t \nabla f(x_t), g_{t}\Big \rangle \nonumber\\
    =&\Big \langle \eta_{t-1} A_{t-1} \nabla f(x_t), g_{t}\Big \rangle -\Big \langle (\eta_{t-1} A_{t-1}-\eta_t A_t ) \nabla f(x_t), g_{t}\Big \rangle \nonumber\\
    \geq&\Big \langle \eta_{t-1} A_{t-1} \nabla f(x_t), g_{t}\Big \rangle -\Big \Vert \nabla f(x_t) \Big \Vert_{\infty}d\Big\Vert \eta_{t-1} A_{t-1} - \eta_t A_t\Big\Vert_1\Big \Vert g_{t}\Big \Vert_{\infty} \nonumber\\
    &\Big( \textit{By H\"older's inequality} \Big) \nonumber \nonumber\\
    \geq& \Big \langle \eta_{t-1} A_{t-1} \nabla f(x_t), g_{t}\Big \rangle - dN^2 \1_{t\ne a_{\tilde{a}_t}}\Big(\Big \Vert \eta_{t-1}A_{t-1}\Big\Vert  -\Big \Vert \eta_t A_t\Big \Vert_1 \Big) \nonumber\\
        &-dN^2 \1_{t=a_{\tilde{a}_t}}\Big(\Big \Vert \eta_{t-1}A_{t-1}- \eta_t A_t\Big \Vert_1 \Big)\\
        &\Big( \textit{Since }\Big \Vert g_{t} \Big \Vert_\infty \leq N,\  \eta_{t-1} \geq \eta_t > 0, A_{t-1} = A_t \textit{ 
 when } t\ne a_{\tilde{a}_t}\Big)\nonumber
\end{align}
Perform telescope sum, we have
\begin{align}
\label{eq:prod_lower_bound}
    \sum_{t=1}^T\Big \langle Q_t, g_{t}\Big \rangle \geq& \sum_{t=1}^T\Big \langle \eta_{t-1} A_{t-1} \nabla f(x_t), g_{t}\Big \rangle - dN^2\sum_{k=1}^{\tilde{a}_T-1} \Big(\Big \Vert \eta_{a_k} A_{a_k}\Big\Vert_1 -\Big \Vert \eta_{a_{k+1}-1} A_{a_{k+1}-1}\Big \Vert_1 \Big)\nonumber\\
        &-dN^2 \sum_{k=1}^{\tilde{a}_T}\Big \Vert \eta_{a_k-1}A_{a_k-1}- \eta_{a_k} A_{a_k}\Big \Vert_1 - dN^2 \Big(\Big \Vert \eta_{a_{\tilde{a}_T}} A_{a_{\tilde{a}_T}}\Big \Vert_1 -\Big \Vert \eta_{T} A_{T}\Big \Vert_1 \Big)\nonumber\\
    \geq& \sum_{t=1}^T\Big \langle \eta_{t-1} A_{t-1} \nabla f(x_t), g_{t}\Big \rangle - dN^2\sum_{k=1}^{\tilde{a}_T-1} \Big \Vert \eta_{a_k} A_{a_k}\Big \Vert_1 \nonumber\\
        &-dN^2 \sum_{k=1}^{\tilde{a}_T}\Big \Vert \eta_{a_k-1}A_{a_k-1}- \eta_{a_k} A_{a_k}\Big \Vert_1 - dN^2 \Big \Vert \eta_{a_{\tilde{a}_T}} A_{a_{\tilde{a}_T}}\Big \Vert_1 \nonumber\\
    \geq& \sum_{t=1}^T\Big \langle \eta_{t-1} A_{t-1} \nabla f(x_t), g_{t}\Big \rangle - dN^2\sum_{k=1}^{\tilde{a}_T} \Big \Vert \eta_{a_k} A_{a_k}\Big \Vert_1 \nonumber\\
        &-dN^2 \sum_{k=1}^{\tilde{a}_T}\Big(\Big \Vert \eta_{a_k-1}A_{a_k-1}\Big \Vert_1 + \Big \Vert\eta_{a_k} A_{a_k}\Big \Vert_1 \Big) \nonumber\\
    =& \sum_{t=1}^T\Big \langle \eta_{t-1} A_{t-1} \nabla f(x_t), g_{t}\Big \rangle - 2dN^2\sum_{k=1}^{\tilde{a}_T} \Big \Vert \eta_{a_k} A_{a_k}\Big \Vert_1 -dN^2 \sum_{k=1}^{\tilde{a}_T}\Big \Vert \eta_{a_k-1}A_{a_k-1}\Big \Vert_1 \nonumber\\
    \geq& \sum_{t=1}^T\Big \langle \eta_{t-1} A_{t-1} \nabla f(x_t), g_{t}\Big \rangle - 3dN^2\sum_{k=1}^{\tilde{a}_T} \eta_{a_k-1} C_u 
\end{align}
Next, we derive an upper bound for $\sum_{t=1}^T\Big \langle Q_t, g_t\Big \rangle$ by deriving an upper-bound for the RHS of  \eqref{eq:sum_prod}. We derive an upper bound for each part.

\begin{align}
\Big\langle Q_t, m_{t}\Big \rangle &=\Big \langle \eta_t A_t \nabla f(x_t), m_{t}\Big \rangle =\Big \langle \nabla f(x_t),  \eta_t A_t m_{t}\Big \rangle \nonumber\\
&=\Big \langle \nabla f(x_t), x_{t} - x_{t+1}\Big \rangle \nonumber\\
&\leq f(x_t) - f(x_{t+1}) + \frac{L}{2}\Big \Vert x_{t+1} - x_t\Big \Vert^2 \label{eq:bd13}\\
&\,\,\,\,\,\,\, \Big( \textit{By L-smoothness of }f \Big) \nonumber
\end{align}
Perform telescope sum, we have
\begin{align}
\sum_{t=1}^T\Big \langle Q_t, m_{t}\Big \rangle \leq f(x_1) - f(x_{T+1}) + \frac{L}{2} \sum_{t=1}^T\Big \Vert \eta_t A_t m_{t}\Big \Vert^2 
\label{eq:bd14}
\end{align}
\begin{align}
\Big\langle Q_t - Q_{t+1}, m_{t}\Big \rangle =&\Big \langle \eta_t A_t \nabla f(x_t) - \eta_{t+1} A_{t+1} \nabla f(x_{t+1}), m_{t}\Big \rangle \nonumber\\
    =&\Big \langle \eta_t A_t \nabla f(x_t) - \eta_t A_t \nabla f(x_{t+1})  , m_{t}\rangle \nonumber \\
        &+\Big \langle \eta_t A_t \nabla f(x_{t+1}) - \eta_{t+1} A_{t+1} \nabla f(x_{t+1}) , m_{t}\rangle \nonumber\\
    =&\Big \langle \nabla f(x_t) - \nabla f(x_{t+1}), \eta_t A_t m_{t}\Big \rangle +\Big \langle (\eta_t A_t - \eta_{t+1} A_{t+1}) \nabla f(x_t) ,m_{t}\Big \rangle \nonumber\\
        &=\Big \langle \nabla f(x_t) - \nabla f(x_{t+1}), x_t - x_{t+1}\Big \rangle +\Big \langle \nabla f(x_t), ( \eta_t A_t - \eta_{t+1} A_{t+1}) m_{t}\Big \rangle \nonumber\\
    \leq& L\Big \Vert x_{t+1} - x_t\Big \Vert^2 +\Big \langle \nabla f(x_t), ( \eta_t A_t - \eta_{t+1} A_{t+1}) m_{t}\Big \rangle \nonumber\\
        &\Big( \textit{By smoothness of }f \Big) \nonumber \nonumber\\
    \leq& L\Big \Vert x_{t+1} - x_t\Big \Vert^2 +\Big \Vert \nabla f(x_t)\Big \Vert_\infty d\Big\Vert \eta_t A_t - \eta_{t+1} A_{t+1}\Big \Vert_1\Big \Vert m_{t}\Big \Vert_\infty \nonumber\\
        &\Big( \textit{By H\"older's inequality} \Big) \nonumber \\
    \leq& L\Big \Vert x_{t+1} - x_t\Big \Vert^2 + dN^2 \1_{t+1\ne a_{\tilde{a}_{t+1}}}\Big(\Big \Vert \eta_{t}A_{t}\Big\Vert_1  -\Big \Vert \eta_{t+1} A_{t+1}\Big \Vert_1 \Big) \nonumber\\
        &+dN^2 \1_{t+1=a_{\tilde{a}_{t+1}}}\Big(\Big \Vert \eta_{t}A_{t}- \eta_{t+1} A_{t+1}\Big \Vert_1 \Big)\label{eq:bd15}\\
        &\Big( \textit{Since }\eta_{t} \geq \eta_{t+1} > 0, A_{t+1} = A_t \textit{  when } t+1\ne a_{\tilde{a}_{t+1}}\Big)\nonumber
\end{align}
Perform telescope sum, we have 
\begin{align}
    \sum_{t=1}^{T-1}\Big \langle Q_t - Q_{t+1}, m_{t}\rangle \leq& L \sum_{t=1}^{T-1}\Big \Vert \eta_t A_t m_{t}\Big \Vert^2 + dN^2\sum_{k=1}^{\tilde{a}_{T}-1} \Big(\Big \Vert \eta_{a_k} A_{a_k}\Big\Vert_1 -\Big \Vert \eta_{a_{k+1}-1} A_{a_{k+1}-1}\Big \Vert_1 \Big)\nonumber\\
        &+dN^2 \sum_{k=1}^{\tilde{a}_{T}-1}\Big \Vert \eta_{a_{k+1}-1}A_{a_{k+1}-1}- \eta_{a_{k+1}} A_{a_{k+1}}\Big \Vert_1 \nonumber\\
        &+dN^2 \Big(\Big \Vert \eta_{a_{\tilde{a}_{T}}} A_{a_{\tilde{a}_{T}}}\Big \Vert_1 -\Big \Vert \eta_{T} A_{T}\Big \Vert_1 \Big)\nonumber\\
    \leq& L \sum_{t=1}^{T-1}\Big \Vert \eta_t A_t m_{t}\Big \Vert^2 + dN^2\sum_{k=1}^{\tilde{a}_{T}-1} \Big \Vert \eta_{a_k} A_{a_k}\Big\Vert_1 \nonumber\\
        &+dN^2 \sum_{k=1}^{\tilde{a}_{T}-1}\Big(\Big \Vert \eta_{a_{k+1}-1}A_{a_{k+1}-1}\Big \Vert_1+ \Big \Vert\eta_{a_{k+1}} A_{a_{k+1}}\Big \Vert_1\Big) \nonumber\\
        &+dN^2 \Big \Vert \eta_{a_{\tilde{a}_{T}}} A_{a_{\tilde{a}_{T}}}\Big \Vert_1\nonumber\\
    \leq& L \sum_{t=1}^{T-1}\Big \Vert \eta_t A_t m_{t}\Big \Vert^2 + dN^2\sum_{k=1}^{\tilde{a}_{T}} \Big \Vert \eta_{a_k} A_{a_k}\Big\Vert_1 \nonumber\\
        &+dN^2 \sum_{k=1}^{\tilde{a}_{T}-1}\Big(\Big \Vert \eta_{a_{k+1}-1}A_{a_{k+1}-1}\Big \Vert_1+ \Big \Vert\eta_{a_{k+1}} A_{a_{k+1}}\Big \Vert_1\Big) \nonumber\\
    \leq& L \sum_{t=1}^{T-1}\Big \Vert \eta_t A_t m_{t}\Big \Vert^2 + 2dN^2\sum_{k=1}^{\tilde{a}_{T}} \Big \Vert \eta_{a_k} A_{a_k}\Big\Vert_1 \nonumber\\
        &+dN^2 \sum_{k=1}^{\tilde{a}_{T}-1}\Big \Vert \eta_{a_{k+1}-1}A_{a_{k+1}-1}\Big \Vert_1\nonumber\\
    \leq& L \sum_{t=1}^{T-1}\Big \Vert \eta_t A_t m_{t}\Big \Vert^2 + 3dN^2\sum_{k=1}^{\tilde{a}_{T}}\eta_{a_k-1} C_u
    \label{eq:bd16}
\end{align}
We also have
\begin{align}
   \Big \langle Q_T, m_{T}\Big \rangle =&\Big \langle \eta_T A_T \nabla f(x_T), m_{T} \Big \rangle =\Big \langle \nabla f(x_T), \eta_T A_T m_{T}\Big \rangle \nonumber\\
    \leq& L \frac{1-\beta_1}{\beta_1}\Big \Vert  \eta_T A_T m_{T}\Big \Vert^2 + \frac{\beta_1}{4L(1-\beta_1)}\Big \Vert \nabla f(x_T)\Big \Vert^2 \nonumber\\
        &\Big( \textit{By Young's inequality} \Big) \nonumber\\
    \leq& L \frac{1-\beta_1}{\beta_1}\Big \Vert  \eta_T A_T m_{T}\Big \Vert^2 + \frac{\beta_1d}{4L (1-\beta_1)} N^2
    \label{eq:bd17}
\end{align}
Combine  \eqref{eq:bd14},  \eqref{eq:bd16} and  \eqref{eq:bd17} into  \eqref{eq:sum_prod}, we have
\begin{align}
    \sum_{t=1}^T\Big \langle Q_t, g_{t}\Big \rangle \leq& L\Big \Vert  \eta_T A_T m_{T}\Big \Vert^2 + \frac{\beta_1^2d}{4L (1-\beta_1)^2} N^2 \nonumber \\ 
        &+ f(x_1) - f(x_{T+1}) + \frac{L}{2} \sum_{t=1}^T\Big \Vert \eta_t A_t m_{t}\Big \Vert^2 \nonumber \\
        &+ \frac{\beta_1}{1-\beta_1} L \sum_{t=1}^{T-1}\Big \Vert \eta_t A_t m_{t}\Big \Vert^2 + \frac{3\beta_1}{1-\beta_1} dN^2\sum_{k=1}^{\tilde{a}_{T}}\eta_{a_k-1} C_u \nonumber\\
    \leq& f(x_1) - f(x_{T+1}) + \Big( \frac{1}{1-\beta_1}+\frac{1}{2} \Big)L \sum_{t=1}^T\Big \Vert \eta_t A_t m_{t}\Big \Vert^2 \nonumber \\
        &+ \frac{\beta_1^2d}{4L (1-\beta_1)^2}N^2 + \frac{3\beta_1}{1-\beta_1} dN^2\sum_{k=1}^{\tilde{a}_{T}}\eta_{a_k-1} C_u
    \label{eq:bd18}
\end{align}
Combine \eqref{eq:prod_lower_bound} and \eqref{eq:bd18}, we have
\begin{align}
 \sum_{t=1}^T\Big \langle \eta_{t-1} A_{t-1} \nabla f(x_t), g_{t}\Big \rangle -& 3dN^2\sum_{k=1}^{\tilde{a}_T} \eta_{a_k-1} C_u  \leq    \sum_{t=1}^T\Big \langle Q_t, g_{t}\Big \rangle \nonumber \\
    \leq& f(x_1) - f(x_{T+1}) + \Big( \frac{1}{1-\beta_1}+\frac{1}{2} \Big)L \sum_{t=1}^T\Big \Vert \eta_t A_t m_{t}\Big \Vert^2 \nonumber \\
        &+ \frac{\beta_1^2d}{4L (1-\beta_1)^2}N^2 + \frac{3\beta_1}{1-\beta_1} dN^2\sum_{k=1}^{\tilde{a}_{T}}\eta_{a_k-1} C_u
\end{align}
Hence we have
\begin{align}
    \sum_{t=1}^T\Big \langle \eta_{t-1} A_{t-1} \nabla f(x_t), g_{t}\Big \rangle \leq& f(x_1) - f(x_{T+1}) + \Big( \frac{1}{1-\beta_1} + \frac{1}{2} \Big)L \sum_{t=1}^T\Big \Vert \eta_t A_t m_t\Big \Vert^2 \nonumber \\
        &+ \frac{\beta_1^2d}{4L (1-\beta_1)^2}N^2 + \frac{3dN^2}{1-\beta_1}\sum_{k=1}^{\tilde{a}_T} \eta_{a_k-1} C_u \nonumber\\
    \leq& f(x_1) - f^* + \Big( \frac{1}{1-\beta_1}+\frac{1}{2} \Big) L \eta_0^2 N^2 C_u^2d \sum_{t=1}^T t^{-2p} \nonumber \\
        &+ \frac{\beta_1^2d}{4L (1-\beta_1)^2}N^2 +  \frac{3dN^2}{1-\beta_1}\eta_0C_u\sum_{k=1}^{\tilde{a}_T} {(a_k-\1_{k\ne 1})}^{-p} \nonumber\\
    \leq& f(x_1) - f^* + \Big( \frac{1}{1-\beta_1}+\frac{1}{2} \Big) L \eta_0^2 N^2 C_u^2d\Big(1+ \int_{1}^{T} t^{-2p}\, \mathrm{d}t\Big) \nonumber \\
        &+ \frac{\beta_1^2d}{4L (1-\beta_1)^2}N^2 +  \frac{3dN^2}{1-\beta_1}\eta_0C_u\sum_{k=1}^{\tilde{a}_T} {(a_k-\1_{k\ne 1})}^{-p} \nonumber\\
    \leq& f(x_1) - f^* + \Big( \frac{1}{1-\beta_1}+\frac{1}{2} \Big) L \eta_0^2 N^2 C_u^2d \int_{1}^{T} t^{-2p}\, \mathrm{d}t \nonumber \\
        &+ \underbrace{\frac{\beta_1^2d}{4L (1-\beta_1)^2}N^2 + \frac{3dN^2}{1-\beta_1}\eta_0C_u\sum_{k=1}^{\tilde{a}_T} {(a_k-\1_{k\ne 1})}^{-p}}_{J} \nonumber\\
        &+ \underbrace{\Big( \frac{1}{1-\beta_1}+\frac{1}{2} \Big) L \eta_0^2 N^2 C_u^2d}_{K}\nonumber\\
    \leq& f(x_1) - f^* + K \int_{1}^{T} t^{-2p}\, \mathrm{d}t + J +K
\end{align}
Take expectations on both sides, we have
\begin{align}
    \sum_{t=1}^T\Big \langle \eta_{t-1} A_{t-1} \nabla f(x_t), \nabla f(x_t)\Big \rangle \leq& \mathbb{E} f(x_1) - f^*+ K \int_{1}^{T} t^{-2p}\, \mathrm{d}t + J +K\nonumber\\
    \leq& F - f^*+ K \int_{1}^{T} t^{-2p}\, \mathrm{d}t + J +K
    \label{eq:bd19}
\end{align}
Note that we have $\eta_t$ decays monotonically with $t$, hence
\begin{align}
    \sum_{t=1}^T\Big \langle \eta_{t-1} A_{t-1} \nabla f(x_t), \nabla f(x_t)\Big \rangle &\geq \eta_0 T^{-p} \sum_{t=1}^T\Big \langle A_{t-1} \nabla f(x_t), \nabla f(x_t)\Big \rangle \\
    &\geq \eta_0 T^{1-p} C_l \frac{1}{T}\sum_{t=1}^T\Big \Vert \nabla f(x_t)\Big \Vert^2
    \label{eq:bd20}
\end{align}
Combine  \eqref{eq:bd19} and  \eqref{eq:bd20}, using $f \leq F$, we have
\begin{align}
    \frac{1}{T} \sum_{t=1}^T\Big \Vert \nabla f(x_t)\Big \Vert^2 &\leq \frac{1}{\eta_0 C_l} T^{p -1} \Big(F - f^*+ K \int_{1}^{T} t^{-2p}\, \mathrm{d}t + J +K \Big)
\end{align}
And it is easy to prove when $a_n=2^{n-1}$, we have
\begin{align}
    J=&\ \frac{\beta_1^2d}{4L (1-\beta_1)^2}N^2 + \frac{3dN^2}{1-\beta_1}\eta_0C_u\sum_{k=1}^{\tilde{a}_T} {(a_k-\1_{k\ne 1})}^{-p}\\
    \le&\ \frac{\beta_1^2d}{4L (1-\beta_1)^2}N^2 + \frac{3dN^2}{1-\beta_1}\eta_0C_u\sum_{k=1}^{\infty} {(a_k-\1_{k\ne 1})}^{-p}\\
    =&\ \frac{\beta_1^2d}{4L (1-\beta_1)^2}N^2 + \frac{3dN^2}{1-\beta_1}\eta_0C_u\left(1+\sum_{k=2}^{\infty} {(2^{k-1}-\1_{k\ne 1})}^{-p}\right)\\
    \le&\ \frac{\beta_1^2d}{4L (1-\beta_1)^2}N^2 + \frac{3dN^2}{1-\beta_1}\eta_0C_u\left(1+\sum_{k=2}^{\infty} {(2^{k-2})}^{-p}\right)\\
    =&\ \frac{\beta_1^2d}{4L (1-\beta_1)^2}N^2 + \frac{3dN^2}{1-\beta_1}\eta_0C_u\left(1+\sum_{k=1}^{\infty} {(2^{-p})}^{k-1}\right)\\
    =&\ \frac{\beta_1^2d}{4L (1-\beta_1)^2}N^2 + \frac{3dN^2}{1-\beta_1}\eta_0C_u\left(1+\frac{1}{1-2^{-p}}\right)
\end{align}
\color{black}
\end{proof}

\end{document}